\documentclass{article}

 \usepackage[preprint,nonatbib]{neurips_2020}

\usepackage[nonatbib]{neurips_2020}
\usepackage[utf8]{inputenc} 
\usepackage[T1]{fontenc}    
\usepackage{hyperref}       
\usepackage{url}            
\usepackage{booktabs}       
\usepackage{amsfonts}       
\usepackage{nicefrac}       
\usepackage{microtype}      
\usepackage{pifont}
\usepackage{wrapfig}
\usepackage{float}

\usepackage{enumitem}
\usepackage{filecontents}
\usepackage{todonotes}
\usepackage{subcaption}
\usepackage{mathtools, amsmath,amssymb, amsthm}
\usepackage[noline,boxed, noend]{algorithm2e}
\usepackage{bbm}

\usepackage{amsfonts,bm,dsfont,color}

\usepackage[numbers]{natbib}
\usepackage{mathtools}

\newtheorem{theorem}{Theorem}
\newtheorem{proposition}{Proposition}

\newtheorem{lemma}{Lemma}
\newtheorem{definition}{Definition}
\newtheorem{remark}{Remark}

\newtheorem{example}{Example}

\def\M{\mathcal{M}}

\def\I{\mathcal{I}}

\def\S{\mathbb{S}}
\def\A{\mathcal{A}}
\def\X{\mathcal{X}}

\def\Z{\mathcal{Z}}

\def\B{\mathcal{B}}

\def\E{\mathbb{E}}
\def\V{\mathbb{V}}
\def\Z{\mathcal{Z}}
\def\1{\mathbf{1}}
\def\P{\mathbb{P}}
\def\R{\mathbb{R}}
\def\N{\mathbb{N}}

\newcommand{\mc}[1]{\mathcal{#1}}
\newcommand{\mb}[1]{\mathbf{#1}}

\newcommand{\norm}[1]{\left\lVert#1\right\rVert}

\DeclareMathOperator*{\argmax}{arg\,max}
\DeclareMathOperator*{\argmin}{arg\,min}
\DeclareMathOperator*{\diam}{diam}
\DeclareMathOperator*{\simp}{\mathbf{\Delta}}

\def\t{\top}
\def\one{\mathbbm{1}}
\def\zero{\bm{0}}

\def\argmin{\text{argmin}}
\def\max{\text{max}}
\def\diag{\text{diag}}

\newcommand{\ignore}[1]{}
\newcommand{\floor}[1]{\left\lfloor #1 \right\rfloor}
\newcommand{\ceil}[1]{\left\lceil #1 \right\rceil}

\newcommand{\ip}[1]{\left\langle #1 \right\rangle}

\DeclareMathOperator*{\simpm}{\mathbf{\tilde{\Delta}}}

\title{An Empirical Process Approach to the Union Bound: Practical Algorithms for Combinatorial and Linear Bandits}

\author{%
  Julian Katz-Samuels\\
  Allen School of Computer Science \& Engineering\\
University of Washington\\
jkatzsam@cs.washington.edu\\
   \And
   Lalit Jain \\
   Foster School of Business\\
   University of Washington \\
   lalitj@uw.edu  \\
   \AND
   Zohar Karnin \\
   Amazon Web Services \\
   zkarnin@gmail.com  \\
   \And
   Kevin Jamieson \\
  Allen School of Computer Science \& Engineering\\
University of Washington\\
  jamieson@cs.washington.edu \\
}

\begin{document}
\maketitle

\begin{abstract}
This paper proposes near-optimal algorithms for the pure-exploration linear bandit problem in the fixed confidence and fixed budget settings. 
Leveraging ideas from the theory of suprema of empirical processes, we provide an algorithm whose sample complexity scales with the geometry of the instance and avoids an explicit union bound over the number of arms.
Unlike previous approaches which sample based on minimizing a worst-case variance (e.g. G-optimal design), we define an experimental design objective based on the Gaussian-width of the underlying arm set.
We provide a novel lower bound in terms of this objective that highlights its fundamental role in the sample complexity.
The sample complexity of our fixed confidence algorithm matches this lower bound, and in addition is computationally efficient for combinatorial classes, e.g. shortest-path, matchings and matroids, where the arm sets can be exponentially large in the dimension. 
Finally, we propose the first algorithm for linear bandits in the the fixed budget setting. Its guarantee matches our lower bound up to logarithmic factors.
\end{abstract}

\section{Introduction}

The pure exploration stochastic multi-armed bandit (MAB) problem has received attention in recent years because it offers a useful framework for designing algorithms for sequential experiments. In this paper, we consider a very general formulation of the pure exploration MAB problem, namely, \emph{pure exploration (transductive) linear bandits} \citep{fiez2019sequential} : given a set of measurement vectors $\X \subset \R^d$, a set of candidate items $\Z \subset \R^d$, and an \emph{unknown} parameter vector $\theta \in \R^d$, an agent plays a sequential game where at each round she chooses a measurement vector $x \in \X$ and observes a stochastic random variable whose expected value is $\langle x, \theta\rangle$. The goal is to identify $z_* \in \argmax_{z \in \Z} \langle z, \theta \rangle$. This problem generalizes many well-studied problems in the literature including best arm identification \citep{even2006action,jamieson2014lil,icml2013_karnin13,kaufmann2016complexity,chen2015optimal}, Top-K arm identification \citep{DBLP:conf/icml/KalyanakrishnanTAS12,simchowitz2017simulator,chen2017nearlyTopK}, the thresholding bandit problem \citep{locatelli2016optimal}, combinatorial bandits \citep{chen2014combinatorial,gabillon2016improved,chen2017nearly,cao2019disagreement,jain2019new}, and linear bandits where $\X=\Z$ \citep{soare2014best, xu2018fully, tao2018best}.
 
The recent work of \citep{fiez2019sequential} proposed an algorithm that is within a $\log(|\Z|)$ multiplicative factor of previously known lower bounds \cite{soare2014best} on the sample complexity. This term reflects a naive union bound over all informative directions $\{z_* -z : z \in \Z\setminus \{z_*\}\}$. Although one might be inclined to dismiss $\log(|\Z|)$ as a small factor, in many practical problems it can be extremely large. For example, in Top-K $\log(|\Z|) = \Theta(k \log(d))$ which would introduce an additional factor of $k$ that does not appear in the upper bounds of specialized algorithms for this class \cite{DBLP:conf/icml/KalyanakrishnanTAS12,chen2017nearly,kaufmann2016complexity}. As another example, if $\Z$ consists of many vectors pointing in nearly the same direction, $\log(|\Z|)$ can be arbitrarily large, while we show that the true sample complexity does not depend on $\log(|\Z|)$. Finally, in many applications of linear bandits such as content recommendation $|\Z|$ can be enormous and thus the factor $\log(|\Z|)$ can have a dramatic effect on the sample complexity.

The high-level goal of this paper is to study how the geometry of the measurement vectors $\X$ and the candidate items $\Z$ influences the sample complexity of the pure exploration transductive linear bandit problem in the moderate confidence regime.
We appeal to the fundamental TIS-inequality \cite{ibragimov1976norms} which describes the deviation of the suprema of a Gaussian process from its expectation, leading us to propose an experimental design based on minimizing the expected suprema.   
We make the following contributions. First, we show a novel lower bound for the non-interactive oracle MLE algorithm, which devises a fixed sampling scheme using knowledge of $\theta$. 
While this non-interacting lower bound is not a lower bound for adaptive algorithms, it is suggestive of what union bounds are necessary and can be a multiplicative dimension factor larger than known adaptive lower bounds. Second, we develop a new algorithm for the fixed confidence setting (defined below) that nearly matches the performance of this oracle algorithm. Moreover, this algorithm recovers many of the state-of-the-art sample complexity results for combinatorial bandits as special cases. Third, applied specifically to the combinatorial bandit setting, we develop a practical and computationally efficient algorithm. We include experiments that show that our algorithm outperforms existing algorithms, often by an order of magnitude. Finally, we show that our techniques extend to the fixed budget setting where we provide the first fixed budget algorithm for transductive linear bandits. This algorithm matches the lower bound up to a factor that in most standard settings is bounded by $\log(d)$. 


\section{Preliminaries}
In the (transductive) linear bandit problem, the agent is given a set $\X \subset \R^d$ and a set of items $\Z \subset \R^d$. At each round $t$, an algorithm $\A$ selects a measurement $X_t \in \X$ which is measurable with respect to the history $\mathcal{F}_{t-1} = (X_s, Y_s)_{s<t}$ and observes a noisy observation $Y_t = X_t^\t \theta + \eta$  where $\theta \in \R^d$ is the unknown model parameter and $\eta$ is independent mean-0 Gaussian noise\footnote{Our results still apply in the case where the noise is sub-Gaussian, but for simplicity here we assume that the noise is Gaussian (see the Supplementary Material).}. We assume that $ \text{argmax}_{z \in \Z} \langle z, \theta \rangle = \{z_*\}$, and the goal is to identify $z_*$. We consider two distinct settings. 
\begin{definition} \textbf{Fixed-Confidence:}
Fix $\X, \Z, \Theta \subset \R^d$. An algorithm $\A$ is $\delta$-PAC for $(\X, \Z, \Theta)$ if 1) the algorithm has a stopping time $\tau$ wrt $(\mathcal{F}_t)_{t \in \N}$ and 2) at time $\tau$ it makes a recommendation $\widehat{z} \in \mc{Z}$ and for all $\theta \in \Theta$ it satisfies $\P_\theta(\widehat{z} = z_*) \geq 1 -\delta$. 
\end{definition}

\begin{definition} \textbf{Fixed-Budget:}
Fix $\X, \Z, \Theta \subset \R^d$ and a budget $T$. An algorithm $\A$ for fixed-budget returns a recommendation $\widehat{z} \in \mc{Z}$ after $T$ rounds.
\end{definition}


Linear bandits is popular for applications such as content recommendation, digital advertisements, and A/B testing. 
For instance, in content recommendation $\mc{X}=\mc{Z} \subset \R^d$ may be sets of feature vectors describing songs (e.g., beats per minute, genre, etc.) and $\theta \in \R^d$ may represent an individual user's preferences over the song library.
An important sub-class of linear bandits is known as combinatorial bandits which is a focus of this work.

\textbf{Combinatorial Bandits:}
In the \emph{combinatorial bandit} setting, $\X  = \{ \mb{e}_1, \ldots, \mb{e}_d \}$ (where $\mb{e}_i$ is the $i$-th canonical basis vector) and $\Z \subset \{0,1\}^d$.
We will sometimes overload notation by treating $\Z$ as a collection of sets, e.g., for $z\in \Z$ writing $i \in z$ iff $\langle \mb{e}_i, z\rangle = 1$. We next give some examples of the combinatorial bandit setting. 

\begin{example}[\textsc{matroid}]
$\M = (S,\I)$ is a matroid where $S$ is a set of ground elements and $\I \subset 2^S$ is a collection of independent sets. This setting includes best arm identification, Top-K arm identification, identifying the minimum spanning tree with largest expected reward in a graph, and other important applications (see \citep{chen2016pure} for a list of applications). 
\end{example}

\begin{example}[\textsc{Matching}]
For a balanced bipartite graph with $d$ edges and $2\sqrt{d}$ vertices let $\Z$ denote the set of $\sqrt{d}!$ perfect bipartite matchings. The goal is to identify the matching $z\in \Z$ that maximizes $\langle \theta, z \rangle$. 
\end{example}


In some of these settings, $|\Z|$ is exponential in the dimension $d$. For example, in the problem of finding a best matching in a bipartite graph, $|\Z| = (\sqrt{d})!$. 
In this setting a naive evaluation of $\text{argmax}_{z \in \Z} \langle z, \theta \rangle$ by enumerating $\mc{Z}$ becomes impossible even if $\theta$ were known. For such problems, we assume access to a linear maximization oracle
\begin{align}
\textsc{ORACLE}(w) = \argmax_{z \in \Z} \langle z, w \rangle, \label{eq:lin_max_oracle}
\end{align}
which is available in many cases, including matroids, \textsc{matching}, and identifying a shortest path in a directed acyclic graph (DAG). 
We will characterize the computational complexity of an algorithm in terms of the number of calls to the maximization oracle. 
%

\section{Review of Gaussian Processes}


We now discuss how our work departs from previous approaches to the pure exploration linear bandit problem. Consider for a moment a fixed design where $n \geq d$ measurements $x_1,\dots,x_n$ were decided before observing any data, and subsequently for each $1\leq i\leq n$ we observe $y_i = \langle x_i, \theta \rangle + \eta_i$ with $\eta_i \sim \mc{N}(0,1)$.
In this setting the maximum likelihood estimator (MLE) is given by ordinary least squares as $\widehat{\theta} = (\sum_{i=1}^n x_i x_i^{\top})^{-1} \sum_{i=1}^{n} y_i x_i  $.
Substituting the value of $y_i$ into this expression, we obtain $\widehat{\theta} = \theta + \left( \sum_{i=1}^n x_ix_i^\top\right)^{-1/2} \eta$ in distribution where $\eta \sim \mc{N}(0,I_d)$.
After collecting $\{(x_i,y_i)\}_{i=1}^n$ and computing $\widehat{\theta}$, the most reasonable estimate for $z_\ast = \arg\max_{z \in \Z} \langle z, \theta \rangle$ is just $\widehat{z} = \arg\max_{z \in \mc{Z}} \langle z, \widehat{\theta} \rangle$.
The good event that $\widehat{z}=z_\ast$ occurs if and only if $\langle z_\ast -z, \widehat\theta  \rangle > 0$ for all $z \in \mc{Z} \setminus \{z_\ast\}$.
Since $\widehat{\theta}$ is a Gaussian random vector, for each $z\in\Z$,
$\langle z_\ast -z, \widehat\theta - \theta \rangle\sim \mathcal{N}(0, (z_\ast -z)^\top \left( \sum_{i=1}^n x_i x_i^\top \right)^{-1} (z_\ast-z))$.
If we apply a standard sub-Gaussian tail-bound with a union bound over all $z \in \mc{Z} \setminus \{z_*\}$, then we have with probability greater than $1-\delta$ that
\begin{equation}
    \langle z_{\ast} -z , \widehat\theta \rangle \geq \langle z_{\ast} - z, \theta \rangle - \sqrt{2 \|z_\ast -z\|_{A^{-1}}^2 \log(|\mc{Z}|/\delta)}
\end{equation}
for all $z \in \mc{Z} \setminus \{z_*\}$ \textit{simultaneously}, where we have taken $A = \sum_{i=1}^n x_i x_i^\top$ and used the notation $\|x\|_W^2 = x^\top W x$ for any square $W$.
Thus, we conclude that if $n$ and $\{x_1,\dots,x_n\}$ are chosen such that $\max_{z \in \mc{Z}} \frac{\| z_\ast - z \|_{(\sum_{i=1}^n x_i x_i^\top)^{-1}}^2}{\langle z_\ast-z,\theta^* \rangle^2} > 2 \log(|\mc{Z}|/\delta)$ then with probability at least $1-\delta$ we will have that $\langle z_\ast -z, \widehat\theta \rangle > 0$ for all $z \in \mc{Z}$ and consequently, $\widehat{z} = z_\ast$. 
This simple argument is the core of all approaches to pure exploration linear bandits until this paper \cite{soare2014best,icml2013_karnin13,xu2018fully,fiez2019sequential}.
However, applying a naive union bound over all $z \in \mc{Z}$ can be extremely weak and does not exploit the geometry of $\Z$ that induces many correlations among the random variables $\langle z_{\ast} -z , \widehat\theta \rangle$. At the heart of our approach is the following concentration inequality for the suprema of a Gaussian process (Theorem 5.8 in \citep{boucheron2013concentration}). 
\begin{theorem}[Tsirelson-Ibramov-Sudakov Inequality \cite{ibragimov1976norms}]
Let $\S \subset \R^d$ be bounded. Let $(V_s)_{s \in \S}$  be a Gaussian process such that $\E[V_s] = 0$ for all $s \in \S$. Define $\sigma^2 = \sup_{s \in \S} \E[V_s^2]$. Then, for all $u > 0$,
\begin{align*}
\P(|\sup_{s \in \S} V_s - \E \sup_{s \in \S} V_s| \geq u) \leq 2 \exp\left(\frac{-u^2}{2 \sigma^2}\right). 
\end{align*}
\end{theorem} 

Setting $\S = \Z$, we can apply this to the Gaussian process 
$    V_z := (z_{\ast} -z)^{\top}(\widehat{\theta}-\theta) = (z_{\ast} -z)^{\top} (\sum_{i=1}^n x_i x_i^\top)^{-1/2} \eta$
 where, again, $\eta \sim \mc{N}(0,I_d)$. 
We then have with probability at least $1-\delta$
\begin{align*}
    (z_{\ast} -z )^{\top}\widehat\theta\geq (z_{\ast} - z)^\top \theta & - \E_{\eta}\left[\sup_{z\in \Z \setminus \{z_*\}} (z_{\ast} - z)^{\top} A^{-1/2}\eta\right] 
    -\sqrt{2\sup_{z\in \Z } \|z_{\ast}-z\|_{A^{-1}}^2 \log(\tfrac{1}{\delta})}
\end{align*}
for all $z \in \mc{Z} \setminus \{z_*\}$ simultaneously.
This bound naturally breaks into two components. The second-term is the \emph{high-probability} term, and as the discussion above implies, naturally motivates the experimental design objective $ \min_{x_1, \cdots x_n} \max_{z\in \Z \setminus \{z_* \}} \|z_{\ast}-z\|_{(\sum_{i=1}^n x_i x_i^\t )^{-1}}^2 $ from past works on linear-bandit pure exploration. The first term, $\E_{\eta\sim N(0, I_d)}\left[\sup_{z\in \Z \setminus \{z_*\}} (z_{\ast} - z)^{\top} \left( \sum_{i=1}^n x_ix_i^\top\right)^{-1/2}\eta\right]$ is the \emph{Gaussian-width} of the set $\{\left( \sum_{i=1}^n x_ix_i^\top\right)^{-1/2} (z_{\ast}-z)\}_{z\in \Z \setminus \{z_*\}}$ \cite{vershynin2019high}. This term represents the penalty we pay for the union bound over the possible values of $\Z$ and reflects the underlying geometry of our arm set. 
For moderately sized values of $\delta \in (0,1)$ such as the science-stalwart $\delta=0.05$, the Gaussian width term can be substantially larger than the high probability term.
Analogous to above, this motivates choosing $x_1, \cdots, x_n$ to minimize the Gaussian width term.

\textbf{Relaxation to Continuous Experimental Designs.}
In practice, optimizing over all finite sets of $\mc{X}$ of size $n$ to minimize an experimental design objective is NP-hard.
Define $\simp := \{\lambda \in \R^{|\X|} : \sum_i \lambda_i = 1, \, \lambda_i \geq 0 \}$ to be the simplex over elements $\mc{X}$ and define $A(\lambda) = \sum_{x \in \X} \lambda_x xx^\t$ where $\lambda \in \simp$ denotes a convex combination of the measurement vectors.
Defining the design that minimizes the high probability term motivates the definition 
\begin{align*}
\rho^* & := \inf_{\lambda \in \simp} \rho^*(\lambda) \quad\quad \text{ where } \quad\quad \rho^*(\lambda) := \sup_{z \in \Z \setminus \{z_* \}} \frac{\norm{z^*-z}^2_{A(\lambda)^{-1}}}{\langle \theta, z^*-z \rangle^2}.
\end{align*}

On the other hand, minimizing the Gaussian width term motivates the definition
\begin{align*}
\gamma^* & := \inf_{\lambda \in \simp} \gamma^*(\lambda) \quad\quad \text{ where } \quad\quad \gamma^*(\lambda) := \E_{\eta \sim N(0,I)}[ \sup_{z \in \Z \setminus \{z_* \}} \frac{(z^*-z)^\t A(\lambda)^{-1/2} \eta}{ \theta^\t(z^*-z )}]^2.
\end{align*}

While the above suggests the importance of the quantities $\rho^*$ and $\gamma^*$, we will show later how they are intrinsic to the problem hardness.
For now, we point out that these quantities are easily relatable.

\begin{proposition}
\label{prop:upper_bound_gamma_by_rho}
There exists universal constants $c,c^\prime> 0$ such that for any $\mc{X}$ and $\Z$ we have 
$c \rho_* - \inf_{z \neq z_*} \inf_{\lambda \in \simp} \tfrac{\norm{z_* - z}^2_{A(\lambda)^{-1}}}{\langle \theta, z^*-z \rangle^2} \leq \gamma^* \leq \min(c^\prime \log(|\Z|) \rho^*,  d \rho^*). $
\end{proposition}

Typically, $\inf_{z \neq z_*} \inf_{\lambda \in \simp} \tfrac{\norm{z_* - z}^2_{A(\lambda)^{-1}}}{\langle \theta, z^*-z \rangle^2} \ll \rho^*$, in which case $\rho_* \lesssim \gamma_*$. While there are instances where $\gamma^\ast = \Theta(d \rho^*)$,  the upper bound is not necessarily tight.  
\begin{proposition}
\label{prop:gamma_rho_gap}
There exists an instance of transductive linear bandits where $\gamma^* \geq c d \rho^*$, and a separate instance for which $\gamma^* \leq c' \log(d) \rho^*$ where $c,c' >0$ are universal constants.
\end{proposition}

\section{Towards the true sample complexity}


This section formally justifies the quantities $\rho_*$ and $\gamma_*$ defined above.
The following result holds for any $\mc{X}$ and $\mc{Z}$ and was first proven in this generality in \cite{fiez2019sequential}, extending \cite{soare2014best, soare2015sequential, chen2017nearly}.
\begin{theorem}[Lower bound for any adaptive algorithm \citep{fiez2019sequential}]\label{thm:rho_lowerbound}
For any $\delta \in (0,1)$, any $\delta$-PAC algorithm wrt $(\X, \Z, \R^d)$ with stopping time $\tau$ satisfies
$\E_\theta[\tau] \geq \log(\frac{1}{2.4 \delta} ) \rho^*$.
\end{theorem}
Mirroring the approaches developed in  \cite{karnin2016verification,chen2017nearly, garivier2016optimal}, it is possible to develop an algorithm that satisfies $\lim_{\delta \rightarrow 0} \frac{\E_\theta[\tau]}{\log(\frac{1}{\delta} )} = \rho_*$, demonstrating the tightness of Theorem~\ref{thm:rho_lowerbound} in the regime of $\delta$ tending towards $0$.
However, for fixed $\delta \in (0,1)$, algorithms for linear bandits to date have only been able to match this lower bound up to additive factors of $d \rho_*$ or $\log(|\mc{Z}|) \rho_*$ \cite{karnin2016verification,fiez2019sequential} (note, this does not rule out optimality as $\delta \rightarrow 0$). In particular, the lower and the upper bounds of linear bandits do not reflect the underlying geometry of general sets $\mc{X}$ and $\mc{Z}$ in union bounds and are loose in general. For example, in the well-studied case of Top-K, these bounds do not capture some additive factors that are necessary and achievable in addition to $\rho_*$ alone \cite{simchowitz2017simulator,chen2017nearlyTopK}. 

As a step towards characterizing the true sample complexity, we next demonstrate a lower bound that incorporates the geometry of $\mc{X}$ and $\mc{Z}$ for, presumably,  the best possible \textit{non-interactive} algorithm.
Precisely, the procedure is given access to $\theta$, chooses an allocation $\{x_{I_1},x_{I_2},\dots\} \in \mc{X}$, then observes $\{y_{I_1},y_{I_2},\dots \} \in \R$ where $y_{I_t} \sim \mc{N}( \langle x_{I_t},\theta_* \rangle , 1)$, and finally forms the MLE $\widehat{\theta} = \arg\min_{\theta} \sum_{t} (y_{I_t} - \langle x_{I_t} ,\theta \rangle )^2$ and outputs $\widehat{z} = \argmax_{z \in \Z} \langle z, \widehat{\theta} \rangle$.
We emphasize that this procedure can pick any allocation it desires using full knowledge of $\theta$; in particular, it can use the allocation that achieves $\rho_*$.
\begin{theorem}[Lower bound for non-interactive MLE]
\label{thm:mle_lower_bound}
Let $\delta \in (0,1/8)$. Fix $\X, \Z \subset \R^d$ and $\Theta = \R^d$. Fix a problem $\theta \in \Theta$. Then, if the non-interactive MLE is $\delta$-PAC wrt $(\X, \Z, \Theta)$ where a different allocation can be used for each $\theta \in \Theta$, then it uses at least $c( \gamma^*+\rho^{\ast}\log(1/\delta))$ samples for the instance $(\X, \Z, \theta)$ where $c > 0$ is a universal constant.
\end{theorem}
By Proposition \ref{prop:gamma_rho_gap}, $\gamma^*$ can be larger than $\rho^*$ by a multiplicative factor of the dimension $d$, demonstrating that the lower bound of Theorem~\ref{thm:mle_lower_bound} can be much larger than the lower bound of Theorem \ref{thm:rho_lowerbound}.
While there exists problem instances in which the best known adaptive algorithm can achieve a sample complexity strictly smaller than the lower bound of Theorem~\ref{thm:mle_lower_bound} (e.g., best-arm identification), we are unaware of any settings in which the sample complexity of the best adaptive algorithm improves over Theorem~\ref{thm:mle_lower_bound} by more than a factor of $\log(d)$, which is typically considered insignificant.



\section{Fixed Confidence Setting Algorithms}

In this section, we present Algorithm~\ref{alg:action}, Peace, that achieves the state-of-the-art sample complexity for (transductive) linear bandits in the fixed confidence setting. 
In each round $k$ we eliminate from the set of candidates $\Z$ all the elements that are roughly $2^{-k}$ suboptimal. 
In each round the query allocation is fixed according to the best non-adaptive strategy. 

\begin{algorithm}[t]
\small
\textbf{Input: } Confidence level $\delta \in (0,1)$, rounding parameter $\epsilon \in (0,1)$ with default value of $\frac{1}{10}$\;
$\Z_1 \longleftarrow \Z$, $k \longleftarrow 1$, $\delta_k \longleftarrow \delta/2k^2$ \;
$B := \inf_{\lambda \in \simp} \E_{\eta \sim N(0,I)}[ \max_{z,z^\prime \in \Z} (z-z^\prime)^\t A(\lambda)^{-1/2} \eta]^2 +2 \log(\frac{1}{\delta_1}) {\max}_{z, z^\prime \in \Z} \norm{z-z^\prime}^2_{A(\lambda)^{-1}} \vee 1$\;
\While{$|\Z_k| > 1$}{
Let $\lambda_k$ and $\tau_k$ be the solution and value of the following optimization problem
\begin{equation*}
\hspace{-.05cm} \inf_{\lambda \in \simp} \tau(\lambda ; \Z_k) := \E_{\eta \sim N(0,I)}[ \max_{z,z^\prime \in \Z_k} (z-z^\prime)^\t A(\lambda)^{-1/2} \eta]^2 +2\log(\frac{1}{\delta_k}) \max_{z, z^\prime \in \Z_k} \norm{z-z^\prime}^2_{A(\lambda)^{-1}}\hspace{1cm} 
\end{equation*} 
\hspace{-.25cm} Set $N_k \longleftarrow \ceil{2(1+\epsilon)\tau_k(\frac{2^{k+1}}{B})^2} \vee q(\epsilon)$\ and find
$\{x_1, \ldots, x_{N_k} \} \longleftarrow \text{ROUND}(\lambda_k,N_k,\epsilon)$\;
Pull arms $x_1, \ldots, x_{N_k} $ and receive rewards $y_1, \ldots, y_{N_k}$\;
Let $\widehat{\theta}_k \longleftarrow (\sum_{s=1}^{N_k} x_s x_s^\t)^{-1} \sum_{s=1}^{N_k} x_s y_s$ \;
$\Z_{k+1} \longleftarrow \Z_k \setminus \{z \in \Z_k : \exists z^\prime \text{ such that } (z^\prime -z)^\t \widehat{\theta}_k - \frac{B}{2^{k+1}} \geq 0 \}$\;
$k \longleftarrow k+1$
}
  \Return $\Z_{k} = \{\widehat{z}\}$.
 \caption{Fixed Confidence Peace. See text for explanation of \text{ROUND} sub-routine.}
\label{alg:action}
\end{algorithm}

Our algorithm must round a design to an integral solution. It uses an efficient rounding procedure $\text{ROUND}(\lambda,N,\epsilon)$ that for $\lambda \in \simp$ and $N \geq q(\epsilon)$ returns $\kappa \in \N^{|\X|}$ such that $\sum_{x \in \X} \kappa_x = N$ and $\tau(\kappa; Z^\prime) \leq (1+\epsilon) \tau(N \lambda;Z^\prime)$ \citep{allen2020near}. It suffices to take $q(\epsilon) = O(d/\epsilon^2)$ (see the Supplementary Material). Define $S_k := \{z \in \Z : \theta^\t(z^*-z) \leq B 2^{-k}\}$, $\Delta_z := \theta^\t (z_* -z)$, and $\Delta_{min} := \min_{z \in \Z \setminus \{z_*\}} \Delta_z$. 

\begin{theorem}
\label{thm:up_bd}
With probability at least $1-\delta$, Algorithm \ref{alg:action} terminates and returns $z_*$ after a number of samples no more than
\begin{align*}
[& \gamma^*   +\rho^*  \log(\log(\tfrac{B}{\Delta_{min}})/\delta)]c\min(\log(\tfrac{B}{\Delta_{min}}),\log(\tfrac{B}{\min_{k : |S_k| > 1} \min_{\lambda \in \simp} \tau(\lambda; S_k)}))+ c d \log(\tfrac{B}{\Delta_{min}}).
\end{align*}
\end{theorem}

We note that while our upper bound has an extra additive factor of $d$ compared to the lower bound of Theorem~\ref{thm:mle_lower_bound}, this factor is necessary in many instances of interest (see the Supplementary Material). $\log(B) = O(\log(d))$ when $\X = \Z$ and in combinatorial bandits, and $B$ can be replaced by an upper bound on $\max_{z \in \Z} \Delta_z$ when one is known. $\tau(\lambda ; \Z_k)$ can be optimized using stochastic mirror descent; we show that after a suitable transformation, it is convex in the combinatorial bandit setting. We conjecture that it is convex in the general case, as well.


\subsection{Computationally Efficient Algorithm for Combinatorial Bandits}

A drawback of Algorithm~\ref{alg:action} is that it is computationally inefficient when $|\Z|$ is exponentially large in the dimension. 
In this section, we develop an algorithm for combinatorial bandits that is computationally efficient when the \textit{linear maximization oracle} defined in \eqref{eq:lin_max_oracle} is available. We introduce the following notation for a set $Z^\prime \subset \Z$:
\begin{equation}
\gamma(Z^\prime) := \min_{\lambda \in \simp} \E [\sup_{z,z^\prime \in Z^\prime} (z -z^\prime)^\t A^{-1/2} (\lambda) \eta]^2.
\end{equation}

\begin{algorithm}[h]\small
\textbf{Input:} Confidence level $\delta > 0$,  rounding parameter $\epsilon \in (0,1)$ with default value of $\frac{1}{10}$, $\alpha > 0$ ($\alpha = 42941$ suffices though this is wildly pessimistic; we recommend using $\alpha = 4$) \;
$\widehat{\theta}_0 = \zero \in \R^d$, $\Gamma \longleftarrow \gamma(\Z) \vee 1$, $\delta_k \longleftarrow \frac{\delta}{2k^3}$ \;
\For{$k=0,1,2,\ldots$}{
$\tilde{z}_{k} \longleftarrow \argmax_{z \in \Z} \widehat{\theta}_{k}^\t z$\;
Let $\lambda_k, \tau_k $ be the solution and value of the following optimization problem
\begin{align}
\inf_{\lambda \in \simp} \E_{\eta \sim N(0,I)}[ \max_{z \in \Z} \frac{(\tilde{z}_k- z)^\t A(\lambda)^{-1/2} \eta}{2^{-k} \Gamma+ \widehat{\theta}^\t_{k}(\tilde{z}_k - z) }]^2 \label{eq:action_comp_2}
\end{align}
\hspace{-.25cm} Set $N_k \longleftarrow \alpha \ceil{ \tau_k \log(1/\delta_k) (1+\epsilon)} \vee q(\epsilon)$ and find $\{x_1, \ldots, x_{N_k} \} \longleftarrow \text{ROUND}(\lambda_k,N_k)$\; 
Pull arms $x_1, \ldots, x_{N_k} $ and receive rewards $y_1, \ldots, y_{N_k}$\; 
Let $\widehat{\theta}_{k+1} \longleftarrow (\sum_{s=1}^{N_k} x_s x_s^\t)^{-1} \sum_{s=1}^{N_k} x_s y_s$ \;
\textbf{if} \textsc{Unique}$(\Z, \widehat{\theta}_k, 2^{-k} \Gamma)$ \textbf{then} \Return $\tilde{z}_k$
}
\caption{Fixed Confidence Peace with a linear maximization oracle.}
\label{alg:action_comp}
\end{algorithm}


Algorithm \ref{alg:action_comp}  proceeds by estimating the gap of each $z \in \mc{Z}$ at a progressively finer level of granularity. The objective in line \ref{eq:action_comp_2} is carefully designed to find a sample-efficient allocation that ensures that at round $k$, with high probability $\widehat{\theta}_k^\t(\tilde{z}_k-z) \approx \Delta_z$ for all $z \in \Z$ such that $\Delta_z \geq 2^{-k} \Gamma$. The subroutine \textsc{Unique}$(\Z, \widehat{\theta}_k, 2^{-k} \Gamma)$ uses calls to the linear maximization oracle to determine whether the gaps are sufficiently well-estimated to conclude that $\tilde{z}_k$ is $z_{\ast}$ (see the Supplementary Material).

In the Supplementary Material, we provide procedures for computing $\gamma(\Z)$ and \eqref{eq:action_comp_2} only using calls to the linear maximization oracle. The main challenge is to compute an unbiased estimate of the gradient of the objective in equation~\eqref{eq:action_comp_2} (for an appropriate first-order optimization procedure such as stochastic mirror descent), which we now sketch. Since the expectation in \eqref{eq:action_comp_2} is non-negative, it suffices to optimize the square root of the objective function in \eqref{eq:action_comp_2}. Writing $g(\lambda; \eta;z)=\tfrac{(\tilde{z}_k- z)^\t A(\lambda)^{-1/2} \eta}{2^{-k} \Gamma + \widehat{\theta}^\t_{k}(\tilde{z}_k - z) }$, since we may exchange the gradient with respect to $\lambda$ and the expectation over $\eta$, to obtain an unbiased estimate, it suffices to draw $\eta \sim N(0,I)$, and compute $\nabla_{\lambda} \max_{z \in \Z} g(\lambda;\eta;z)$. Since for a collection of differentiable functions $\{h_1, \ldots, h_l \}$, a sub-gradient $\nabla_y \max_i h_i(y)$ is simply $\nabla_y h_0(y)$ where $h_0(y) = \argmax_i h_i(y)$, it suffices to find $\argmax_{z \in \Z} g(\lambda;\eta;z)$. We reformulate this optimization problem as the following equivalent linear program:
\begin{align}
& \min_{s}  s 
& \text{subject to } \max_{z \in \Z} (\tilde{z}_k- z)^\t A(\lambda)^{-1/2} \eta - s[2^{-k} \Gamma + \widehat{\theta}^\t_{k}(\tilde{z}_k - z) ] \leq 0 \label{eq:lin_prog_const}
\end{align}
A call to the linear maximization oracle can check whether the constraint in \eqref{eq:lin_prog_const} is satisfied so the above linear program can be solved using binary search and multiple calls to the maximization oracle.
\begin{theorem}
\label{thm:sample_complexity_compute}
Consider the combinatorial bandit setting. With probability at least $1-4 \delta$ Algorithm \ref{alg:action_comp} terminates and returns $z_*$ after at most
\begin{align*}
[& (\gamma^* + \rho^*)   \log(\log(\gamma(\Z)/\Delta_{min})/\delta) + d]c\log(\gamma(\Z) /\Delta_{min})
\end{align*}
samples and if $\delta \in (\frac{1}{2^d},1)$, then with probability at least $1-4\delta$, the number of oracle calls is upper bounded by 
\begin{align*}
\tilde{O}([d+\log(\max_{z \in \Z} \Delta_z + \Gamma  ) + \log( \frac{d}{\Delta_{min} \delta} ) ] \log(d)^2 \frac{d^3}{\Delta_{min}^2} \frac{\log(\gamma(\Z)/\Delta_{min})^5}{\delta^2})
\end{align*}
\end{theorem}

Theorem~\ref{thm:sample_complexity_compute} nearly matches the sample complexity of Theorem~\ref{thm:up_bd}. 
The latter scales like $\gamma_* + \rho_* \log(1/\delta)$ whereas the former scales like $(\gamma_* + \rho_*) \log(1/\delta)$, reflecting a tradeoff of statistical efficiency for computational efficiency. 
It is unknown if this tradeoff is necessary. 


\section{Fixed Budget Setting}

Next, we turn to the fixed budget setting, where the goal is to minimize the probability of returning a suboptimal item $z \in \Z \setminus \{z_* \}$ given a budget of $T$ total measurements. Algorithm \ref{alg:fixed_budget} is a generalization of the successive halving algorithm \citep{icml2013_karnin13} and the first algorithm for fixed-budget linear bandits. It divides the budget into equally sized epochs and progressively shrinks the set of candidates $\Z_k$. In each epoch, it computes a design that minimizes $\gamma(\Z_k)$ and samples according to a rounded solution. At the end of an epoch, it sorts the remaining items in $\Z_k$ by their estimated rewards and eliminates enough of the items with the smallest estimated rewards to ensure that $\gamma(\Z_{k+1}) \leq \frac{\gamma(\Z_k)}{2}$.

\begin{algorithm}[t]
\small
\textbf{Input: } $\epsilon \in (0,1)$ with default value of $\frac{1}{10}$, budget $T$ such that $T \geq q(\epsilon) \ceil{ \log_2( \gamma(\Z) )}$  \; 
$R  \longleftarrow \ceil{ \log_2( \gamma(\Z) )}$, $N \longleftarrow \floor{T/R}$, $\Z_0 \longleftarrow \Z$, $k \longleftarrow 0$\;
\While{$k \leq R$ \textbf{\emph{and}} $|\Z_k|>1$}{
Let $\lambda_k$ achieve the minimum in $\gamma(\Z_k)$ and find $\{x_1, \ldots, x_{N} \} \longleftarrow \text{ROUND}(\lambda_k,N,\epsilon)$\;
Pull arms $x_1, \ldots, x_{N} $ and obtain rewards $y_1, \ldots, y_{N}$\;
Set $\widehat{\theta}_{k} \longleftarrow (\sum_{s=1}^{N} x_s x_s^\t)^{-1} \sum_{s=1}^{N} x_s y_s$ \;
Compute an ordering $\pi_k$ over $\Z_k$ such that $\ip{\widehat{\theta}_k, z_{\pi_k(i)} - z_{\pi_k(i+1)} } \geq 0$ for all $i$ \;
Let $i_{k+1}$ be the largest integer for which $\gamma(\{ z_{\pi_k(1)}, \ldots, z_{\pi_k(i_{k+1})} \}) \leq \gamma(\Z_k) / 2$ \;
$\Z_{k+1} \longleftarrow \{ z_{\pi_k(1)}, \ldots, z_{\pi_k(i_{k+1})} \}$\;
$k \longleftarrow k+1$\;
}
\Return $ \argmax_{i \in \Z_k} \widehat{\theta}_k^\t z_i$
 \caption{Fixed Budget Peace}
 \label{alg:fixed_budget}
\end{algorithm}

\begin{theorem}
\label{thm:fixed_budget_upper_bound}
Suppose that $\gamma(\{z, z_*\}) \geq 1$ for all $z \in \Z \setminus \{z_*\}$. Then, if $T \geq c\max([\rho^* + \gamma^*],d) \log(\gamma(\Z)) $, Algorithm \ref{alg:fixed_budget}  returns $\widehat{z}\in \Z$ such that
\begin{align*}
\P(\widehat{z} \neq z_*) & \leq 2 \ceil{\log(\gamma(\Z))} \exp(- \frac{T}{c^\prime [\rho^* + \gamma^*]\log(\gamma(\Z)) }).
\end{align*}
\end{theorem}

We note that the combinatorial bandit setting satisfies the assumption that $\gamma(\{z, z_*\}) \geq 1$ for all $z \in \Z \setminus \{z_*\}$, but this lower bound is unessential and the algorithm can be modified to accommodate another lower bound. Theorem~\ref{thm:fixed_budget_upper_bound} implies that if $T \geq O(\log(1/\delta)[\rho^* + \gamma^*] \log(\gamma(\Z)) \log(\log(\gamma(\Z))) )$, then Algorithm \ref{alg:fixed_budget} returns $z_*$ with probability at least $1-\delta$. Finally, $\log(\gamma(\Z))$ is $O(\log(d))$ in many cases, e.g., combinatorial bandits and in linear bandits when $\X = \Z$.

\section{Discussion and Prior Art}
\label{sec:related_work}


\textbf{Transductive Linear Bandits: } There is a long line of work in pure-exploration linear bandits \citep{soare2014best, xu2018fully, tao2018best} culminating in the formulation of the transductive linear bandit problem in \citep{fiez2019sequential}
where the authors developed the first algorithm to provably achieve $\rho^* \log(|\Z|/\delta)$.
The sample complexity of Theorem~\ref{thm:up_bd}, $\gamma_* + \rho_* \log(1/\delta)$, is never worse than \cite{fiez2019sequential} since $\gamma_* \leq \rho_* \log(|\Z|)$ by Proposition~\ref{prop:upper_bound_gamma_by_rho}. 
On the other hand, it is possible to come up with examples where $\gamma_*$ does not scale with $|\mc{Z}|$, but just $\rho_*$ (see experiments).
While our algorithms work for arbitrary $\mc{X},\mc{Z} \subset \R^d$, problem instances of combinatorial bandits most clearly illustrate the advances of our new results over prior art.


\textbf{Combinatorial Bandits: } The pure exploration combinatorial bandit was introduced in \citep{chen2014combinatorial}, and followed by \citep{gabillon2016improved}.
These papers are within a $\log(d)$ factor of the lower bound for the setting where $\Z$ is a matroid.
If $\tilde{\Delta}_i = \theta^\t z_* - \max_{z \in \Z : i  \in z} \theta^\t z$ when $i \not \in z_*$  and $\theta^\t z_* - \max_{z \in \Z : i  \not \in z} \theta^\t z$ otherwise,
then a lower bound is known to scale as $\sum_{i=1}^d \tilde{\Delta}_i^{-2} \log(1/\delta)$.
The following result shows that $\gamma^*$ is within $\log(d)$ of the lower bound, implying that our sample complexity scales as $\sum_{i=1}^d \tilde{\Delta}_i^{-2} \log(d/\delta)$.
\begin{proposition}
\label{prop:gamma_matroid}
Consider the combinatorial bandit setting and suppose that $\Z$ is a matroid. Then,
$\gamma^* \leq c \log(d)  \sum_{i=1}^d \tilde{\Delta}_i^{-2}$ for some absolute constant $c$. 
\end{proposition}

However, in the general setting where $\Z$ is not necessarily a matroid, \cite{chen2017nearly} points out a class with $|\Z|=2$ 
where the sample complexity of \citep{chen2014combinatorial,gabillon2016improved} is loose by a multiplicative factor of $d$.
\citet{chen2017nearly} was the first to provide a lower bound equivalent to $\rho^{\ast}\log(1/\delta)$ for the general combinatorial bandit problem, as well as an upper bound of $\rho^* \log(|\Z|/\delta)$. 
However, as stressed in the current work, the $\log(|\Z|)$ term is not necessary in many scenarios; for example, in Top-K, $\rho^* \log(|\Z|)$ is larger than the best achievable sample complexity by a multiplicative factor of $k$ \cite{chen2017nearlyTopK,simchowitz2017simulator}. This is not in contradiction with the lower bound provided in Theorem 1.9 of \cite{chen2017nearly} which provides a specific worst-case class of instances where the $\log(|\Z|)$ is needed. 

The next technological leap in combinatorial bandits is the algorithm of \citep{cao2019disagreement} (and the follow-up \cite{jain2019new}). They provided an algorithm with a novel sample complexity that replaces $\log(|\Z|)$ with a more geometrically inspired term. Define the sphere $B(z,r) = \{z^\prime \in \Z : \norm{z-z^\prime}_2 = r \}$, and the complexity parameter 
$\varphi_i := \max_{z \in \Z \setminus \{z_* \}  : i \in z^* \Delta z} \frac{ \norm{z-z^\prime}_2^2 \log(d |B(z_*,  \norm{z-z^\prime}_2)|)}{\Delta_z^2}$.
Then \cite{cao2019disagreement} provide a sample complexity scaling like $\varphi^* :=\sum_{i=1}^n \varphi_i$. The following  shows that $\gamma^{\ast}$ is never more than $\log\log(d)$ larger than this complexity.
\begin{proposition}
\label{prop:gamma_previous_combi_results}
Consider the combinatorial bandit setting. Then, $\gamma^* \leq c \varphi^* \log( \log(d))$. 
\end{proposition}

However, for even these sample complexity results that take the geometry into account, there exist clear examples of looseness that our approach avoids.

\begin{proposition}
\label{prop:cao_jain_counter}
There exists an instance of Top-K where $\varphi^* = \Omega(k \log(d)\rho^*)$ but $\gamma_* = O(\log(d) \rho^*  )$.
\end{proposition}

In summary, we have the first algorithm with a sample complexity that simultaneously is nearly optimal for matroids, essentially matches our novel lower bound $\gamma^* + \log(1/\delta) \rho^* \leq \log(|\Z|/\delta) \rho^*$, and is never worse than the sample complexity $\varphi^*$ from \citep{cao2019disagreement,jain2019new}.

\textbf{Computational Results in Combinatorial Bandits: } 
The algorithm CLUCB from \cite{chen2014combinatorial} is computationally efficient and user-friendly.  \cite{cao2019disagreement} and \cite{chen2017nearly} provide computationally efficient algorithms, but these are impractical and we are unaware of any implementations of these algorithms.

\section{Experiments}

\textbf{Combinatorial Bandits:} We compare Algorithm \ref{alg:action_comp} against a uniform allocation strategy (UA) and CLUCB from \citep{chen2014combinatorial}.
We use $\delta = 0.05$ on all the experiments and the empirical probability of failure never exceeded $\delta$ in all of our experiments. We consider two combinatorial structures.\footnote{We also outperform \citep{cao2019disagreement} on a bi-clique problem (see the Supplementary Material).}  \emph{(i)} \emph{Matching:} we use a balanced complete bipartite graph $G = (U \cup V, E)$ where $|U| = |V|=14$. Note that $|\Z| = 14! \geq 8\cdot 10^{10}$. We took two disjoint matchings $M_1$ and $M_2$ and set $\theta_e = 1$ if $e \in M_1$ and $\theta_e = 1-h$ if $e \in M_2$ for $h \in \{.15,.1,.05,.025\}$. Otherwise, $\theta_e = 0$. \emph{(ii) Shortest Path:} we consider a DAG where a source leads into two disjoint feed-forward networks with 26 width-2 layers that then lead into a sink (see Figure \ref{fig:shortest_path_vis} for an illustration). Note that $|\Z| \geq 10^8$. We consider two paths $P_1$ and $P_2$ such that they are in the disjoint feed-forward networks. We set $\theta_e = 1$ if $e \in P_1$ and $\theta_e = 1-h$ if $e \in P_2$ for $h \in \{.2,.15,.1,.05\}$. Otherwise, $\theta_e = -1$.
\begin{wrapfigure}{r}{0.3\textwidth}
    \includegraphics[width=.3\textwidth]{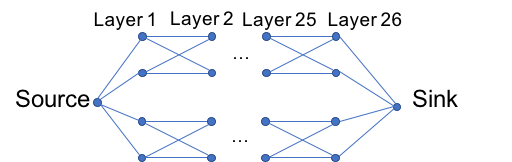}
    \caption{\small Shortest Path Problem}
    \label{fig:shortest_path_vis}
\end{wrapfigure}

The first two panels of Figure \ref{fig:experiments} depict ratio of the average performance of the competing algorithms to the average performance of our algorithm. 
In the matching experiment, as the gap between the best matching $M_1$ and the second best matching $M_2$ get smaller, CLUCB pays a cost of roughly $|U|/h^2$ to distinguish $M_1$ from $M_2$ whereas our algorithm pays a cost of roughly $1/h^2$. A similar phenomenon occurs in the shortest path problem.

\begin{figure}[!tbp]
\includegraphics[width=\textwidth]{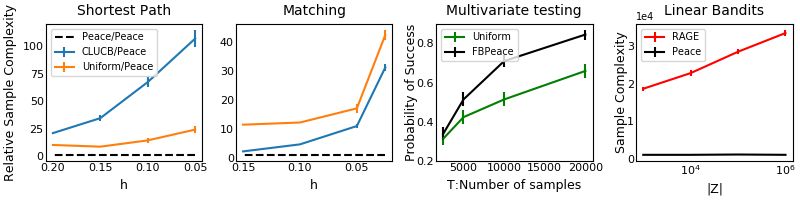}
\caption{Panels (i) and (ii) depict the relative performance of CLUCB and UA to PEACE. Panel (iii) depicts the relative performance of uniform sampling to FBPeace. Panel (iv) compares the performance of RAGE to Peace on the linear bandits experiment.}
\label{fig:experiments}
\end{figure}

\textbf{Multivariate Testing} We consider multivariate testing \cite{hill2017efficient, heller1986statistics} in which there are $d$ options, each having $k$ possible levels. For example, consider determining the optimal content for a display-ad with slots such as headline, body, etc. and each slot has several variations. A layout is specified by a $d$-tuple $f = (f_1, \cdots, f_d)\in \{1,\cdots, k\}^d$ indicating the level chosen for each option. 
For each option $I$, $1\leq I\leq d$ and level $f$, $1\leq f\leq k$, there is a weight $W^{I}_f\in \mathbb{R}$, and for each pair of options $I,J$ and factors $f_I, f_J$, there is a weight $W^{I,J}_{f_I, f_J}\in \mathbb{R}$ capturing linear and quadratic interaction terms respectively.
The total reward of a layout $f=(f_1, \cdots, f_d)$ is given by $W_0 + \sum_{I=1}^{d} W^I_{f_I} + \sum_{I=1}^d\sum_{J=1}^d W^{I,J}_{f_I, f_J}$.
The fixed budget experiment in Figure~\ref{fig:experiments} considers a scenario when $k=6$ and $d=3$ and compares Algorithm~\ref{alg:fixed_budget} (FBPeace) to uniform sampling. We set $W^{1,2}_{1,1} = .8$ and $W^{2,3}_{1,1} = .1$ and all other weights to zero, capturing a setting where the three options must be synchronized. At 10000 samples, FBPeace is 30\% more likely to return the true optimal layout.

\textbf{Linear Bandits.} We considered a setting in $\mathbb{R}^2$, where $\X = \{\mb{e}_1, \cos(3\pi/4)\mb{e}_1+\sin(3\pi/4)\mb{e}_2\}$ and $\mc{Z} = \{\cos(\pi/4+\phi_i) \mb{e}_1+\sin(\pi/4+\phi_i) \mb{e}_2\}_{i=1}^n$ where $\phi_i\sim \text{Uniform}([0,.05])$. The parameter vector is fixed at $\theta = \mb{e}_1$. In Figure~\ref{fig:experiments} we see that as the number of arms increases (from $10^3$ to $10^6$), the number of samples by our algorithms is constant, yet grows linearly in $\log(|\Z|)$  for RAGE \cite{fiez2019sequential}. This reflects the main goal of the paper - optimal union bounding for large classes.

\clearpage
\bibliography{refs}
\clearpage

\pagebreak

\appendix

\section{Outline and Notation}

Section \ref{sec:main_lb_proof} gives the proof of Theorem \ref{thm:mle_lower_bound}. Section \ref{sec:fixed_con_ub_proofs} presents proofs of the two results for the fixed confidence setting. Section \ref{sec:comp_results_alg} proves provides the main results on the computational efficiency of Algorithm \ref{alg:action_comp}. Section \ref{sec:fb_proof} provides the proof of our upper bound for the fixed budget setting. Section \ref{sec:gammastar_results} proves various results related to $\gamma^*$. Section \ref{sec:additional_lower_bounds} gives additional lower bounds for the transductive linear bandit problem. Section \ref{sec:rounding} provides a discussion of rounding. Section \ref{sec:gammastar_technical} presents technical lemmas. Section \ref{sec:comp_effic} discusses the convexity of $\gamma^*$. Section \ref{sec:comp_results} discusses the sample complexity results of other papers. Section \ref{sec:experimental_details} gives further details on the other experiments.

For the combinatorial bandit setting, we assume wlog that for all $i \in [d]$ there exist $z,z^\prime \in \Z$ such that $i \in z$ and $i \not \in z^\prime$. We will sometimes write $z \cap z^\prime$ to denote $(z_1 \cdot z_1^\prime, \ldots, z_d \cdot z_d^\prime)^\t$. In a similar way, we will use $z \Delta z^\prime$ to denote the symmetric difference of $z$ and $z^\prime$, viewed as sets. We use $c,c^\prime, \cdots$ to denote positive universal constants whose values may change from line to line.

\section{Proof of Theorem \ref{thm:mle_lower_bound}}
\label{sec:main_lb_proof}

\begin{proof}[Proof of Theorem \ref{thm:mle_lower_bound}]
Define $\Delta_z =\theta^\t(z^*-z)$. Let $\X = \{x_1,\ldots, x_m \}$. Fix $\{x_{I_1}, \ldots, x_{I_T}\} \subset \X$ to be the measurement vectors pulled by the algorithm. Define the matrix 
\begin{align*}
X= \begin{pmatrix}
x_{I_1}^\t \\
\vdots \\
x_{I_T}^\t 
\end{pmatrix} \\
\end{align*}
Define $\widehat{\theta} = (X^\t X)^{-1} X^\t Y$.

Let $\lambda \in \simp$ be the associated allocation: $\lambda_i = \frac{1}{T}\sum_{s=1}^T \one\{I_s = i \}$. Note that
\begin{align*}
\E_{\eta \sim N(0,I)}[ \sup_{z \in \Z \setminus \{z_*\}} \frac{(z^*-z)^\t (X^\t X)^{-1/2}  \eta}{\Delta_z}] & = \frac{1}{\sqrt{T}} \E_{\eta \sim N(0,I)}[ \sup_{z \in \Z \setminus \{z_*\}} \frac{(z^*-z)^\t A(\lambda)^{-1/2}  \eta}{\Delta_z}]
\end{align*}
and
\begin{align*}
\sup_{z \in \Z \setminus \{z_*\}} \frac{\norm{z^*-z}_{(X^\t X)^{-1}}}{\Delta_z} & = \frac{1}{\sqrt{T}} \sup_{z \in \Z \setminus \{z_*\}} \frac{\norm{z^*-z}_{A(\lambda)^{-1}}}{\Delta_z}.
\end{align*}

Since $x_{I_1}, \ldots, x_{I_T}$ are fixed, the same argument from the proof of Theorem 1 in \citep{fiez2019sequential} implies that since $\A$ is $\delta$-PAC, 
\begin{align}
T \geq \log(\frac{1}{2.4 \delta} )  \sup_{z \in \Z \setminus \{z_*\}} \frac{\norm{z^*-z}^2_{A(\lambda)^{-1}}}{\theta^\t(z^*-z)^2} \geq \sup_{z \in \Z} \frac{\norm{z^*-z}^2_{A(\lambda)^{-1}}}{\theta^\t(z^*-z)^2} . \label{eq:mle_lb_1}
\end{align}
where the second inequality follows since $\delta \in (0,1/8)$. Next, we will show that that 
\begin{align}
T \geq  \frac{1}{4}\E_{\eta \sim N(0,I)}[ \sup_{z \in \Z \setminus \{z_*\}} \frac{(z^*-z)^\t A(\lambda)^{-1/2} \eta}{\theta^\t(z^*-z)}]^2. \label{eq:ml_lb_2}
\end{align}
Note that $\widehat{\theta} \sim N(\theta, (X^\t X)^{-1})$ so that
\begin{align*}
\E \sup_{z \in \Z \setminus \{z_*\}} \frac{(z-z_*)^\t (\widehat{\theta}-\theta)}{\Delta_z}  &= \E_{\eta \sim N(0,I)}[ \sup_{z \in \Z \setminus \{z_*\}} \frac{(z-z_*)^\t (X^\t X)^{-1/2} \eta}{\Delta_z}]\\
& = \E_{\eta \sim N(0,I)}[ \sup_{z \in \Z \setminus \{z_*\}} \frac{(z_*-z)^\t (X^\t X)^{-1/2} \eta}{\Delta_z}]
\end{align*}
where we used the fact that $(z_*-z)^\t (X^\t X)^{-1/2}  \eta$ and $(z-z_*)^\t (X^\t X)^{-1/2}  \eta$ are equal in distribution. 

By Theorem 5.8 in \citep{boucheron2013concentration}, with probability at least $1-e^{-1/2}$, 
\begin{align*}
 \E_{\eta \sim N(0,I)}[ & \sup_{z \in \Z \setminus \{z_*\}} \frac{(z^*-z)^\t (X^\t X)^{-1/2} \eta}{\Delta_z}]-\sup_{z \in \Z \setminus \{z_*\}} \frac{(z-z_*)^\t (\widehat{\theta}-\theta)}{\Delta_z} \\
 & \leq \sup_{z \in \Z \setminus \{z_*\}} \frac{\norm{z^*-z}_{(X^\t X)^{-1}}}{\Delta_z} 
\end{align*}
Towards a contradiction, suppose that inequality \eqref{eq:ml_lb_2} does not hold. Then, with probability at least $1-e^{-1/2}$ we have
\begin{align}
\sup_{z \in \Z} & \frac{(z-z_*)^\t (\widehat{\theta}-\theta)}{\Delta_z} \nonumber \\
& \geq \E_{\eta \sim N(0,I)}[ \sup_{z \in \Z \setminus \{z_*\}} \frac{(z^*-z)^\t (X^\t X)^{-1/2}  \eta}{\Delta_z}] -\sup_{z \in \Z \setminus \{z_* \}} \frac{\norm{z^*-z}_{(X^\t X)^{-1}}}{\Delta_z} \nonumber \\
& = \frac{1}{\sqrt{T}}\E_{\eta \sim N(0,I)}[ \sup_{z \in \Z \setminus \{z_*\}} \frac{(z^*-z)^\t A(\lambda)^{-1/2}  \eta}{\Delta_z}] -\frac{1}{\sqrt{T}}\sup_{z \in \Z \setminus \{z_*\}} \frac{\norm{z^*-z}_{A(\lambda)^{-1}}}{\Delta_z} \nonumber \\
& \geq \frac{1}{\sqrt{T}}\E_{\eta \sim N(0,I)}[ \sup_{z \in \Z \setminus \{z_*\}} \frac{(z^*-z)^\t A(\lambda)^{-1/2}  \eta}{\Delta_z}] - 1 \label{eq:mle_lb_rho_lower_bound} \\
& >  1 \label{eq:mle_lb_contradiction}
\end{align}
where inequality \eqref{eq:mle_lb_rho_lower_bound} follows from inequality \eqref{eq:mle_lb_1} and inequality \eqref{eq:mle_lb_contradiction} follows from the negation of inequality \eqref{eq:ml_lb_2}. Rearranging the above inequality, if \eqref{eq:ml_lb_2} fails to hold, then there exists a $z \in \Z \setminus \{z_*\}$ such that 
\begin{align*}
(z-z_*)^\t \widehat{\theta} > 0.
\end{align*}
But, this contradicts the assumption that $\A$ is $\delta$-PAC with $\delta \in (0,1/8)$. Thus, \eqref{eq:ml_lb_2} holds. 

We conclude that
\begin{align*}
T & \geq \frac{1}{2}\left(\frac{1}{4}\E_{\eta \sim N(0,I)}[ \sup_{z \in \Z \setminus \{z_*\}} \frac{(z^*-z)^\t A(\lambda)^{-1/2}  \eta}{\Delta_z}]+  \log(\frac{1}{2.4 \delta} )  \sup_{z \in \Z \setminus \{z_*\}} \frac{\norm{z^*-z}^2_{A(\lambda)^{-1}}}{\theta^\t(z^*-z)^2}\right) \\
& \geq \frac{1}{2}\left[\frac{1}{4}\gamma^* + \log(\frac{1}{2.4 \delta} ) \rho^*\right]
\end{align*}

\end{proof}

\section{Fixed Confidence Upper Bound Proofs}
\label{sec:fixed_con_ub_proofs}

\subsection{Peace Algorithm Proofs}

\begin{proof}[Proof of Theorem \ref{thm:up_bd}]
\textbf{Step 1: Define a good event.} Define $\delta_k = \frac{\delta}{k^2}$. Let $x_1, \ldots, x_{N_k}$ denote the pulled measurement vectors in round $k$. By Theorem 5.8 in \citep{boucheron2013concentration}, with probability at least $1-\frac{\delta}{k^2}$
\begin{align}
&\sup_{z,z^\prime \in \Z_k} |(z-z^\prime)^\t(\widehat{\theta}_k-\theta)| \nonumber \\
&\leq \E[ \sup_{z,z^\prime \in \Z_k} (z-z^\prime)^\t(\widehat{\theta}_k-\theta)] + \sqrt{ 2\log(2k^2/\delta)\max_{z, z^\prime \in \Z_k} \norm{z-z^\prime}^2_{(\sum_{i=1}^{N_k} x_i x_i^\t)^{-1}}} \nonumber \\
&= \E_{\eta \sim N(0,I)}[ \sup_{z,z^\prime \in \Z_k} (z-z^\prime)^\t (\sum_{i=1}^{N_k} x_i x_i^\t)^{-1/2}\eta ] + \sqrt{ 2\log(2k^2/\delta)\max_{z, z^\prime \in \Z_k} \norm{z-z^\prime}^2_{(\sum_{i=1}^{N_k} x_i x_i^\t)^{-1}}} \nonumber  \\
&\leq \sqrt{\frac{(1+\epsilon)}{N_k}} \Big( \E_{\eta \sim N(0,I)}[ \sup_{z,z^\prime \in \Z_k} (z-z^\prime)^\t A(\lambda_k)^{-1/2}\eta ] \nonumber \\
& \hspace{2cm} + \sqrt{ 2\log(2k^2/\delta)\max_{z, z^\prime \in \Z_k} \norm{z-z^\prime}^2_{A(\lambda_k)^{-1}}} \Big) \label{eq:fc_ub_rounding} \\
& \leq \sqrt{\frac{2(1+\epsilon)\tau_k}{N_k}} \label{eq:fc_ub_sqrt_2}
\end{align}
where inequality \eqref{eq:fc_ub_rounding} follows by the guarantee on the the rounding subroutine ROUND and Lemma \ref{lem:rounding_lemma}, and the line \eqref{eq:fc_ub_sqrt_2} uses $\sqrt{a} + \sqrt{b} \leq \sqrt{2a + 2b}$ and the definition of $\tau_k$. Define the events
\begin{align*}
\mc{E}_k & = \{ \sup_{z,z^\prime \in \Z_k} |(z-z^\prime)^\t(\widehat{\theta}_k-\theta)| \leq \sqrt{\frac{2(1+\epsilon)\tau_k}{N_k}} \} \\
\mc{E} & = \cap_{k=1}^\infty \mc{E}_k.
\end{align*}
Note that line \eqref{eq:fc_ub_sqrt_2} implies that $\P(\mc{E}_k) \geq 1 - \frac{\delta}{k^2}$. Thus, we have
\begin{align*}
\P(\mc{E}) = \prod_{k=1}^\infty \P(\mc{E}_k | \cap_{l=1}^{k-1} \mc{E}_l) \geq \prod_{k=1}^\infty (1-\frac{\delta}{k^2}) = \frac{\sin(\pi \delta)}{\pi \delta} \geq 1 -\delta
\end{align*}
where the last line used $\delta \in (0,1)$. We suppose $\mc{E}$ holds for the remainder of the proof. 

\textbf{Step 2: Correctness.} Define $S_k := \{z \in \Z : \theta^\t(z^*-z) \leq B 2^{-k}\}$. We show that $z_* \in \Z_k$ and $\Z_k \subset S_k$ for $k=2,3 \ldots $. Using the event $\mc{E}$, we have that
\begin{align*}
\sup_{z,z^\prime \in \Z_1} |(z-z^\prime)^\t(\widehat{\theta}_1-\theta)|& \leq  \sqrt{\frac{2(1+\epsilon)\tau_1}{N_1}} \leq  \frac{B}{4}
\end{align*}
where we used $N_k \geq 2\tau_k(\frac{2^{k+1}}{B})^2(1+\epsilon)$. First, fix any $z \not \in S_1$. We will then show that $z \not \in \Z_1$. By definition, $\theta^\t(z_*-z) \geq \frac{B}{2}$. Note that
\begin{align*}
(z_* - z)^\t \widehat{\theta}_1 - B 2^{-2} & = (z_* - z)^\t( \widehat{\theta}_1-\theta) + \theta^\t(z_*-z) - B 2^{-2}  \\
& \geq (z_* - z)^\t( \widehat{\theta}_1-\theta) + \frac{B}{4} \\
& \geq - \frac{B}{4} + \frac{B}{4} \\
& \geq 0
\end{align*}
where we applied the assumption that $z \in S_1$ and the event. Thus, by the elimination rule, $z \not \in \Z_1$.

Now, we show that $z_* \in \Z_1$. Let $z \in \Z_0$. Then, using the event we have that
\begin{align*}
(z - z_*)^\t \widehat{\theta}_1 - B 2^{-2} & = (z - z_*)^\t( \widehat{\theta}_1-\theta) + \theta^\t(z_*-z) - B 2^{-2}  \\
& <  (z - z_*)^\t( \widehat{\theta}_1-\theta) - B 2^{-2} \\
& \leq \frac{B}{4} - \frac{B}{4} \\
& = 0.
\end{align*}
This proves the base case.

Next, we prove the inductive step. Suppose that $\Z_{k-1} \subset S_{k-1}$; we show that $\Z_{k} \subset S_{k}$. For any $z,z^\prime \in \Z_{k-1}$, 
\begin{align*}
|(z-z^\prime)^\t(\widehat{\theta}_{k-1}-\theta)| & \leq    \sqrt{\frac{2(1+\epsilon)\tau_k}{N_k}} \\
& \leq B2^{-(k+1)} .
\end{align*}
Let $z \in S_{k}^c$ so that $\theta^\t(z^*-z) > B 2^{-k}$. Then,
\begin{align*}
(z^*-z)^\t\widehat{\theta}_{k-1} -  B2^{-(k+1)} & = (z^*-z)^\t(\widehat{\theta}_{k-1} -\theta) +(z^*-z)^\t\theta  - B2^{-(k+1)} \\
& \geq (z^*-z)^\t(\widehat{\theta}_{k-1} -\theta) + B2^{-(k+1) } \\ 
&\geq - B2^{-(k+1) } + B2^{-(k+1) } \\
&= 0
\end{align*}
Thus, $z \not \in  \Z_k$, proving one part of the inductive step. 

Next, we show $z_* \in \Z_k$. By the inductive hypothesis, $z_* \in \Z_{k-1}$. Let $z \in  \Z_{k-1}$. Then, 
\begin{align*}
(z-z_*)^\t\widehat{\theta}_{k-1} -  B2^{-(k+1)} & = (z-z_*)^\t(\widehat{\theta}_{k-1} -\theta) +(z-z^*)^\t\theta  - B2^{-(k+1)} \\
& < (z-z_*)^\t(\widehat{\theta}_{k-1} -\theta) - B2^{-(k+1)} \\
& \leq B2^{-(k+1)}- B2^{-(k+1)} \\
&=0
\end{align*}

\textbf{Step 3: Upper bounding the sample complexity.} Now, we bound the number of samples taken until the algorithm terminates. Since $\Z_k \subset S_k$ for $k=2,3 \ldots $ as we showed in the previous step, once $k\geq c \log(B/ \Delta_{min})$, we have that $\Z_k = \{z_*\}$ and thus there are at most $c \log(B /\Delta_{min})$ rounds. In round $k$, the algorithm takes $N_k = \ceil{2\tau_k(\frac{2^{k+1}}{B})^2(1+\epsilon)} \vee q(\epsilon)$ samples and, thus, the sample complexity is bounded by the following sum
\begin{align}
\sum_{k=1}^{c \log(B/\Delta_{min})} N_k \leq c^\prime [ \log(B/\Delta_{min})d + \sum_{k=1}^{c \log(B/\Delta_{min})} \tau_k(\frac{2^{k}}{B})^2] \label{eq:fc_ub_main}
\end{align}
where we used $q(\epsilon) = O(d)$ by the guarantees on the rounding procedure and $\epsilon =1/10$. Now, we focus on upper bounding the second term in the above expression. For $k=1$, then
\begin{align}
\tau_1(\frac{2^{1}}{B})^2 & \leq \frac{c}{B} \leq c^\prime \label{eq:fc_ub_k_1}
\end{align}
where we used the relation $B = \tau_1 \vee 1 $. 

Next, we bound the terms $k > 1$. Note that 
\begin{align*}
\tau_k(\frac{2^{k}}{B})^2 =& \E_{\eta \sim N(0,I)}[ \max_{z,z^\prime \in \Z_k} (z-z^\prime)^\t A(\lambda)^{-1/2} \eta]^2(\frac{2^{k}}{B})^2 \\
& +2\log(\frac{1}{\delta_k}) \max_{z, z^\prime \in \Z_k} \norm{z-z^\prime}^2_{A(\lambda)^{-1}}(\frac{2^{k}}{B})^2
\end{align*}
We begin by bounding the second term. Fix $\lambda$. Then, 
\begin{align}
\max_{z, z^\prime \in \Z_k} \norm{z-z^\prime}^2_{A(\lambda)^{-1}}(\frac{2^{k}}{B})^2 & \leq \max_{z, z^\prime \in S_k} \norm{z-z^\prime}^2_{A(\lambda)^{-1}}(\frac{2^{k}}{B})^2 \label{eq:fc_ub_rho_bound_1} \\
& \leq c \max_{z \in S_k \setminus \{z_* \}} \norm{z_*-z}^2_{A(\lambda)^{-1}}(\frac{2^{k}}{B})^2 \label{eq:fc_ub_rho_bound_2}\\
&  \leq c \max_{z \in \Z \setminus \{z_*\}} \frac{\norm{z^*-z}^2_{A(\lambda)^{-1}}}{\theta^\t(z^*-z)^2} \label{eq:fix_con_samp_1}
\end{align}
where line \eqref{eq:fc_ub_rho_bound_1} follows since $\Z_k \subset S_k$ for $k =2,3, \ldots$, line \eqref{eq:fc_ub_rho_bound_2} follows since the triangle inequality implies $\max_{z,z^\prime \in S_k} \norm{z-z^\prime}^2_{A(\lambda)^{-1}} \leq  c\max_{z \in S_k \setminus \{z_* \}} \norm{z_*-z}^2_{A(\lambda)^{-1}}$, and line \eqref{eq:fix_con_samp_1} follows since for all $z \in S_k \setminus \{z_*\}$, $\Delta_z \leq 2^{-k} B$ by definition. Next, we bound the first term: 
\begin{align}
  \E_{\eta \sim N(0,I)}[ \max_{z,z^\prime \in \Z_k} &\frac{(z-z^\prime)^\t A(\lambda)^{-1/2} \eta}{2^{-k}B}]^2 \nonumber \\
 & = 4 \E_{\eta \sim N(0,I)}[ \max_{z \in \Z_k} \frac{(z_*-z)^\t A(\lambda)^{-1/2} \eta}{2^{-k}B}]^2 \nonumber  \\
& \leq 4  \E_{\eta \sim N(0,I)}[ \max_{z \in S_k} \frac{(z_*-z)^\t A(\lambda)^{-1/2} \eta}{2^{-k}B}]^2 \label{eq:fc_ub_gamma_bound_1} \\
& \leq 4  \E_{\eta \sim N(0,I)}[ \max(\max_{z \in S_k \setminus \{z_*\}} \frac{(z_*-z)^\t A(\lambda)^{-1/2} \eta}{\theta^\t (z^*-z)},0)]^2 \label{eq:fc_ub_gamma_bound_2}  \\
& \leq 8  \big[ \E_{\eta \sim N(0,I)}[ \max_{z \in S_k \setminus \{z_*\}} \frac{(z_*-z)^\t A(\lambda)^{-1/2} \eta}{\theta^\t (z^*-z)}]^2 \nonumber \\
& +  \max_{z \in S_k \setminus \{z_*\}} \frac{\norm{z^*-z}^2_{A(\lambda)^{-1}}}{\theta^\t(z^*-z)^2} \big] \label{eq:fix_con_samp_2}
\end{align}
where line \eqref{eq:fc_ub_gamma_bound_1} follows by $\Z_k \subset S_k$,  line \eqref{eq:fc_ub_gamma_bound_2} follows by Lemma \ref{lem:width_scale_set},  for all $z \in S_k \setminus \{z_*\}$, $\Delta_z \leq 2^{-k} B$, and $z_* \in S_k$, and line \eqref{eq:fix_con_samp_2} follows by Lemma \ref{lem:width_0}. 
Thus, combining \eqref{eq:fix_con_samp_1} and \eqref{eq:fix_con_samp_2}, and taking the infimum over $\lambda$, we obtain 
\begin{align}
\tau_k(\frac{2^{k}}{B})^2 & \leq c [\inf_\lambda \E_{\eta \sim N(0,I)}[ \max_{z \in \Z \setminus \{z_*\}} \frac{(z^*-z)^\t A(\lambda)^{-1/2} \eta}{\theta^\t(z^*-z)}]^2  \nonumber \\
&  +\max_{z \in \Z \setminus\{z_*\}} \frac{\norm{z^*-z}^2_{A(\lambda)^{-1}}}{\theta^\t(z^*-z)^2} \log(k^2/\delta)] \nonumber \\
 & \leq c^\prime [\gamma^*+  \rho^*  \log(k^2/\delta) ]\label{eq:fixed_confidence_ub}
\end{align}
where line \eqref{eq:fixed_confidence_ub} follows by Lemma \ref{lem:get_rhostar_gammastar}. Thus, combining \eqref{eq:fc_ub_main}, \eqref{eq:fc_ub_k_1}, and \eqref{eq:fixed_confidence_ub}, we obtain
\begin{align}
\sum_{k=1}^{c \log(B/\Delta_{min})} N_k  \leq c  \log(B/\Delta_{min})[d + \gamma^*+  \rho^*  \log( \log(B/\Delta_{min})/\delta)]. \label{eq:fc_ub_main_result_1}
\end{align}

Next, we will prove 
\begin{align}
\sum_{k=1}^{c \log(B/\Delta_{min})} N_k  \leq c  \log(B/\Delta_{min})d +  \log(\frac{B}{\min_{k : |S_k| > 1} F_k})[\gamma^*+  \rho^*   \log( \log(B/\Delta_{min})/\delta)]. \label{eq:fc_ub_main_result_2}
\end{align}
where
\begin{align*}
F_k := \begin{cases} 
       \inf_{\lambda}  \max_{z, z^\prime \in S_k } \norm{z-z^\prime}^2_{A(\lambda)^{-1}} \log(\frac{2k^2}{\delta}) + \E_{\eta }[ \max_{z \in S_k} (z-z^\prime)^\t A(\lambda)^{-1/2} \eta]^2 & k \geq 1 \\
      B& k=0 
   \end{cases}
\end{align*}
\eqref{eq:fc_ub_main_result_1} and \eqref{eq:fc_ub_main_result_2} together would imply the result. By a similar argument used to establish \eqref{eq:fc_ub_main_result_1}, it suffices to prove
\begin{align*}
\sum_{k=2}^{c \log(B/\Delta_{min})} \tau_k(\frac{2^{k}}{B})^2 \leq \log(\frac{B}{\min_{k : |S_k| > 1} F_k})[  \gamma^* + \rho^*  \log( \log(B/\Delta_{min})/\delta)]
\end{align*}
Let $L$ be the largest integer such that $|S_L| >1$. Define 
\begin{align*}
H_i = \{k \in [L] : F_k \in (\frac{F_0}{2^{-(i+1)}}, \frac{F_0}{2^{-i}}] \}.
\end{align*}
and define
\begin{align*}
k_i = \max(k : k \in H_i)
\end{align*}
for $i \in \ceil{\log_2(F_0/F_L)}$. Then, the sample complexity is upper bounded by
\begin{align}
\sum_{k=2}^{c \log(B/\Delta_{min})} \tau_k(\frac{2^{k}}{B})^2 & \leq  \sum_{k=2}^{c\log_2(B/\Delta_{min})} F_k (\frac{2^{k}}{B})^2 \label{eq:fc_ub_relate_to_F_k} \\
& = c \sum_{i=1}^{\ceil{\log_2(F_0/F_L)}} \sum_{k \in H_i} F_k (\frac{2^{k}}{B})^2 \nonumber \\
&  \leq c^\prime \sum_{i=1}^{\ceil{\log_2(F_0/F_L)}} \max_{k \in H_i} F_k \sum_{k \in H_i} (\frac{2^{k}}{B})^2 \nonumber \\
& \leq c^{\prime \prime} \sum_{i=1}^{\ceil{\log_2(F_0/F_L)}} \max_{k \in H_i} F_k  (\frac{2^{k_i}}{B})^2 \label{eq:fc_ub_result_2_summable} \\
& \leq c^{\prime \prime \prime} \sum_{i=1}^{\ceil{\log_2(F_0/F_L)}}  F_{k_i}  (\frac{2^{k_i}}{B})^2 \nonumber \\
& \leq c^{\prime \prime \prime \prime}  \ceil{\log_2(F_0/F_L)} \big[\inf_\lambda \E_{\eta \sim N(0,I)}[ \max_{z \in \Z \setminus} \frac{(z^*-z)^\t A(\lambda)^{-1/2} \eta}{\theta^\t(z^*-z)}]^2  \nonumber \\
& +\max_{z \in \Z \setminus} \frac{\norm{z^*-z}^2_{A(\lambda)^{-1}}}{\theta^\t(z^*-z)^2}] \log(\log(B/\Delta_{min})/\delta) \big] \label{eq:fc_ub_result_2_last_step}
\end{align}
where line \eqref{eq:fc_ub_relate_to_F_k} follows since $\Z_k \subset S_k$, line \eqref{eq:fc_ub_result_2_summable} follows since $\sum_{l=1}^m (2^l)^2 \leq c 2^{2m}$, and line \eqref{eq:fc_ub_result_2_last_step} follows by \eqref{eq:fixed_confidence_ub}.

\end{proof}

\subsection{Computationally Efficient Algorithm for Combinatorial Bandits Proofs}

Before giving the proof of Theorem \ref{thm:sample_complexity_compute}, we restate the algorithm with subroutines for solving the optimization problems approximately. Define $\zero = (0, \ldots, 0)^\t$.

\begin{algorithm}\small
\textbf{Input:} Confidence level $\delta > 0$,  rounding parameter $\epsilon \in (0,1)$ with default value of $\frac{1}{10}$, $\alpha > 0$ ($\alpha = 42941$ suffices though this is wildly pessimistic; we recommend using $\alpha = 4$) \;
$\lambda, \frac{1}{4}\Gamma^\prime \longleftarrow \text{ComputeAlloc}(\zero, \zero, 1,\frac{ \delta}{4})$, which approximately solves
\begin{align*}
\gamma(\Z) := \inf_{\lambda \in \simp} \E_{\eta \sim N(0,I)}[ \max_{z \in \Z}  z^\t A(\lambda)^{-1/2} \eta]^2 
\end{align*}
$\Gamma \longleftarrow \Gamma^\prime \vee 1$, $\widehat{\theta}_0 \longleftarrow \zero \in \R^d$, $\delta_k \longleftarrow \frac{\delta}{2k^3}$ \;
\For{$k=0,1,2,\ldots$}{
$\tilde{z}_{k} \longleftarrow \argmax_{z \in \Z} \widehat{\theta}_{k}^\t z$\;
$\lambda_k, \tau_k \longleftarrow \text{ComputeAlloc}(\tilde{z}_k, \widehat{\theta}_k,2^{-k} \Gamma, \frac{6\delta}{4 \pi^2(k+1)^2})$, which approximately solves
\begin{align}
\inf_{\lambda \in \simp} \E_{\eta \sim N(0,I)}[ \max_{z \in \Z} \frac{(\tilde{z}_k- z)^\t A(\lambda)^{-1/2} \eta}{2^{-k} \Gamma+ \widehat{\theta}^\t_{k}(\tilde{z}_k - z) }]^2 \label{eq:action_comp_2_supp}
\end{align}
\hspace{-.25cm} Set $N_k \longleftarrow \alpha \ceil{ \tau_k \log(1/\delta_k) (1+\epsilon)} \vee q(\epsilon)$ and find $\{x_1, \ldots, x_{N_k} \} \longleftarrow \text{ROUND}(\lambda_k,N_k)$\; 
Pull arms $x_1, \ldots, x_{N_k} $ and receive rewards $y_1, \ldots, y_{N_k}$\; 
Let $\widehat{\theta}_{k+1} \longleftarrow (\sum_{s=1}^{N_k} x_s x_s^\t)^{-1} \sum_{s=1}^{N_k} x_s y_s$ \;
\textbf{if} \textsc{Unique}$(\Z, \widehat{\theta}_k, 2^{-k} \Gamma)$ \textbf{then} \Return $\tilde{z}_k$
}
\caption{Fixed Confidence Peace with a linear maximization oracle.}
\label{alg:action_comp2}
\end{algorithm}

We briefly note that the optimization problem in \eqref{eq:action_comp_2_supp} includes  $\gamma(\Z)$ as a special case by the following identity:
\begin{align*}
\E_{\eta \sim N(0,I)}[ \max_{z, z^\prime \in \Z}  (z-z^\prime)^\t A(\lambda)^{-1/2} \eta]^2 = 4 E_{\eta \sim N(0,I)}[ \max_{z \in \Z}  z^\t A(\lambda)^{-1/2} \eta]^2.
\end{align*}

We also define the \textsc{Unique} subroutine (Algorithm \ref{alg:unique}), originally provided in \citep{chen2017nearly}. It finds the empirical best $\tilde{z}$ and the emprical second best $z^\prime$ and determines whether enough samples have been collected to conclude that $\tilde{z}$ is the best. It uses at most $d$ calls to the linear maximization oracle.
\begin{algorithm}\small
\textbf{Input:} $\Z$, estimate $\tilde{\theta}$, shift $b > 0$ \;
$\tilde{z} \longleftarrow  \argmax_{z \in \Z} \tilde{\theta}^\t z$\;
\For{$i=1,2,\ldots,d$ s.t. $i \in \tilde{z}$}{
\begin{align*}
\tilde{\theta}^{(i)} = \begin{cases} 
      \tilde{\theta}_j  &  j \neq i \\
      -\infty& j = i 
   \end{cases}
\end{align*}
$\tilde{z}^{(i)} \longleftarrow \argmax_{z \in \Z} (\tilde{\theta}^{(i)})^\t z$\;
\textbf{if} $\tilde{\theta}^\t(\tilde{z}-\tilde{z}^{(i)}) - b \leq 0$ \textbf{then} \Return False
}
\Return True
\caption{\textsc{Unique}.}
\label{alg:unique}
\end{algorithm}

\begin{proof}[Proof of Theorem \ref{thm:sample_complexity_compute}]

We will first show that if we can solve the optimization problem 
\begin{align*}
\E \max_z \frac{(z_0- z)^\t A(\lambda)^{-1/2} \eta}{b + \theta_0^\t(z_0 - z) }.
\end{align*}
for arbitrary $\theta_0 \in \R^d$, $z_0 \in \Z $, and $b > 0$, then the sample complexity claim follows. In particular, this implies solving the optimization problems $\gamma(\Z)$ and \eqref{eq:action_comp_2_supp}. Then, we will show that solving it approximately using the subroutine $\text{ComputeAlloc}$ only affects up to a constant factor and bound the number of oracle calls.

\textbf{Step 1: Good event holds with high probability.} Define the sets
\begin{align*}
S_k = \begin{cases}
\{z \in \Z : \Delta_z \leq \Gamma 2^{-k}\} & k \geq 1 \\
\Z & k = 0
\end{cases}
\end{align*}
and define $\delta_k = \frac{\delta}{2 k^3}$. Define the events for all $j \in [k]$
\begin{align*}
\Sigma_{k,j} & = \{ \sup_{z,z^\prime \in S_j} |(z-z^\prime)^\t(\widehat{\theta}_{k}-\theta)|  
 \leq \\
& \hspace{3cm} \sqrt{ 2(1+\epsilon)(1 + \pi \log(1/\delta_k)) \frac{\E[\sup_{z,z^\prime \in S_j} (z-z^\prime)^\t A(\lambda_k)^{-1/2} \eta]^2}{N_k}} \} \\
\Sigma_k & = \cap_{j=0}^k \Sigma_{k,j} \\
\Sigma & = \cap_{k=1}^{\log(\Gamma/\Delta_{min})} \cap_{j=0}^k \Sigma_{k,j}
\end{align*}
Let $x_1, \ldots, x_{N_k}$ denote the measurement vectors selected in round $k$. Theorem 5.8 from \citep{boucheron2013concentration} implies that with probability at least $1-\frac{\delta}{k^3}$ 
\begin{align}
& \sup_{z,z^\prime \in S_j} |(z-z^\prime)^\t(\widehat{\theta}_{k+1}-\theta)| \nonumber \\
& \leq \E\sup_{z,z^\prime \in S_j} (z-z^\prime)^\t(\widehat{\theta}_k-\theta) + \sqrt{ 2\log(1/\delta_k)\max_{z, z^\prime \in S_j} \norm{z-z^\prime}^2_{(\sum_{i=1}^{N_k} x_i x_i^\t)^{-1}}} \nonumber \\
& =\E\sup_{z,z^\prime \in S_j} (z-z^\prime)^\t (\sum_{i=1}^{N_k} x_i x_i^\t)^{-1/2} \eta + \sqrt{ 2\log(1/\delta_k)\max_{z, z^\prime \in S_j} \norm{z-z^\prime}^2_{(\sum_{i=1}^{N_k} x_i x_i^\t)^{-1}}} \nonumber \\
& \leq \E\sup_{z,z^\prime \in S_j} (z-z^\prime)^\t (\sum_{i=1}^{N_k} x_i x_i^\t)^{-1/2} \eta  \nonumber \\
& +  \sqrt{\pi \log(1/\delta_k)\E[\sup_{z,z^\prime \in S_j} (z-z^\prime)^\t (\sum_{i=1}^{N_k} x_i x_i^\t)^{-1/2} \eta]^2} \label{eq:comp_ub_event} \\
& \leq \sqrt{ 2(1 + \pi \log(1/\delta_k))\E[\sup_{z,z^\prime \in S_j} (z-z^\prime)^\t (\sum_{i=1}^{N_k} x_i x_i^\t)^{-1/2} \eta]^2}  \label{eq:comp_ub_sqrt_bd} \\
& \leq \sqrt{ 2(1+\epsilon)(1 + \pi \log(1/\delta_k))\frac{\E[\sup_{z,z^\prime \in S_j} (z-z^\prime)^\t A(\lambda_k)^{-1/2} \eta]^2}{N_k}} \label{eq:comp_ub_rounding}  
\end{align}
where line \eqref{eq:comp_ub_event} follows by Lemma \ref{lem:rho_gamma}, line \eqref{eq:comp_ub_sqrt_bd} follows by $\sqrt{a} +\sqrt{b} \leq \sqrt{2(a+b)}$, and line \eqref{eq:comp_ub_rounding} follows by Lemma \ref{lem:rounding_lemma}. Therefore, $\P(\Sigma_{k,j}^c) \leq \frac{\delta}{k^3}$. By law of total probability, 
\begin{align*}
\P(\Sigma^c) & \leq \sum_{k=1}^\infty \sum_{j=0}^k \P(\Sigma_{k,j}^c | \cap_{l=1}^{k-1} \Sigma_l)  \leq \sum_{k=1}^\infty  (k+1) \frac{\delta}{k^3} \leq 3 \delta.
\end{align*}
We suppose the event $\Sigma$ holds for the rest of the proof.

\textbf{Step 2: gaps are well estimated every round $k$} Now, we show that the following hold: at every round $k \geq 1$,
 \begin{enumerate}
\item if $z \in S_k^c$,
\begin{align*}
|(z_* -z)^\t(\widehat{\theta}_k - \theta)| \leq  \frac{\Delta_z}{8}
\end{align*}
\item if $z \in S_k$,
\begin{align*}
|(z_* -z)^\t(\widehat{\theta}_k - \theta)| \leq  \frac{2^{-k} \Gamma}{8}.
\end{align*}
\end{enumerate}
We proceed inductively. First, we prove the base case $k=1$.  On the event $\Sigma_{1,1}$, we have using the definition of $N_1$, for all $z \in \Z$,
\begin{align}
|(z_* -z)^\t(\widehat{\theta}_1 - \theta)| & \leq \sup_{z,z^\prime \in \Z} |(z-z^\prime)^\t(\widehat{\theta}_1-\theta)| \nonumber \\
& \leq  \sqrt{\frac{2(1+\epsilon)(1 + \pi \log(1/\delta_k))\E[\sup_{z,z^\prime \in \Z} (z-z^\prime)^\t A(\lambda)^{-1/2} \eta]^2}{N_0} } \nonumber \\
& \leq  \sqrt{ \frac{8(1+\epsilon)(1 + \pi \log(1/\delta_k))\E[\sup_{z \in \Z} (\tilde{z}_0-z)^\t A(\lambda)^{-1/2} \eta]^2}{N_0} } \nonumber \\
& \leq \frac{2^{-1} \Gamma}{8} \label{eq:comp_ub_gaps_base_case}
\end{align}
where in the last line we used $N_k = \alpha \ceil{ \tau_k \log(1/\delta_k) (1+\epsilon)} \vee q(\epsilon)$.
Observe that whether $z \in S_1$ or $z \in S_1^c$ the base case follows. Next, we show the inductive step. Suppose that at round $k \geq 1$, if $z \in S_k^c$,
\begin{align*}
|(z_* -z)^\t(\widehat{\theta}_k - \theta)| \leq  \frac{\Delta_z}{8}
\end{align*}
and 
if $z \in S_k$,
\begin{align*}
|(z_* -z)^\t(\widehat{\theta}_k - \theta)| \leq \frac{2^{-k} \Gamma}{8}.
\end{align*}
Now, consider round $k+1$. Fix $z_0 \in S_{k+1}^c$. If $\Delta_z \geq \frac{\Gamma}{2}$, there is nothing to show by \eqref{eq:comp_ub_gaps_base_case}. Thus, suppose $\Delta_z \leq \frac{\Gamma}{2}$. Then, there exists $j \leq k$ such that $\Gamma 2^{-(j+1)} \leq \Delta_{z_0} \leq \Gamma 2^{-j}$. Then, 
\begin{align}
\frac{|(z_*-z_0)^\t(\widehat{\theta}_{k+1}-\theta)|}{\Delta_{z_0}} & \leq \sup_{z,z^\prime \in S_j} |\frac{(z-z^\prime)^\t(\widehat{\theta}_{k+1}-\theta)}{\Delta_{z_0}}| \nonumber \\
& \leq  \sqrt{ 2(1+\epsilon)(1 + \pi \log(1/\delta_k))\frac{\E[\sup_{z, z^\prime \in S_j} \frac{(z-z^\prime)^\t A(\lambda)^{-1/2} \eta}{\Delta_{z_0}}]^2}{N_k}} \label{eq:comp_ub_ind_step_event} \\
& \leq \sqrt{ 8(1+\epsilon)(1 + \pi \log(1/\delta_k))\frac{\E[\sup_{z \in S_j} \frac{(\tilde{z}_k-z)^\t A(\lambda)^{-1/2} \eta}{\Delta_{z_0}}]^2}{N_k}} \nonumber \\
& \leq  \sqrt{36(1+\epsilon)(1 + \pi \log(1/\delta_k))\frac{\E[\sup_{z \in S_j} \frac{(\tilde{z}_k-z)^\t A(\lambda)^{-1/2} \eta}{\Delta_{z} + 2^{-k} \Gamma}]^2}{N_k}} \label{eq:comp_ub_ind_step_1} \\
& \leq \sqrt{ 36(1+\epsilon)(1 + \pi \log(1/\delta_k))\frac{\E[\sup_{z \in \Z} \frac{(\tilde{z}_k-z)^\t A(\lambda)^{-1/2} \eta}{\Delta_{z} + 2^{-k} \Gamma}]^2}{N_k}}  \\
& \leq \sqrt{ 162(1+\epsilon)(1 + \pi \log(1/\delta_k))\frac{\E[\sup_{z \in \Z} \frac{(\tilde{z}_k-z)^\t A(\lambda)^{-1/2} \eta}{(\tilde{z}_k-z)^\t\widehat{\theta}_k + 2^{-k} \Gamma}]^2}{N_k}} \label{eq:comp_ub_ind_step_2} \\
& \leq \frac{1}{8} \label{eq:comp_ub_ind_step_3}
\end{align}
where line \eqref{eq:comp_ub_ind_step_event} follows by the event $\Sigma$, line \eqref{eq:comp_ub_ind_step_1} follows from Lemma \ref{lem:width_scale_set} since $\tilde{z}_k \in S_j$ and for all $z \in S_j$, $3\Delta_{z_0} \geq \Delta_z + 2^{-k} \Gamma$, \eqref{eq:comp_ub_ind_step_2} follows by the inductive hypothesis and Lemma \ref{lem:comp_effic_stat_claim}, and  \eqref{eq:comp_ub_ind_step_3} follows by the definition of $N_k$. Next, fix $z_0 \in S_{k+1}$; a similar series of inequalities shows that
\begin{align*}
|(z_* -z_0)^\t(\widehat{\theta}_{k+1} - \theta)| \leq \frac{2^{-(k+1)} \Gamma}{8},
\end{align*}
yielding the claim. 

\textbf{Step 3: Correctness.} To show correctness, it suffices to show that at round $k$, if $\tilde{z}_k \neq z_*$, then the \textsc{Unique}$(\Z, \widehat{\theta}_k, 2^{-k} \Gamma)$ returns false. Inspection of the subroutine reveals that it suffices to show that  $(\tilde{z}_k -z_*)^\t\widehat{\theta}_k - 2^{-k} \Gamma \leq 0$. By the claim in Step 2, we have that
\begin{align*}
(\tilde{z}_k -z_*)^\t\widehat{\theta}_k - 2^{-k} \Gamma & = (\tilde{z}_k -z_*)^\t(\widehat{\theta}_k-\theta) - \Delta_{\tilde{z}_k} - 2^{-k} \Gamma \\
& \leq \max(\frac{\Delta_{\tilde{z}_k}}{8},\frac{2^{-k} \Gamma}{8})  - \Delta_{\tilde{z}_k} - 2^{-k} \Gamma \\
& \leq 0
\end{align*}
proving correctness.

\textbf{Step 4: Upper bound the sample complexity.} Note that at round $k$, \textsc{Unique}$(\Z, \widehat{\theta}_k, 2^{-k} \Gamma)$ checks whether the gap between $\tilde{z}_k$ and $\argmax_{z \neq \tilde{z}_k} \widehat{\theta}_k^\t z$ is at least $2^{-k} \Gamma$, and terminates if it is. Thus, by the claim in Step 2, the algorithm terminates and outputs $z_*$ once $k \geq c \log(\Gamma /\Delta_{\min})$. Thus, the sample complexity is upper bounded by
\begin{align}
\sum_{k=1}^{c \log(\Gamma/\Delta_{min})} N_k \leq c^\prime [ \log(\Gamma/\Delta_{min})d + \sum_{k=1}^{c \log(\Gamma/\Delta_{min})} \tau_k(\frac{2^{k}}{\Gamma})^2] \label{eq:comp_fc_ub_main}
\end{align}
where we used $q(\epsilon) = O(d)$ by the guarantees on the rounding procedure and $\epsilon =1/10$. Now, we focus on upper bounding the second term in the above expression. For $k=1$, then
\begin{align}
\tau_1(\frac{2^{1}}{\Gamma})^2 & \leq \frac{c}{\Gamma} \leq c^\prime \label{eq:comp_fc_ub_k_1}
\end{align}
where we used the relation $\Gamma = \tau_1 \vee 1 $. Thus, to obtain the upper bound on the sample complexity, it suffices to upper bound 
\begin{align*}
\tau_k = \inf_{\lambda \in \simp} \E_{\eta \sim N(0,I)}[ \max_{z \in \Z} \frac{(\tilde{z}_k-z)^\t A(\lambda)^{-1/2} \eta}{2^{-k} \Gamma + \widehat{\theta}^\t_{k}(\tilde{z}_k - z) }]^2
\end{align*}
for $k > 1$. Fix $\lambda \in \simp$. We have that
\begin{align*}
\E_{\eta \sim N(0,I)}[ \max_{z \in \Z} \frac{(\tilde{z}_k-z)^\t A(\lambda)^{-1/2} \eta}{2^{-k} \Gamma + \widehat{\theta}^\t_{k}(\tilde{z}_k - z) }]^2 & \leq c \E_{\eta \sim N(0,I)}[ \max_{z \in \Z} \frac{(\tilde{z}_k-z)^\t A(\lambda)^{-1/2} \eta}{2^{-k} \Gamma + \Delta_z }]^2 \\
& \leq c^\prime [\E_{\eta \sim N(0,I)}[ \max_{z \in \Z} \frac{(z_*-z)^\t A(\lambda)^{-1/2} \eta}{2^{-k} \Gamma + \Delta_z }]^2 \\
& + \E_{\eta \sim N(0,I)}[ \max_{z \in \Z} \frac{(z_*-\tilde{z}_k)^\t A(\lambda)^{-1/2} \eta}{2^{-k} \Gamma + \Delta_z }]^2] 
\end{align*}
We bound the first term as follows. Fix $z_0 \in \Z \setminus \{z_*\}$. 
\begin{align}
\E_{\eta \sim N(0,I)}[& \max_{z \in \Z} \frac{(z_*-z)^\t A(\lambda)^{-1/2} \eta}{2^{-k} \Gamma + \Delta_z }]^2 \nonumber \\
& = \E_{\eta \sim N(0,I)}[ \max_{z \in \Z \setminus \{z_* \}} \max(\frac{(z_*-z)^\t A(\lambda)^{-1/2} \eta}{2^{-k} \Gamma + \Delta_z }, 0)]^2 \nonumber \\
& \leq \E_{\eta \sim N(0,I)}[ \max_{z \in \Z \setminus \{z_* \}} |\frac{(z_*-z)^\t A(\lambda)^{-1/2} \eta}{2^{-k} \Gamma + \Delta_z }|]^2 \nonumber \\
& \leq 8\E_{\eta \sim N(0,I)}[ \max_{z \in \Z \setminus \{z_* \}} \frac{(z_*-z)^\t A(\lambda)^{-1/2} \eta}{2^{-k} \Gamma + \Delta_z }]^2 +  8\frac{\norm{z_*- z_0}_{A(\lambda)^{-1}}^2}{(2^{-k} \Gamma + \Delta_{z_0})^2 } \label{eq:comp_ub_samp_comp_2} \\
& \leq 8[\E_{\eta \sim N(0,I)}[ \max_{z \in \Z \setminus \{z_* \}} \frac{(z_*-z)^\t A(\lambda)^{-1/2} \eta}{ \Delta_z }]^2 \nonumber \\
& +  \max_{z \neq z_*} \frac{\norm{z_*- z}_{A(\lambda)^{-1}}^2}{ \Delta_z^2 }] \label{eq:comp_ub_final_1}
\end{align}
where line \eqref{eq:comp_ub_samp_comp_2} follows by exercise 7.6.9 in \citep{vershynin2019high}. 

It remains to bound the second term. Note that  
\begin{align}
\E_{\eta \sim N(0,I)}[ \max_{z \in \Z} \frac{(z_*-\tilde{z}_k)^\t A(\lambda)^{-1/2} \eta}{2^{-k} \Gamma + \Delta_z }]^2 & \leq \E_{\eta \sim N(0,I)}[ \max(\frac{(z_*-\tilde{z}_k)^\t A(\lambda)^{-1/2} \eta}{2^{-k} \Gamma  },0)]^2 \nonumber \\
& \leq c \frac{\norm{z_*-\tilde{z}_k}_{A(\lambda)^{-1}}^2}{(2^{-k} \Gamma)^2} \nonumber \\
& \leq c \frac{\norm{z_*-\tilde{z}_k}_{A(\lambda)^{-1}}^2}{\Delta_{\tilde{z}_k}^2} \label{eq:comp_ub_samp_comp_3} \\
& \leq c \max_{z \in \Z \setminus \{z_*\}}  \frac{\norm{z_*-z}_{A(\lambda)^{-1}}^2}{\Delta_{z}^2} \label{eq:comp_ub_final_2}
\end{align}
where line \eqref{eq:comp_ub_samp_comp_3} follows since $\tilde{z}_k \in S_{k+2}$ by Lemma \ref{lem:comp_effic_stat_claim}. 

Thus, combining \eqref{eq:comp_fc_ub_main}, \eqref{eq:comp_fc_ub_k_1}, \eqref{eq:comp_ub_final_1}, and \eqref{eq:comp_ub_final_2} yield the upper bound 
\begin{align*}
\sum_{k=1}^{c \log(\Gamma/\Delta_{min})} N_k \leq c \log(\Gamma/\Delta_{min})[d + \gamma^* + \rho^*].
\end{align*}

\textbf{Step 5: Computation.} Next, we show that we can solve the optimization problems $\gamma(\Z)$ and \eqref{eq:action_comp_2_supp} approximately and bound the number of oracle calls. In the interest of brevity, define
\begin{align*}
g_k(\lambda) := \E_{\eta \sim N(0,I)}[ \max_{z \in \Z} \frac{(\tilde{z}_k- z)^\t A(\lambda)^{-1/2} \eta}{2^{-k} \Gamma+ \widehat{\theta}^\t_{k}(\tilde{z}_k - z) }]^2 
\end{align*}
Let $\mc{D}_{k,1}$ denote the event that $\text{GetAlloc}(\tilde{z}_k, \widehat{\theta}_k,2^{-k} \Gamma, \frac{6\delta}{4 \pi^2(k+1)^2}) \text{ returns } \lambda_k \in \simp \text{ such that }$
\begin{align}
g_k(\lambda_k) \leq c[ \inf_{\lambda \in \simp} g_k(\lambda) +1] \label{eq:comp_effic_approx_factor_1}
\end{align}
Let $\mc{D}_{k,1}$ denote the event that $\text{GetAlloc}(\tilde{z}_k, \widehat{\theta}_k,2^{-k} \Gamma, \frac{6\delta}{4 \pi^2(k+1)^2})$ uses at most the following number of oracle calls
\begin{align}
c[d+\log(\phi \cdot k^2 ) + \log(\log(d)^2 \frac{d^3}{(\Gamma 2^{-k})^2} \frac{1}{\delta^2}) ] \log(d)^2 \frac{d^3}{(\Gamma 2^{-k})^2} \frac{k^4}{\delta^2} \label{eq:comp_cost}
\end{align}
where $\phi \leq \max_z \Delta_z + \Gamma$. Furthermore, define $\mc{D}_1 = \cap_k \mc{D}_{k,1}$ and $\mc{D}_2= \cap_k \mc{D}_{k,1}$. 

GetAlloc is applied with confidence level $\frac{6 \delta}{4\pi^2 k^2}$, and thus by Theorem \ref{thm:comp_thm} and a standard union bound argument, with probability at least $\P(\mc{D}_1) \geq 1-\frac{\delta}{4} $ and $\P(\mc{D}_2) \geq1-\frac{1}{2^d}\cdot \frac{1}{4}$.

Next, let $\mc{C}_k$ denote that event that $\text{EvalAlloc}(\tilde{z}_k, \widehat{\theta}_k,2^{-k} \Gamma, \frac{6\delta}{4 \pi^2(k+1)^2})$ that the algorithm outputs a $\tau_k$ such that  
\begin{align}
g_k(\lambda_k) \leq \tau_k \leq c [g_k(\lambda_k)+1] \label{eq:comp_effic_approx_factor_2}
\end{align}
and the number of oracle calls is upper bounded by 
\begin{align*}
O( \frac{d^2}{(\Gamma 2^{-k})^2} \log(k/\delta) \log(\frac{d k}{\Gamma 2^{-k} \delta})).
\end{align*}
Define $\mc{C} = \cap_k \mc{C}_k$. Since EvalAlloc is applied with confidence level $\delta = \frac{6 \delta}{4\pi^2 k^2}$ and  by Lemma \ref{lem:eval_alloc} and a standard union bound argument, $\P(\mc{C} ) \geq 1-\frac{\delta }{2}$. 

Suppose that $\mc{D}_1 \cap \mc{C} \cap \mc{E}$ occurs. Inspection of the proof reveals that nothing is lost by the approximation in \eqref{eq:comp_effic_approx_factor_1} and \eqref{eq:comp_effic_approx_factor_2}. Thus, by a union bound, it follows that with probability at least $1-4\delta$, the algorithm terminates and returns $z_*$ after the stated number of samples in the theorem. 

Now, suppose  $\mc{D}_1 \cap \mc{D}_2 \cap \mc{C} \cap \mc{E}$ holds. Since there are $c \log(\Gamma /\Delta_{\min})$ rounds, the bound on the number of oracle calls follows by the dominant term appearing in line \eqref{eq:comp_cost}. Thus, by the union bound and assuming $\delta \geq \frac{1}{2^d}$, the event $\mc{D}_1 \cap \mc{D}_2 \cap \mc{C} \cap \mc{E}$ occurs with probability at least $1-4\delta$. This completes the proof.

\end{proof}

The following Lemma is an essential ingredient in the proof of the upper bound for the computationally efficient algorithm for combinatorial bandits.

\begin{lemma}
\label{lem:comp_effic_stat_claim}
Let $k \geq 1$. Consider the $k$th round of Algorithm \ref{alg:action_comp}. Suppose that
 \begin{itemize}
\item if $z \in S_k^c$,
\begin{align}
|(z_* -z)^\t(\widehat{\theta}_k - \theta)| \leq  \frac{\Delta_z}{8} \label{eq:comp_effic_claim_hyp_1}
\end{align}
\item if $z \in S_k$,
\begin{align}
|(z_* -z)^\t(\widehat{\theta}_k - \theta)| \leq  \frac{2^{-k} \Gamma}{8}. \label{eq:comp_effic_claim_hyp_2}
\end{align}
\end{itemize}
Then, the following hold:
\begin{enumerate}
\item 
\begin{align}
\tilde{z}_k  \in S_{k+2}, \label{eq:comp_effic_claim_result_1}
\end{align}
\item if $z \in S_k^c$
\begin{align}
|(\tilde{z}_k-z)^\t\widehat{\theta}_k -(z_*-z)^\t \theta | & \leq  \frac{1}{2} \Delta_z. \label{eq:comp_effic_claim_result_2}
\end{align}
\item if $z \in S_k$,
\begin{align}
|(\tilde{z}_k-z)^\t\widehat{\theta}_k -(z_*-z)^\t \theta | & \leq \frac{1}{2} 2^{-k} \Gamma. \label{eq:comp_effic_claim_result_3}
\end{align}
\item There exist universal constants $c,c^\prime>0$ such that
\begin{align*}
c \E[\sup_{z \in \Z} \frac{(\tilde{z}_k-z)^\t A(\lambda)^{-1/2} \eta}{\Delta_{z} + 2^{-k} \Gamma}]^2 & \leq \E[\sup_{z \in \Z} \frac{(\tilde{z}_k-z)^\t A(\lambda)^{-1/2} \eta}{(\tilde{z}_k-z)^\t\widehat{\theta}_k + 2^{-k} \Gamma}]^2 \\
& \leq c^\prime \E[\sup_{z \in \Z} \frac{(\tilde{z}_k-z)^\t A(\lambda)^{-1/2} \eta}{\Delta_{z} + 2^{-k} \Gamma}]^2
\end{align*}
\end{enumerate}
\end{lemma}

\begin{proof}

\textbf{Step 1: 1 holds at round $k$.} Note that if $z \in S_{k+2}^c \cap S_k$, then
\begin{align*}
\widehat{\theta}_k^\t(z_*-z) \geq \Delta_z - \frac{2^{-k} \Gamma}{8} > 0
\end{align*}
by \eqref{eq:comp_effic_claim_hyp_2} and since $z \in S_{k+2}^c \cap S_k$ implies that $\Delta_z \geq \frac{2^{-k} \Gamma}{4} $. Thus, $z \neq \tilde{z}_k$. On the other hand, if $z \in  S_k^c$,
\begin{align*}
\widehat{\theta}_k^\t(z_*-z) \geq \Delta_z - \frac{\Delta_z}{8} > 0
\end{align*}
by \eqref{eq:comp_effic_claim_hyp_1}, so that $z \neq \tilde{z}_k$. Together, these cases together imply that $\tilde{z}_k  \in S_{k+2} $.

\textbf{Step 2: 2 and 3 hold at round $k$.} First, suppose $z \in S_k^c$. We have that
\begin{align}
|(\tilde{z}_k-z)^\t\widehat{\theta}_k -(z_*-z)^\t \theta | & \leq |(\tilde{z}_k-z)^\t(\widehat{\theta}-\theta)| + | \theta^\t(\tilde{z}_k -z)-\theta^\t(z_* -z) | \nonumber \\
& \leq  |(\tilde{z}_k-z_*)^\t(\widehat{\theta}-\theta)| + |(z_*-z)^\t(\widehat{\theta}-\theta)| + |\theta^\t(\tilde{z}_k -z_*)| \nonumber \\
& \leq \frac{1}{8}(2^{-k} \Gamma + \Delta_z) + \frac{1}{4} 2^{-k} \Gamma \label{eq:comp_claim_step_2} \\
& \leq \frac{1}{2} \Delta_z \nonumber
\end{align}
where line \eqref{eq:comp_claim_step_2} follows by \eqref{eq:comp_effic_claim_hyp_1} and by \eqref{eq:comp_effic_claim_result_1} which we have shown holds at round $k$. By a similar argument, if $z \in S_k$,
\begin{align*}
|(\tilde{z}_k-z)^\t\widehat{\theta}_k -(z_*-z)^\t \theta | & \leq \frac{1}{2} 2^{-k} \Gamma.
\end{align*}

\textbf{Step 3: 4 holds at round $k$.} We have shown that \eqref{eq:comp_effic_claim_result_1} and \eqref{eq:comp_effic_claim_result_2} hold at round $k$. Fix $z \in \Z$. If $z \in S_k^c$, by \eqref{eq:comp_effic_claim_result_2} we have that $\Delta_z \geq \frac{2}{3} \widehat{\theta}^\t(\tilde{z}_k-z)$ and thus
\begin{align*}
 \frac{1}{\Delta_{z} + 2^{-k} \Gamma} \leq \frac{3}{2} \frac{1}{(\tilde{z}_k-z)^\t\widehat{\theta}_k + 2^{-k} \Gamma}.
\end{align*}
On the other hand, if $z \in S_k$, by \eqref{eq:comp_effic_claim_result_3}, we have that $\Delta_z \geq \widehat{\theta}_k(\tilde{z}_k-z) - \frac{2^{-k} \Gamma }{2}$. Thus, 
\begin{align*}
 \frac{1}{\Delta_{z} + 2^{-k} \Gamma} \leq 2\frac{1}{(\tilde{z}_k-z)^\t\widehat{\theta}_k + 2^{-k} \Gamma}.
\end{align*}
Therefore, since in addition $\tilde{z}_k \in \Z$, we may apply Lemma \ref{lem:width_scale_set} to obtain
\begin{align*}
\E[\sup_{z \in \Z} \frac{(\tilde{z}_k-z)^\t A(\lambda)^{-1/2} \eta}{\Delta_{z} + 2^{-k} \Gamma}] \leq 2\E[\sup_{z \in \Z} \frac{(\tilde{z}_k-z)^\t A(\lambda)^{-1/2} \eta}{(\tilde{z}_k-z)^\t\widehat{\theta}_k + 2^{-k} \Gamma}]
\end{align*}
yielding one of the inequalities. By a similar argument, we obtain the other inequality, proving the claim. 
\end{proof}

\section{Computational Results for Computationally Efficient Algorithm for Combinatorial Bandits}
\label{sec:comp_results_alg}

In this section, we present the computational subroutines for the computationally efficient algorithm for combinatorial bandits. The main optimization problem in Algorithm \ref{alg:action_comp} is given in line \eqref{eq:action_comp_2_supp}. Fix $z_0 \in \Z$, $b > 0$, and $\theta_0 \in \R^d$ for the remainder of the section; we will omit dependence on these quantities because they are fixed. Since the Gaussian width is nonnegative, it suffices to solve:
\begin{align*}
\inf_{\lambda \in \simp} g(\lambda) := \E \max_{z \in \Z} \frac{(z_0- z)^\t A(\lambda)^{-1/2} \eta}{b + \theta_0^\t(z_0 - z) }.
\end{align*}
Define the following functions
\begin{align*}
g(\lambda;\eta) & :=    \max_{z \in \Z} \frac{(z_0- z)^\t A(\lambda)^{-1/2} \eta}{b + \theta_0^\t(z_0 - z) } \\
g(\lambda;\eta;z) & :=   \frac{(z_0- z)^\t A(\lambda)^{-1/2} \eta}{b + \theta_0^\t(z_0 - z) }\\
g( \lambda; \eta;r) & := \max_{z \in \Z} z^\t(A(\lambda)^{-1/2} \eta + r \theta_0) - r(b + \theta^\t_0 z_{0} ) -z_0^\t A(\lambda)^{-1/2} \eta \\
g(\lambda; \eta;r;z) & := z^\t(A(\lambda)^{-1/2} \eta + r \theta_0) - r(b + \theta^\t_0 z_{0} ) -z_0^\t A(\lambda)^{-1/2} \eta 
\end{align*}

\subsection{Main Subroutine}

\begin{algorithm}\small
\textbf{Input: }  $z_0 \in \Z$,  $\theta_0 \in \R^d$, Offset $b > 0$, $\delta > 0$ \;
$\lambda \longleftarrow \text{GetAlloc}(z_0, \theta_0, b, \delta)$\;
$\tau \longleftarrow \text{EvalAlloc}(z_0, \theta_0, b, \lambda, \delta)$\;
\textbf{Return } $(\lambda, \tau)$
 \caption{ComputeAlloc($z_0, \theta_0, b,\delta)$}
 \label{alg:comp_alloc}
\end{algorithm}

ComputeAlloc($z_0, \theta_0, b, \delta)$ is the main subroutine; it solves and evaluates $\inf_{\lambda} g(\lambda)$. $\text{GetAlloc}(z_0, \theta_0, b, \delta)$ and $\text{EvalAlloc}(z_0, \theta_0, b, \lambda, \delta)$ only use calls to the linear maximization oracle. $\text{GetAlloc}(z_0, \theta_0, b, \delta)$ finds a solution within a constant additive factor of the optimal solution to the optimization problem $\inf_{\lambda \in \simp} g(\lambda) $ with probability at least $1-\delta$. $\text{EvalAlloc}(z_0, \theta_0, b, \lambda, \delta)$ determines the value of $g(\lambda)$ within a constant additive factor with probability at least $1-\delta$.

GetAlloc (Algorithm \ref{alg:mirror_descent}) performs stochastic mirror descent over the subset of the simplex that is a mixture with the uniform distribution
\begin{align*}
\simpm := \{ \lambda \in \R^d : \lambda = \frac{1}{2} (\kappa + \kappa^\prime) \text{ where } \kappa \in \simp \text{ and } \kappa^\prime = (1/d, \ldots,1/d)^\t \}.
\end{align*}
Define the Bregman divergence associated with a function $f$:
\begin{align*}
D_f(x,y) = f(x) - f(y) - \nabla f(y)^\t(x-y).
\end{align*}
GetAlloc calls estimateGradient (Algorithm \ref{alg:est_grad}) to obtain an unbiased estimate of the gradient. estimateGradient needs to solve a maximization problem, for which it calls computeMax (Algorithm \ref{alg:compute_max}), a subroutine that essentially performs binary search.

\begin{algorithm}\small
\textbf{Input: } $z_0 \in \Z$, Offset $b \in \R$, $\theta_0 \in \R^d$, confidence level $\delta > 0$ \;
Define $\Phi(\lambda) = \sum_{i=1}^d \lambda_i \log(\lambda_i)$\;
$T \longleftarrow c \log(d)^2 \frac{d^3}{b^2} \frac{1}{\delta^2}$ where $c >0$ is a universal constant obtained in the proof of Theorem \ref{thm:comp_thm}\;
$\kappa = \frac{c^\prime}{\frac{d^{3}}{b^2}} \sqrt{\frac{2}{T}}$ where $c^\prime >0$ is a universal constant obtained in the proof in of Theorem \ref{thm:comp_thm}\;
$\lambda^{(1)} \longleftarrow \argmin_{\lambda \in \simpm} \Phi(\lambda)$\;
\For{$s=1, 2, \ldots, T$}{
Let $r_s \longleftarrow \text{estimateGradient}(z_0, \theta_0, b, \lambda)$\;
$\lambda_{s+1} = \argmin_{\lambda \in \simpm} \kappa r_s^\t \lambda + D_{\Phi}(\lambda, \lambda_s)$
}
\textbf{Return } $\frac{1}{T} \sum_{s=1}^T \lambda^{(s)}$
 \caption{GetAlloc($z_0, \theta_0, b, \delta)$: Stochastic Mirror Descent for Transductive Bandits with linear maximization oracle}
 \label{alg:mirror_descent}
\end{algorithm}

\begin{algorithm}\small
\textbf{Input: } $\lambda \in \simp$, $z_0 \in \Z$, Offset $b \in \R$, $\theta_0 \in \R^d$\;
Draw $\eta \sim N(0,I)$\;
$\textsc{max-val}  \longleftarrow \text{computeMax}(z_0, \theta_0, b, \lambda, \eta, 0) $\;
Choose
\begin{align*}
\bar{z} & \in \argmax_{z \in Z} g( \lambda; \eta;\textsc{max-val}; z)
\end{align*}
Return $\nabla_{\lambda} g(\lambda;\eta;\bar{z})  $
 \caption{$\text{estimateGradient}(z_0, \theta_0, b, \lambda)$: Compute unbiased stochastic subgradient}
 \label{alg:est_grad}
\end{algorithm}

\begin{algorithm}\small 
\textbf{Input: } $\lambda \in \simp$, $z_0 \in \Z$, Offset $b \in \R$, $\theta_0 \in \R^d, \eta \in \R^d$, $\textsc{tol} \geq 0$\;
Define
\begin{align*}
\textsc{low} = 0, \qquad \textsc{high} = 2 
\end{align*}
\While{$g(\lambda; \eta: \textsc{high}) \geq 0$}{
$\textsc{high}  \longleftarrow 2\cdot\textsc{high} $\;}
\While{$g(\lambda; \eta; \textsc{low}) \neq 0 $ or $\frac{1}{2}(\textsc{high} + \textsc{low}) > \textsc{tol}$}{
\eIf{$g(\lambda; \eta;\frac{1}{2}(\textsc{high} + \textsc{low}) ) < 0$}{
$\textsc{low}  \longleftarrow \frac{1}{2}(\textsc{high} + \textsc{low}) $}{
$\textsc{high}  \longleftarrow \frac{1}{2}(\textsc{high} + \textsc{low}) $}
$\textsc{low} \longleftarrow g(\lambda;\eta;z^\prime)$ for some $z^\prime \in \argmax g(\lambda;\eta;\textsc{low};z)$}
Return $\textsc{low}$
 \caption{$\text{computeMax}(z_0, \theta_0, b, \lambda, \eta, \textsc{tol})$: Compute $g(\lambda;\eta)$}
 \label{alg:compute_max}
\end{algorithm}

EvalAlloc (Algorithm \ref{alg:eval_alloc}) estimates the number of samples to take in a round, only using calls to the linear maximization oracle. Because it estimates the mean of estimator that is not necessarily sub-Gaussian, but has controlled variance, this subroutine uses the median-of-means estimator.

\begin{algorithm}\small
\textbf{Input: } $\lambda \in \simp$, $z_0 \in \Z$,  $\theta_0 \in \R^d$, Offset $b \in \R$ \;
$T\longleftarrow 864 \frac{d^2}{b^2} \log(1/\delta) $\;
Draw $\eta_1, \ldots, \eta_T \sim N(0,I)$ \;
$y_s \longleftarrow \text{computeMax}(z_0, \theta_0, b, \lambda, \eta_s, \textsc{tol}=1/2)$  for $s=1, \ldots, T$\;
Let $\tau$ be the output of the median of means estimator applied to $y_1, \ldots, y_T$\;
\textbf{Return } $[\tau +1]^2 $
 \caption{EvalAlloc($z_0, \theta_0, b, \lambda)$: Estimate $g(\lambda)$}
 \label{alg:eval_alloc}
\end{algorithm}

\subsection{Proofs}

Recall the definitions:
\begin{align*}
g(\lambda;\eta) & :=    \max_{z \in \Z} \frac{(z_0- z)^\t A(\lambda)^{-1/2} \eta}{b + \theta_0^\t(z_0 - z) } \\
g(\lambda;\eta;z) & :=   \frac{(z_0- z)^\t A(\lambda)^{-1/2} \eta}{b + \theta_0^\t(z_0 - z) }\\
g( \lambda; \eta;r) & := \max_{z \in \Z} z^\t(A(\lambda)^{-1/2} \eta + r \theta_0) - r(b + \theta^\t_0 z_{0} ) -z_0^\t A(\lambda)^{-1/2} \eta \\
g(\lambda; \eta;r;z) & := z^\t(A(\lambda)^{-1/2} \eta + r \theta_0) - r(b + \theta^\t_0 z_{0} ) -z_0^\t A(\lambda)^{-1/2} \eta 
\end{align*}
The following Lemma provides the guarantee for estimateGradient. Define $\phi =  \sqrt{\max_{z \in \Z}\theta_0^\t (z_0 -z)  + b}$.

\begin{lemma}
\label{lem:grad_oracle_calls}
Consider the combinatorial bandit setting. Fix $z_0 \in \Z$, $b > 0$, $\theta_0 \in \R^d$, and $\lambda \in \simpm$. 
$\text{estimateGradient}(z_0, \theta_0, b, \lambda)$ returns an unbiased stochastic gradient of the function $g(\lambda)$ with probability $1$. Let $\xi > 0$. With probability at least $1-\frac{2\xi}{2^d} $, it terminates after $O(d  + \log(\frac{d}{b }) + \log(\frac{\phi}{\xi}))$ oracle calls.
\end{lemma}

\begin{proof}


\textbf{Step 1: Correctness.} Let $\eta \sim N(0,I)$. Note that $\E g(\lambda;\eta) = g(\lambda)$. Since $\eta \sim N(0,I)$, with probability $1$ $\argmax_{z \in \Z} g(\lambda;\eta;z)$ is unique and, therefore,
\begin{align*}
\nabla_{\lambda} \max_{z} g(\lambda;\eta;z) = \nabla_{\lambda} g(\lambda;\eta;\argmax_{z \in \Z} g(\lambda;\eta;z)).
\end{align*}
We claim that we can interchange the expectation and differentation. Note that
\begin{align*}
\E \nabla_{\lambda} g(\lambda;\eta) = \nabla_{\lambda} \E g(\lambda;\eta) \Longleftrightarrow (\E \nabla_{\lambda} g(\lambda;\eta))_i = (\nabla_{\lambda} \E g(\lambda;\eta))_i \quad	\forall i.
\end{align*}
Since
\begin{align*}
|\nabla_{\lambda} g(\lambda;\eta)_i| \leq  |\frac{\lambda_i^{-3/2} \eta_i }{b + \theta_0^\t(z_0 - z)}|
\end{align*}
and $\E_{\eta_i} | \frac{\lambda_i^{-3/2} \eta_i }{b + \theta_0^\t(z_0 - z)} |< \infty$, we have by standard results on exchanging differentiation and expectation that the claim follows. Thus, we have 
\begin{align*}
\E \nabla_{\lambda} g(\lambda;\eta;\argmax_{z \in \Z} g(\lambda;\eta;z)) = \E \nabla_{\lambda} \max_{z} g(\lambda;\eta;z) = \nabla_{\lambda} \E \max_{z} g(\lambda;\eta;z).
\end{align*}
As a consequence, to show that estimateGradient returns an unbiased gradient, it suffices to show that Algorithm \ref{alg:est_grad} identifies $\argmax_{z \in \Z} g(\lambda;\eta;z)$. Note that $g(\lambda;\eta) $ is equivalent to the following linear program problem 
\begin{align*}
r_* =& \min_{r} r \\
& \text{s.t. } g(\lambda ; \eta;r) =  \max_{z \in \Z} z^\t(A^{-1/2}(\lambda) \eta + \bar{r} \theta_0) - r(b + \theta_0^\t z_{0} ) -z_0^\t A(\lambda)^{-1/2} \eta \leq 0 .
\end{align*} 
The estimageGradient algorithm terminates once it finds $\bar{r} > 0$ such that $\max_{z \in \Z} g(\lambda;\eta ;\bar{r};z) = 0$. Let $\bar{z} \in \argmax_{z \in \Z} g(\lambda;\eta ;\bar{r};z)$. Then,
\begin{align*}
0 & = \max_z g(\lambda;\eta ;\bar{r};z)\\
& = g(\lambda;\eta; \bar{r}; \bar{z}) \\
& = \bar{z}^\t(A^{-1/2}(\lambda) \eta + \bar{r} \theta_0) - \bar{r}(b + \theta_0^\t z_{0} ) -z_0^\t A(\lambda)^{-1/2} \eta \\
& > z^\t(A^{-1/2}(\lambda) \eta + \bar{r} \theta_0) - \bar{r}(b + \theta_0^\t z_{0} ) -z_0^\t A(\lambda)^{-1/2} \eta.
\end{align*}
where the strict inequality holds with probability $1$ since $\eta \sim N(0,I)$. Rearranging the above inequality, this implies that for every for all $z \in \Z \setminus \{\bar{z}\}$
\begin{align*}
\frac{(z_0- z)^\t A(\lambda)^{-1/2} \eta}{b + \theta_0^\t(z_0 - z)} <   \bar{r} =\frac{(z_0- \bar{z})^\t A(\lambda)^{-1/2} \eta}{b + \theta_0^\t(z_0 - \bar{z})}
\end{align*}
implying that $\bar{z} = \argmax_z g(\lambda;\bar{r};z)$, showing estimateGradient returns an unbiased gradient.

\textbf{Step 2: Running time.} Next, we bound the number of oracle calls. Define $\tilde{y} = \sup_{z \in \Z} \frac{(z_0- z)^\t A(\lambda)^{-1/2} \eta}{b + \theta_0^\t(z_0 - z) }$. 
By Theorem 5.8 of \citep{boucheron2013concentration}, we have that
\begin{align*}
\V(\tilde{y}) \leq 4 \sup_{z \in \Z} \V(\frac{(z_0- z)^\t A(\lambda)^{-1/2} \eta}{b + \theta_0^\t(z_0 - z) }) \leq 8 \frac{d^2}{b^2}.
\end{align*} 
where we used $\lambda \in \simpm$. Define the event 
\begin{align*}
\mc{E} = \{\sup_{z \in \Z} \frac{(z_0- z)^\t A(\lambda)^{-1/2} \eta}{b + \theta_0^\t(z_0 - z) } \leq 8\frac{d^2}{b^2} \frac{1}{\delta}\}
\end{align*}
Thus, by Chebyshev's inequality, we have that 
\begin{align}
\P(\mc{E}^c) \leq \delta. \label{eq:comp_effic_gauss_bound_1}
\end{align}

Thus, choosing $\delta = \frac{\xi}{2^d}$, we have with probability at least $\P(\mc{E}^c) \leq \frac{\xi}{2^d}$. Then, by Lemma \ref{lem:bin_search_main_lemma} the first while loop requires
\begin{align*}
O(\log(\frac{d}{b \delta }))= O(\log(\frac{d}{b}) + d + \log(\frac{1}{\xi}))
\end{align*}
oracle calls.

Next, we consider the second while loop. Define the event
\begin{align*}
\mc{D} = \{|g(\lambda;\eta;z) - g(\lambda;\eta;z)| > \frac{\xi}{\phi 2^{2d}}, \forall z \neq z^\prime \in \Z \}.
\end{align*}
By Lemma \ref{lem:normals_well_sep}, we have that with probability at least $\mc{D} \geq 1-\frac{\xi}{2^d}$.
Then, by Lemma \ref{lem:bin_search_main_lemma}, the second while loop requires at most $O(d  + \log(\frac{d}{b }) + \log(\frac{\phi}{\xi}))$ oracle calls. A standard union bound argument for event $\mc{E} \cap \mc{D}$ yields the result.


\end{proof}

The following Theorem provides the guarantee for GetAlloc.

\begin{theorem}
\label{thm:comp_thm}
Consider the combinatorial bandit setting. Fix $z_0 \in \Z$, $b > 0$, and $\theta_0 \in \R^d$. With probability at least $1-\delta$ GetAlloc($z_0, \theta_0, b, \delta)$ returns $\bar{\lambda} \in \simp$ such that
\begin{align*}
g(\bar{\lambda})^2 \leq c[ \min_{\lambda \in \simp} g(\lambda)^2 + 1]. 
\end{align*}
Let $\xi > 0$. Furthermore, with probability at least $1-\frac{2\xi}{2^d}$, the number of oracle calls is bounded above by 
\begin{align*}
c[d+\log(\phi/\xi ) + \log(\log(d)^2 \frac{d^3}{b^2} \frac{1}{\delta^2}) ] \log(d)^2 \frac{d^3}{b^2} \frac{1}{\delta^2}.
\end{align*}
\end{theorem}

\begin{proof}

\textbf{Step 1: Guarantee on final allocation.} Note that for any $z \in \Z$, 
\begin{align*}
|\nabla_{\lambda} g(\lambda;\eta;z)_i| = \one\{i \in z_0 \Delta z\} |\frac{\lambda_i^{-3/2} \eta_i }{b + \theta_0^\t(z_0 - z)}|
\end{align*}
and thus
\begin{align*}
\E \max_{z \in \Z} \norm{\nabla_{\lambda} g(\lambda;\eta;z)}_{\infty}^2 & \leq c\frac{d^{3}}{b^2} \E \max_i \eta_i^2 \\ 
& \leq \log(d)c\frac{d^{3}}{b^2} 
\end{align*}
where we used the fact that $\lambda \in \simpm$. 

Note that the mirror map used is 
\begin{align*}
\Phi(\lambda) = \sum_{i=1}^d \lambda_i \log(\lambda_i).
\end{align*}
It is not hard to see that
\begin{align*}
\sup_{ \lambda \in \simpm} \Phi(\lambda) - \min_{ \lambda^\prime \in \simpm} \Phi(\lambda^\prime) \leq \log(d)
\end{align*}
By Theorem 6.1 in \citep{bubeck2015convex}, 
\begin{align*}
\E g(\bar{\lambda}) - \min_{\lambda \in \simpm} g(\lambda) \leq c\log(d) \frac{d^{3/2}}{b} \sqrt{\frac{1}{T}}. 
\end{align*}

Then, by Markov's inequality,
\begin{align*}
\P(g(\bar{\lambda}) - \min_{\lambda \in \simpm} g(\lambda) \geq 1) & \leq \E g(\bar{\lambda}) - \min_{\lambda \in \simp} g(\lambda) \\
& \leq c\log(d) \frac{d^{3/2}}{b} \sqrt{\frac{1}{T}} \\
& \leq c\log(d) \frac{d^{3/2}}{b} \sqrt{\frac{1}{T}} \\
& = \delta
\end{align*}
by our choice of $T$. Noting that $ \min_{\lambda \in \simpm} g(\lambda) \leq \sqrt{2}  \min_{\lambda \in \simp} g(\lambda)$ yields the result.

\textbf{Step 2: Bound the number of oracle calls.} Using Lemma \ref{lem:grad_oracle_calls} with $\xi^\prime = \frac{\xi}{T}$ and union bounding over each of the $T$ iterations, with probability at least $1-\frac{2\xi}{2^d}$ the number of oracle calls is at most
\begin{align*}
c[d+\log(\frac{d}{b \delta}) + \log(\phi/\xi \cdot T)]T = c[d+\log(\phi/\xi ) + \log(\log(d)^2 \frac{d^3}{b^2} \frac{1}{\delta^2}) ] \log(d)^2 \frac{d^3}{b^2} \frac{1}{\delta^2}.
\end{align*}
\end{proof}

The following Lemma provides the guarantee for Algorithm \ref{alg:eval_alloc}.

\begin{lemma}
\label{lem:eval_alloc}
When with probability at least $1-\delta$, Algorithm \ref{alg:eval_alloc} returns $\tau$ such that $g(\lambda)^2 \leq (\tau+1)^2 \leq g(\lambda)^2 + 4$. Furthermore, with probability at least, $1-\delta$, it uses $O( \frac{d^2}{b^2} \log(1/\delta) \log(\frac{d}{b \delta}))$ oracle calls.
\end{lemma}

\begin{proof}
Let $\tilde{y}_s = \sup_{z \in \Z} \frac{(z_0- z)^\t A(\lambda)^{-1/2} \eta_s}{b + \theta_0^\t(z_0 - z) }$. 
By Theorem 5.8 of \citep{boucheron2013concentration}, we have that 
\begin{align*}
\V(\tilde{y}_s) \leq 4 \sup_{z \in \Z} \V(\frac{(z_0- z)^\t A(\lambda)^{-1/2} \eta}{b + \theta_0^\t(z_0 - z) }) \leq 8 \frac{d^2}{b^2}
\end{align*}
Applying the median of means estimator (see \citep{hsu2014heavy}) to $\tilde{y}_1, \ldots, \tilde{y}_T$ yields that with probability at least $1-\delta$ $\tilde{\tau}$ satisfies
\begin{align*}
|\tilde{\tau}-\E\sup_{z \in \Z} \frac{(z_0- z)^\t A(\lambda)^{-1/2} \eta_s}{b + \theta_0^\t(z_0 - z) }| \leq 1/2
\end{align*}
by our choice of $T$ and standard results for median of means estimation. Since the procedure computeMax a tolerance of $1/2$, by Lemma \ref{lem:bin_search_main_lemma}, we have that $|y_s -\tilde{y}_s| \leq 1/2$ for all $s =1,\ldots, T$. Thus, it follows that $|\tilde{\tau}-\tau| \leq 1/2$. Thus,
\begin{align*}
|\tau-\E\sup_{z \in \Z} \frac{(z_0- z)^\t A(\lambda)^{-1/2} \eta_s}{b + \theta_0^\t(z_0 - z) }| \leq 1.
\end{align*}
Manipulating the above inequality yields the result.

It remains to bound the number of oracle calls. Consider $\tilde{y}_s = \sup_{z \in \Z} \frac{(z_0- z)^\t A(\lambda)^{-1/2} \eta_s}{b + \theta_0^\t(z_0 - z) }$. By the same argument made in inequality \eqref{eq:comp_effic_gauss_bound_1}, we have that with probability at least $1-\frac{\delta}{T}$, $\tilde{y}_s \leq O(\frac{d^{2}T}{b^2 \delta})$. Union bounding over all $\tilde{y}_s$ $s\in [T]$, we have that with probability at least $1- \delta$, $\sup_{s} \tilde{y}_s \leq O(\frac{d^4}{b^4 \delta} \log(1/\delta) )$. Since the procedure computeMax uses a tolerance of $1/2$, by Lemma \ref{lem:bin_search_main_lemma} we have that each call of computeMax uses at most $O(\log(\frac{d}{b \delta}))$ calls to the linear maximization oracle, yielding the result.
\end{proof}

\subsection{Technical Lemmas}

\begin{lemma}
\label{lem:normals_well_sep}
Consider the combinatorial bandit setting. Fix $\theta_0 \in \R^d$ and $b \geq 0$. Let $\xi > 0$. Then, 
\begin{align*}
\P(\exists z \neq z^\prime \in \Z: |g(\lambda; \eta; z) - g(\lambda; \eta; z^\prime)| & \leq  \frac{\xi}{\phi 2^{2d}}  ) \leq \frac{\xi}{2^d}.
\end{align*}
\end{lemma}

\begin{proof}
Let $m = |\Z|$. Fix $z \neq z^\prime \in \Z$. Fix $z_0 \in \Z$, $k \in \N$, and $\theta_0 \in \R^d$. For the sake of brevity, define $h(\tilde{z}) := g(\lambda; \eta; \tilde{z})$.
Note that $|h(z) -h(z^\prime)|$ is a truncated normal distribution. Now, we lower bound its variance.
\begin{align*}
\V(h(z) -h(z^\prime)) & = \norm{\frac{(z_0- z)^\t A(\lambda)^{-1/2} }{2^{-k} B + \theta_0^\t(z_0 - z)} - \frac{(z_0- z^\prime)^\t A(\lambda)^{-1/2} }{b + \theta_0^\t(z_0 - z^\prime)}}_2 \\
& \geq \norm{\frac{(z_0- z)^\t A(\lambda)^{-1/2} }{b + \theta_0^\t(z_0 - z)} - \frac{(z_0- z^\prime)^\t A(\lambda)^{-1/2} }{b+ \theta_0^\t(z_0 - z^\prime)}}_\infty \\
& \geq \min_{v \in \Z} \norm{\frac{1}{b + \theta_0^\t(z_0 - v)} }_\infty \\
& \geq \frac{1}{\phi^2}
\end{align*}
where we used the fact that for every $\norm{z-z^\prime}_\infty \geq 1$ for combinatorial bandits and and the definition of $\phi$. 

Then, using the cdf of the half normal, we have that
\begin{align*}
\P(|h(z) - h(z^\prime)| \leq \frac{\xi}{\phi 2^{2d}}) & =  \int_0^{\frac{\xi}{\phi 2^{2d}}} \frac{1}{\sqrt{\V(h(z) -h(z^\prime))}} \sqrt{2/\pi} \exp(-\frac{y^2 }{2\V(h(z) -h(z^\prime))}) dy \\
& \leq \int_0^{\frac{\xi}{\phi 2^{2d}}} \phi \sqrt{2/\pi} \exp(-\frac{y^2 }{2\V(h(z) -h(z^\prime))}) dy \\
& \leq \sqrt{2/\pi}  \frac{\xi}{2^{2d}}.
\end{align*}
Thus, using a union bound, we have that
\begin{align*}
\P(\exists z \neq z^\prime \in \Z: |h(z) - h(z^\prime)| \leq \frac{\xi}{\phi 2^{2d}} & \leq \frac{|\Z| }{2^{2d}})  \\
& \leq  \frac{\xi}{2^{d}} .
\end{align*}

\end{proof}

Lemmas \ref{lem:bin_search_lemma} and \ref{lem:bin_search_update_rule} show that Algorithm \ref{alg:compute_max} essentially performs binary search.

\begin{lemma}
\label{lem:bin_search_main_lemma}
The following two claims holds regarding Algorithm \ref{alg:compute_max}. 
\begin{enumerate}
\item At the end of the first while loop of Algorithm \ref{alg:compute_max}, $g(\lambda;\eta) \in [\textsc{low} ,  \textsc{high} ]$ and it takes at most $O(\log(g(\lambda;\eta)))$ oracle calls.
\item In the second while loop of Algorithm \ref{alg:compute_max}, it always holds that $g(\lambda;\eta) \in [\textsc{low} ,  \textsc{high} ]$. Furthermore, define $\bar{z} = \argmax_{z \in \Z} g(\lambda; \eta; z)$. Then, if $g(\lambda;\eta) - \max_{z \neq \bar{z}} g(\lambda;\eta;z) > \varepsilon$, then it terminates after $O(\log(\frac{g(\lambda;\eta)}{\varepsilon}))$ oracle calls.
\end{enumerate}
 
\end{lemma}

\begin{proof}
We begin by proving the first claim. By Lemma \ref{lem:bin_search_lemma}, if $\textsc{high} < g(\lambda;\eta)$, then $g(\lambda; \eta ; \textsc{high})> 0$ and \textsc{high} keeps increasing. At some point, we have $\textsc{high} > g(\lambda;\eta)$, which by Lemma \ref{lem:bin_search_lemma} implies that $g(\lambda; \eta ; \textsc{high})< 0$ and the while loop terminates. Notice that since $z_0 \in \Z$, $g(\lambda ; \eta) \geq 0=\textsc{low}$. Furthermore, since \textsc{high} doubles at each  round the first while loop takes at most $O(\log(g(\lambda;\eta))))$ oracle calls. This completes the proof of the first claim.

Next, we prove the second claim regarding the second while loop. At the beginning of the second while loop, $g(\lambda;\eta) \in [\textsc{low} ,  \textsc{high} ]$. It is a straightforward consequence of Lemma \ref{lem:bin_search_lemma} that at the end of the if else statement in the second while loop it holds that $g(\lambda;\eta) \in [\textsc{low} ,  \textsc{high} ]$. In the last line of the while loop where 
\begin{equation*}
\textsc{low} \longleftarrow g(\lambda;\eta;z^\prime) \text{ for some } z^\prime \in \argmax g(\lambda;\eta;\textsc{low};z)
\end{equation*}
it follows from Lemma \ref{lem:bin_search_update_rule} that $g(\lambda; \eta; \textsc{low}) \geq 0$. Then, by Lemma \ref{lem:bin_search_lemma}, it follows that $g(\lambda;\eta) \geq \textsc{low}$. Thus, the claim that $g(\lambda;\eta) \in [\textsc{low} ,  \textsc{high} ]$ during the second while loop holds.

Finally, we bound the number of oracle calls. Assume $g(\lambda;\eta) - \max_{z \neq \bar{z}} g(\lambda;\eta;z) > \varepsilon$ where $\bar{z} = \argmax_{z \in \Z} g(\lambda; \eta; z)$. Since at the end of the first while loop $\textsc{high} \leq 2 g(\lambda; \eta)$ and the second while loop performs binary search, we have that after $O(\log(\frac{g(\lambda;\eta)}{\varepsilon}))$ oracle calls, 
\begin{align*}
g(\lambda;\eta) \geq \textsc{low} > g(\lambda;\eta)- \varepsilon.
\end{align*}
Let $y \in \argmax_{z \in \Z} g(\lambda;\eta;\textsc{low};z)$; we claim that $y = \argmax_{z \in \Z} g(\lambda;\eta;z)$. By Lemma \ref{lem:bin_search_update_rule}, we have that $g(\lambda;\eta;\textsc{low};y) \geq 0$. Rearranging, we obtain
\begin{align*}
\frac{(z_0- y)^\t A(\lambda)^{-1/2} \eta}{b + \theta_0^\t(z_0 - y)} \geq \textsc{low} > g(\lambda;\eta) - \varepsilon > \max_{z \neq \bar{z}} g(\lambda;\eta;z),
\end{align*}
which implies that 
\begin{align*}
y = \argmax_{z \in \Z} g(\lambda;\eta;z) = \argmax_{z \in \Z}  \frac{(z_0- z)^\t A(\lambda)^{-1/2} \eta}{b + \theta_0^\t(z_0 - z)}
\end{align*}
proving the claim. 

Thus, inspection of the algorithm shows that it suffices to show that  $g(\lambda;\eta; g(\lambda;\eta)) = 0$, but this follows directly from Lemma \ref{lem:bin_search_lemma}.

\end{proof}

\begin{lemma}
\label{lem:bin_search_lemma}
If $g(\lambda; \eta ; r) < 0$, then $r > g(\lambda; \eta)$ and if $g(\lambda; \eta ; r) > 0$, then $r < g(\lambda; \eta)$.
\end{lemma}

\begin{proof}
Suppose $g(\lambda; \eta ; r) < 0$. Then, by definition,
\begin{align*}
\max_{z \in \Z} z^\t(A^{-1/2}(\lambda) \eta + r \theta_0) - r(b + \theta^\t_0 z_{0} ) -z_0^\t A(\lambda)^{-1/2} \eta < 0.
\end{align*}
Rearranging, we have that for all $z \in \Z$,
\begin{align*}
\frac{(z_0- z)^\t A(\lambda)^{-1/2} \eta}{b + \theta_0^\t(z_0 - z) } < r,
\end{align*}
thus proving the first claim. Next, suppose $g(\lambda; \eta ; r) > 0$. Then, rearranging as above, there exists a $z \in \Z$ such that
\begin{align*}
\frac{(z_0- z)^\t A(\lambda)^{-1/2} \eta}{b + \theta_0^\t(z_0 - z) } > r,
\end{align*}
proving the second claim.
\end{proof}

\begin{lemma}
\label{lem:bin_search_update_rule}
If $\max_{z \in \Z} g(\lambda; \eta; z; L_1) \geq 0$, then letting $L_2 = g(\lambda; \eta; z^\prime)$ for some $z^\prime \in \argmax _{z \in \Z} g(\lambda; \eta; z; L_1)$, we have that $L_2 \geq L_1$ and $g(\lambda; \eta; L_2) \geq 0$. Furthermore, $g(\lambda; \eta; \textsc{low}) \geq 0$ throughout the execution of Algorithm \ref{alg:compute_max}. 
\end{lemma}

\begin{proof}
We have that
\begin{align*}
g(\lambda; \eta; z^\prime ; L_1) = (z^\prime)^\t(A^{-1/2}(\lambda) \eta + L_1 \theta_0) - L_1(b + \theta^\t_0 z_{0} ) -z_0^\t A(\lambda)^{-1/2} \eta \geq 0.
\end{align*}
Rearranging, we have that
\begin{align*}
L_2 := \frac{(z_0- z^\prime)^\t A(\lambda)^{-1/2} \eta}{b + \theta_0^\t(z_0 - z^\prime) } \geq L_1,
\end{align*}
proving the first claim. Furthermore, rearranging the equality 
\begin{align*}
\frac{(z_0- z^\prime)^\t A(\lambda)^{-1/2} \eta}{b + \theta_0^\t(z_0 - z^\prime) } = L_2
\end{align*}
yields $0 = g(\lambda; \eta; z^\prime ;L_2) \leq  \max_{z \in \Z} g(\lambda; \eta; z ;L_2)  $, yielding the second inequality. 

Finally, $g(\lambda; \eta; \textsc{low}) \geq 0$ follows inductively. In the base case, $\textsc{low} = 0$ and we observe that for $0=g(\lambda; \eta; z_0;0) \leq  \max_{z \in \Z} g(\lambda; \eta; z ;0)  $. The inductive step follows by the update and the above claims. 

\end{proof}

\section{Fixed Budget Upper Bound Proofs}
\label{sec:fb_proof}

Lemma \ref{lem:fixed_budget_correctness} is the main step in the proof of the upper bound for the fixed budget algorithm. 

\begin{lemma}
\label{lem:fixed_budget_correctness}
Suppose $T \geq c R \max([\rho^* +\gamma^*],d)$. If $z_* \in \Z_k$, then $z_*$ is eliminated in round $k$ with probability at most 
\begin{align*}
2 \exp(\frac{-T}{c^\prime [\rho^*+\gamma^*] }). 
\end{align*}
\end{lemma}

\begin{proof}
Let $N ={\floor{T/R}}$. Let $\X = \{x_1, \ldots, x_m \}$. Let $\lambda_k$ denote the design chosen by the algorithm in round $k$. Let $x_{I_1}, \ldots, x_{I_N}$ denote the measurement vectors selected in round $k$ and define $\bar{\lambda} \in \simp$ by $\bar{\lambda}_i = \frac{1}{N}\sum_{s=1}^{N} \one\{I_s = i \}$. Let $\xi > 0$ (a constant to be chosen later). Define
\begin{align*}
\Delta = & \argmin \Delta^{\prime} \\
& \text{s.t. } \sup_{z,z^\prime \in \Z_k} \frac{\norm{z-z^\prime}_{A(\bar{\lambda})^{-1}}^2}{(\Delta^\prime)^2}  \leq \xi [\rho^* + \gamma^*].
\end{align*}

Define the event
\begin{align*}
\mc{E} = \{\sup_{ z,z^\prime \in \Z_k} \frac{|(z-z^\prime)^\t(\widehat{\theta}_{k}-\theta)|}{\Delta}& \leq \sqrt{\frac{\E[\sup_{z,z^\prime \in \Z_k} \frac{(z-z^\prime)^\t A(\bar{\lambda})^{-1/2} \eta}{\Delta}]^2 }{\floor{T/R}}}+ \frac{1}{2} \}.
\end{align*}
By Theorem 5.8 in \citep{boucheron2013concentration} with probability at least 
\begin{align*}
\P(\mc{E}^c) \leq 2\exp(\frac{-\floor{T/R} }{ 8  \frac{\sup_{z,z^\prime \in \Z_k} \norm{z-z^\prime}_{A(\bar{\lambda})^{-1}}^2}{\Delta^2} }) \leq  2\exp(\frac{-\floor{T/R}}{8 \xi [\rho^* + \gamma^*]})
\end{align*}
where we used the definition of $\Delta$. 
Suppose $\mc{E}$ occurs for the remainder of the proof.

Define
\begin{align*}
\Z_{k,wrong} = \{ z \in \Z_k : \widehat{\theta}_k^\t( z_* -z) < 0 \}. 
\end{align*} 
Towards a contradiction, suppose $z_*$ is eliminated at round $k$. Then, by definition of the algorithm,
\begin{align*}
\gamma(\Z_{k,wrong} \cup \{z_*\} ) \geq \frac{\gamma(\Z_k)}{2} = \frac{1}{2} \E \sup_{z,z^\prime \in \Z_k} (z -z^\prime)^\t A(\lambda_k)^{-1/2}  \eta.
\end{align*}
Define $z_0 = \argmax_{z \in \Z_{k,wrong}} \Delta_z$. Then,
 \begin{align}
 \frac{1}{2(1+\epsilon)} \E \sup_{z,z^\prime \in \Z_k} (z -z^\prime)^\t A (\bar{\lambda})^{-1/2} \eta & \leq  \frac{1}{2} \E \sup_{z,z^\prime \in \Z_k} (z -z^\prime)^\t A(\lambda_k)^{-1/2} \eta \label{eq:fb_proof_round} \\
 & \leq  \min_{\lambda}\E \sup_{z,z^\prime \in \Z_{k,wrong} \cup \{z_*\} } (z -z^\prime)^\t A (\lambda)^{-1/2} \eta \nonumber \\
 & \leq c \min_{\lambda}\E \sup_{z \in \Z_{k,wrong} \cup \{z_*\} } (z_* -z)^\t A (\lambda)^{-1/2} \eta \nonumber \\
 & \leq c^\prime \min_{\lambda}\E \sup_{z \in \Z_{k,wrong} } (z_* -z)^\t A (\lambda)^{-1/2} \eta \nonumber \\
 & + \norm{z_* -z_0}_{A(\lambda)^{-1}} \label{eq:fb_proof_no_z_star} 
 \end{align}
where line \eqref{eq:fb_proof_round} follows by the guarantees of the rounding procedure and Lemma \ref{lem:rounding_lemma} and line \eqref{eq:fb_proof_no_z_star} follows by Lemma \ref{lem:width_0}. Thus,
\begin{align}
 \E [\frac{\sup_{z,z^\prime \in \Z_k} (z -z^\prime)^\t A (\bar{\lambda})^{-1/2} \eta}{\Delta_{z_0}}]^2 & \leq  c\min_{\lambda}\E [\sup_{z \in \Z_{k,wrong} } \frac{(z_* -z)^\t A (\lambda)^{-1/2} \eta}{\Delta_{z_0}}]^2 \nonumber \\
 & + \frac{\norm{z_* -z_0}_{A(\lambda)^{-1}}^2 }{\Delta_{z_0}^2} \nonumber \\
 & \leq c^\prime [ \gamma^* + \rho^*]  \label{eq:fb_proof_ub}
\end{align}
where line \eqref{eq:fb_proof_ub} follows by Lemma \ref{lem:get_rhostar_gammastar}. Furthermore, we have that
\begin{align}
 \E [\frac{\sup_{z,z^\prime \in \Z_k} (z -z^\prime)^\t A (\bar{\lambda})^{-1/2} \eta}{\Delta_{z_0}}]^2 & \geq c  \sup_{z,z^\prime \in \Z_k} \frac{\norm{z-z^\prime}_{A(\bar{\lambda})^{-1}}^2}{\Delta_{z_0}^2}  \label{eq:fb_proof_lb}
\end{align}
by Lemma \ref{lem:rho_gamma}. Combining inequalities \eqref{eq:fb_proof_ub} and \eqref{eq:fb_proof_lb}, we have that there exists a univesral constant $\xi > 0$ such that $\Delta_{z_0} \geq \Delta$ (choose this $\xi$).

Then, 
\begin{align}
|(z_*-z_0)^\t (\widehat{\theta}_{k}-\theta)| & \leq \sup_{ z,z^\prime \in \Z_k} |(z-z^\prime)^\t(\widehat{\theta}_{k}-\theta)| \nonumber \\
& \leq \Delta_{z_0} \sqrt{\frac{\E[\sup_{z,z^\prime \in \Z_k} \frac{(z-z^\prime)^\t A(\bar{\lambda})^{-1/2} \eta}{\Delta_{z_0}}]^2 }{\floor{T/R}}}+ \frac{\Delta}{2} \label{eq:fb_proof_apply_event}\\
& \leq \Delta_{z_0} c^\prime \sqrt{ \frac{ \gamma^* + \rho^* }{\floor{T/R}}}+ \frac{\Delta}{2} \label{eq:fb_proof_get_to_rhostar_gammastar} \\
& < \frac{\Delta_{z_0}}{2} + \frac{\Delta}{2} \label{eq:fb_proof_big_T} \\
& \leq \Delta_{z_0}. \nonumber
\end{align}
where line \eqref{eq:fb_proof_apply_event} follows by the event $\mc{E}$, line \eqref{eq:fb_proof_get_to_rhostar_gammastar} follows by \eqref{eq:fb_proof_ub}, and line \eqref{eq:fb_proof_big_T} follows since $T \geq c R[\rho^* +\gamma^*]$ for an appropriately large universal constant $c>0$. Rearranging the above inequality implies that
\begin{align*}
(z_*-z_0)^\t \widehat{\theta}_{k} > 0
\end{align*}
and thus $z_0 \not \in \Z_{k,wrong}$, a contradiction. Therefore, on $\mc{E}$, $z_*$ is not eliminated. 

\end{proof}

\begin{proof}[Proof of Theorem \ref{thm:fixed_budget_upper_bound}]
Define the event
\begin{align*} 
E_k & = \{ z_* \text{ is not eliminated in round }k \} , \\
E & = \cap_{k=0}^{R} E_k.
\end{align*}
Then, by the law of total probability, Lemma \ref{lem:fixed_budget_correctness}, and the definition of $R = \ceil{\log( \gamma(\Z))} $,
\begin{align*}
\P(E^c) & \leq \P(E_1^c) + \sum_{k=2}^R \P(E_k^c | \cap_{l=1}^{k-1} E_l) \\
& \leq \ceil{\log( \gamma(\Z))} \exp(\frac{-T}{32 R [\rho^*+\gamma^*] }). 
\end{align*}
Assume the event $E$ holds. Recall the assumption that $\gamma(\{z, z_*\}) \geq 1$ for all $z \in \Z \setminus \{z_*\}$. Since by the definition of the algorithm and $R$,
\begin{align*}
\gamma(\Z_R) \leq \frac{\gamma(\Z)}{2^R} \leq 1
\end{align*}
the algorithm must terminate in one of the $\ceil{\log( \gamma(\Z))}$ rounds and return $z_*$, completing the proof. 

\end{proof}

\section{$\gamma^*$ Results}
\label{sec:gammastar_results}

In this Section, we prove various results related to $\gamma^*$.

\begin{proof}[Proof of Proposition \ref{prop:gamma_rho_gap}]
Define $\theta = e_1$ and $z_* = e_1$. Let
\begin{align*}
\Z = \{v \in \R^d : \norm{v}_2 = 1, v_1 = 0 \} \cup \{z_* \}.
\end{align*}
Let $\X = \{e_1, \ldots, e_d \}$. Then, for any $\lambda \in \simp$,
\begin{align*}
 \E_{\eta \sim N(0,I)}[ \max_{z \in \Z \setminus \{z_* \}} \frac{(z_*-z)^\t A(\lambda)^{-1/2} \eta}{\theta^\t(z_*-z)}]^2 & =  \E_{\eta \sim N(0,I)}[ \max_{z \in \Z \setminus \{z_* \}} (z_*-z)^\t A(\lambda)^{-1/2} \eta]^2 \\
& =  \E_{\eta \sim N(0,I)}[ \max_{z \in \Z \setminus \{z_* \}} z^\t A(\lambda)^{-1/2} \eta]^2 \\
& \geq (d-1) \E_{\eta \sim N(0,I)}[ \max_{z \in \Z \setminus \{z_* \}} z^\t  \eta]^2 \\
& \geq (d-1) (d + c)
\end{align*}
where $c$ is a universal constant where the second to last inequality follows by symmetry and the last inequality follows by example 7.5.7 in \citep{vershynin2019high}. On the other hand,
\begin{align*}
\rho^* & = \inf_\lambda \max_{z \in \Z \setminus \{z_* \}} \frac{\norm{z^*-z}^2_{A(\lambda)^{-1}}}{\theta^\t(z^*-z)^2} \\
& = \inf_\lambda \max_{z \in \Z \setminus \{z_* \}} \norm{z^*-z}^2_{A(\lambda)^{-1}} \\
& \leq 4 d
\end{align*}
where we took $\lambda = (1/d, \ldots, 1/d)^\t$. Thus, there exists an instance where $\rho^* \leq d c$ and $\gamma^* \geq c^\prime d^2$, proving the result.

Top-K is an example of a problem instance where $\gamma^* \leq c \log(d) \rho^*$ (see Proposition \ref{prop:top_k}).

\end{proof}

\begin{proposition}
\label{prop:top_k}
Consider an instance of Top-K. Assume wlog $\theta_1 \geq \theta_2 \geq \ldots \geq \theta_d$.
\begin{align*}
\gamma^* \leq c \log(d) [\sum_{i \leq k} (\theta_i - \theta_{k+1})^{-2} + \sum_{i > k} (\theta_k - \theta_{i})^{-2}].
\end{align*}
\end{proposition}

\begin{proof}[Proof of Proposition \ref{prop:top_k}]
Define 
\begin{align*}
\Delta_i =  \begin{cases}
                                   \theta_i - \theta_{k+1} & \text{if } i \leq k\\
                                   \theta_k - \theta_{i} & \text{if } i > k
  \end{cases}
\end{align*}
Set $\lambda_i = \frac{\Delta_i^{-2}}{\sum_{j \in [d]} \Delta_j^{-2}}$. Note that $\Z = \{z \subset [d]: |z| = k\}$. Then, 
\begin{align*}
\gamma^* & \leq \E_{\eta \sim N(0,I)}[ \max_{z \subset \Z \setminus [k]} \frac{\sum_{i \in [k] \Delta z} \frac{1}{\sqrt{\lambda_i}} \eta_i}{\sum_{i \in [k] \setminus z} \theta_i-\sum_{j \in z \setminus [k]} \theta_j}]^2 \\
& = \sum_{j \in [n]} \Delta_j^{-2} \E_{\eta \sim N(0,I)}[ \max_{z \subset \Z \setminus [k]} \frac{\sum_{i \in [k] \setminus z} (\theta_i-\theta_k )\eta_i +\sum_{j \in z \setminus [k]} (\theta_{k+1}-\theta_j) \eta_j}{\sum_{i \in [k] \setminus z} \theta_i-\sum_{j \in z \setminus [k]} \theta_j}]^2 \\
& =  \sum_{j \in [n]} \Delta_j^{-2} \E_{\eta \sim N(0,I)}[ \max_{z \subset \Z \setminus [k]} v_z^\t \eta + w_z^\t \eta]^2
\end{align*} 
where we defined the vectors
\begin{align*}
(v_z )_i  &= \frac{\theta_i-\theta_k }{\sum_{i \in [k] \setminus z} \theta_i-\sum_{j \in z \setminus [k]} \theta_j} \1\{i \in [k] \setminus z \} \\
(w_z )_i  &= \frac{\theta_{k+1}-\theta_i}{\sum_{i \in [k] \setminus z} \theta_i-\sum_{j \in z \setminus [k]} \theta_j} \1\{i \in [z] \setminus [k] \} \\
\end{align*}

Note that 
\begin{align*}
\norm{v_z}_1 = \frac{\sum_{i \in [k] \setminus z} (\theta_i-\theta_k )}{\sum_{i \in [k] \setminus z} \theta_i-\sum_{j \in z \setminus [k]} \theta_j} \leq \frac{\sum_{i \in [k] \setminus z} \theta_i-\sum_{j \in z \setminus [k]} \theta_j}{\sum_{i \in [k] \setminus z} \theta_i-\sum_{j \in z \setminus [k]} \theta_j}  = 1
\end{align*}
where we used the fact that $|[k] \setminus z| = |z \setminus [k]|$ and the assumption $\theta_1 \geq \theta_2 \geq \ldots \geq \theta_n$. Similarly, 
\begin{align*}
\norm{w_z}_1 \leq 1.
\end{align*}
Thus, 
 \begin{align*}
\gamma^* & \leq  \sum_{j \in [d]} \Delta_j^{-2} \E_{\eta \sim N(0,I)}[ \max_{z \subset \Z \setminus [k]} v_z^\t \eta + w_z^\t \eta]^2 \\
& \leq \sum_{j \in [d]} \Delta_j^{-2} \E_{\eta \sim N(0,I)}[ {\max}_{v: \norm{v}_1 \leq 1} v^\t \eta +  {\max}_{w :  \norm{w}_1 \leq 1}w^\t \eta]^2 \\
& \leq c \log(d) \sum_{j \in [d]} \Delta_j^{-2} 
\end{align*} 
where in the final inequality we used Example 7.5.9 of \citep{vershynin2019high}. 

\end{proof}

\begin{proof}[Proof of Proposition \ref{prop:upper_bound_gamma_by_rho}]
\begin{align}
\gamma^* & = \inf_\lambda \E_{\eta \sim N(0,I)}[ \max_{z \in \Z} \frac{[A(\lambda)^{-1/2}(z^*-z)]^\t  \eta}{\theta^\t(z^*-z)}]^2 \nonumber \\
& \leq \inf_\lambda c \log(|\Z|) \diam(\{\frac{A(\lambda)^{-1/2}(z^*-z)}{\theta^\t(z^*-z)} : z \in \Z \setminus \{z_*\} \})^2 \label{eq:gamma_by_rho_apply_diam} \\
& \leq c^\prime \log(|\Z|) \inf_\lambda \max_{z \in \Z \setminus z_*} \norm{\frac{A(\lambda)^{-1/2}(z^*-z)}{\theta^\t(z^*-z)} }_2^2  \nonumber \\
& =  c^\prime \log(|\Z|) \inf_\lambda \max_{z \in \Z \setminus z_*} \frac{\norm{z^*-z}^2_{A(\lambda)^{-1}}}{\theta^\t(z^*-z)^2}  \nonumber \\
& = c^\prime \log(|\Z|) \rho^*  \nonumber
\end{align}
where we used exercise 7.5.10 of \citep{vershynin2019high} in line \eqref{eq:gamma_by_rho_apply_diam}. On the other hand, Proposition 7.5.2 of \citep{vershynin2019high} implies that
\begin{align*}
\gamma^* \leq d \rho^*
\end{align*}

Now, we prove the lower bound. There exists $\xi > 0$, $z_1 \in \Z$,  and $\lambda_1 \in \simp$ such that
\begin{align*}
\xi + \inf_{z \neq z_*} \inf_{\lambda \in \simp} \frac{\norm{z_* - z}^2_{A(\lambda)^{-1}}}{\Delta_z^2} \geq \frac{\norm{z_*-z_1}_{A(\lambda_1)^{-1}}}{\Delta_{z_1}^2}.
\end{align*}
Let $\lambda_2 \in \simp$ attain $\gamma^*$. Let $\bar{\lambda} = \frac{1}{2}(\lambda_1 +\lambda_2)$ Then,
\begin{align}
\min_{\lambda \in \simp} \max_{z \neq z_*} & \frac{\norm{z_* - z}^2_{A(\lambda)^{-1}}}{\Delta_z^2} \nonumber \\
& \leq \max_{z \neq z_*} \frac{\norm{z_* - z}^2_{A(\bar{\lambda})^{-1}}}{\Delta_z^2} \nonumber \\
& \leq 4(\max_{z \neq z_*} \norm{\frac{z_* - z}{\Delta_z} - \frac{z_* - z_1}{\Delta_{z_1}}}^2_{A(\bar{\lambda})^{-1}}+ \frac{\norm{z_* - z_{1}}^2_{A(\bar{\lambda})^{-1}}}{\Delta_{z_1}^2}) \nonumber  \\
& \leq 4(\frac{\pi}{2} \E_{\eta \sim N(0,I)} \max_{z,z^\prime \in \Z \setminus \{z_*\}} (\frac{z_* - z}{\Delta_z} - \frac{z_* - z_1}{\Delta_{z_1}})^\t A(\bar{\lambda})^{-1/2} \eta]^2 \label{eq:gamma_by_rho_apply_width_ub} \\
& + \frac{\norm{z_* - z_{1}}^2_{A(\bar{\lambda})^{-1}}}{\Delta_{z_1}^2}) \nonumber \\
& = 4(2\pi \E_{\eta \sim N(0,I)} \max_{z \in \Z \setminus \{z_*\}} (\frac{z_* - z}{\Delta_z})^\t A(\bar{\lambda})^{-1/2} \eta]^2 + \frac{\norm{z_* - z_{1}}^2_{A(\bar{\lambda})^{-1}}}{\Delta_{z_1}^2}) \nonumber \\
& \leq 8(2\pi \E_{\eta \sim N(0,I)} \max_{z \in \Z \setminus \{z_*\}} (\frac{z_* - z}{\Delta_z})^\t A(\lambda_2)^{-1/2} \eta]^2 + \frac{\norm{z_* - z_{1}}^2_{A(\lambda_1)^{-1}}}{\Delta_{z_1}^2}) \label{eq:gamma_by_rho_apply_sud_fernique}\\
& \leq 8(2\pi \inf_{\lambda \in \simp} \E_{\eta \sim N(0,I)} \max_{z \in \Z \setminus \{z_*\}} (\frac{z_* - z}{\Delta_z})^\t A(\lambda_2)^{-1/2} \eta]^2 \nonumber \\
&  + \inf_{z \neq z_*} \inf_{\lambda \in \simp} \frac{\norm{z_* - z}^2_{A(\lambda)^{-1}}}{\Delta_z^2} +\xi). \nonumber 
\end{align}
where  line \eqref{eq:gamma_by_rho_apply_width_ub} follows by Lemma \ref{lem:rho_gamma} and line \eqref{eq:gamma_by_rho_apply_sud_fernique} follows by the Sudakov-Fernique inequality (Theorem 7.2.11 of \citep{vershynin2019high}) since $A(\bar{\lambda})^{-1} \preceq 2 A(\lambda_2)^{-1}$. Since $\xi > 0$ is arbitrary, sending $\xi \longrightarrow 0$ yields the lower bound. 
\end{proof}


\begin{proof}[Proof of Proposition \ref{prop:gamma_previous_combi_results}]
Recall the definition $B(z,r) = \{z^\prime \in \Z : \norm{z-z^\prime}_2 = r \}$.  Let $\lambda_i = \frac{\varphi_i}{\varphi^*}$. Further, define 
\begin{align*}
A_i = \{ j \in [d] : \log(d |B(z_*, j)| ) \in [2^{i-1}, 2^i]\}.
\end{align*} 
Let $v > 0$ a constant to be chosen later. Then,
\begin{align}
\sqrt{\gamma^*} & \leq \E_{\eta \sim N(0,I)}[ \max_{z \in \Z \setminus  z_*} \frac{\sum_{i \in z_* \Delta z} \frac{1}{\sqrt{\lambda_i}} \eta_i}{\Delta_z}] \nonumber \\
& = \frac{1}{v}\E_{\eta \sim N(0,I)}[ \max_{z \in \Z \setminus  z_*} v \frac{\sum_{i \in z_* \Delta z} \frac{1}{\sqrt{\lambda_i}} \eta_i}{\Delta_z}] \nonumber \\
& = \frac{\sqrt{\varphi^*}}{v}\E_{\eta \sim N(0,I)}[ \log( \max_{z \in \Z \setminus  z_*} \exp (v \frac{\sum_{i \in z_* \Delta z} \frac{1}{\sqrt{\varphi_i}} \eta_i}{\Delta_z}))] \nonumber \\
& \leq \frac{\sqrt{\varphi^*}}{v}\log( \E_{\eta \sim N(0,I)}[ \max_{z \in \Z \setminus  z_*} \exp (v\frac{\sum_{i \in z_* \Delta z} \frac{1}{\sqrt{\varphi_i}} \eta_i}{\Delta_z})]) \label{eq:prev_combi_jensen} \\
& =  \frac{\sqrt{\varphi^*}}{v} \log( \E_{\eta}[ \max_{i \in [4 \log(d)]} \max_{z \neq  z_*, |z \Delta z_*| \in A_i} \exp (v\frac{\sum_{i \in z_* \Delta z} \frac{1}{\sqrt{\varphi_i}} \eta_i}{\Delta_z})]) \label{eq:prev_combi_A_i_def} \\
& \leq \frac{\sqrt{\varphi^*}}{v}\log(  \sum_{i \in [4 \log(d)]}  \E_{\eta}[\max_{z \neq  z_*, |z \Delta z_*| \in A_i} \exp (v\frac{\sum_{i \in z_* \Delta z} \frac{1}{\sqrt{\varphi_i}} \eta_i}{\Delta_z})])  \label{eq:prev_combi_sum_max} \\
& \leq \frac{\sqrt{\varphi^*}}{v} \log(4 \log(d)  \max_{i \in [4 \log(d)]}  \E_{\eta }[\max_{z \neq  z_*, |z \Delta z_*| \in A_i} \exp (v\frac{\sum_{i \in z_* \Delta z} \frac{1}{\sqrt{\varphi_i}} \eta_i}{\Delta_z})]) \label{eq:prev_combi_mgf} 
\end{align} 
where line \eqref{eq:prev_combi_jensen} follows by Jensen's inequality, where line \eqref{eq:prev_combi_A_i_def} follows by the definition of $A_i$, and line \eqref{eq:prev_combi_sum_max} follows since the max is upper bounded by the sum. 

Notice that line \eqref{eq:prev_combi_mgf} contains the moment generating function of a Gaussian random variable. We upper bound its variance as follows. Suppose $|z_* \Delta z| \in A_i$. Then,
\begin{align}
 & \hspace{-1cm} \V(  \frac{\sum_{i \in z_* \Delta z} \frac{1}{\sqrt{\varphi_i}} \eta_i}{\Delta_z}) \nonumber \\
& = \V(\frac{\sum_{i \in z_* \Delta z} \min_{z^\prime : i \in z^* \Delta z^\prime} \frac{\Delta_{z^\prime}}{\sqrt{|z_* \Delta z^\prime| \log(d |B(z_*, |z_* \Delta z^\prime|)|)}} \eta_i}{\Delta_z}) \label{eq:prev_combi_allocation_def}\\
&= \sum_{i \in z_* \Delta z} \frac{ \min_{z^\prime : i \in z^* \Delta z^\prime} \tfrac{\Delta_{z^\prime}^2}{|z_* \Delta z^\prime| \log(d |B(z_*, |z_* \Delta z^\prime|)|)}}{\Delta_z^2} \label{eq:prev_combi_var}\\
& =  \frac{1}{|z_* \Delta z| \log(d |B(z_*, |z_* \Delta z|)|)} \sum_{i \in z_* \Delta z}  \frac{{\min}_{z^\prime : i \in z^* \Delta z^\prime} \tfrac{\Delta_{z^\prime}^2}{|z_* \Delta z^\prime| \log(d |B(z_*, |z_* \Delta z^\prime|)|)} }{\frac{\Delta_{z}^2}{|z_* \Delta z| \log(d |B(z_*, |z_* \Delta z|)|)}} \nonumber \\
& \leq  \frac{1}{ \log(d |B(z_*, |z_* \Delta z|)|)} \nonumber \\
& \leq  \frac{1}{2^{i-1}} \label{eq:prev_combi_A_i_ub}
\end{align}
where line \eqref{eq:prev_combi_allocation_def} follows by the definition of $\varphi_i$, line \eqref{eq:prev_combi_var} follows since $\eta \sim N(0,I)$, and line \eqref{eq:prev_combi_A_i_ub} follows since $|z_* \Delta z| \in A_i$. Now, continuing and using this upper bound on the variance, we have
\begin{align}
 \frac{\sqrt{\varphi^*}}{v} & \log(4 \log(d)  \max_{i \in [4 \log(d)]}  \E_{\eta }[\max_{z \neq  z_*, |z \Delta z_*| \in A_i} \exp (v\frac{\sum_{i \in z_* \Delta z} \frac{1}{\sqrt{\varphi_i}} \eta_i}{\Delta_z})]) \nonumber \\
& \leq \frac{\sqrt{\varphi^*}}{v} \log(4 \log(d)  \max_{i \in [4 \log(d)]} | \cup_{j \in A_i} B(z_*, j)| \exp(v^2 \frac{c}{2^{i+1}} )) \label{eq:prev_combi_crude_width_ub} \\
& =\max_{i \in [4 \log(d)]}  \sqrt{\varphi^*} [\log(4 \log(d)) /v+ \log(| \cup_{j \in A_i} B(z_*, j)|)/v +v \frac{c}{2^{i+1}} ] \nonumber \\
& = \max_{i \in [4 \log(d)]}  \sqrt{\varphi^* \log(4 \log(d)) \log(| \cup_{j \in A_i} B(z_*, j)|)\frac{c}{2^{i+1}}} \label{eq:prev_combi_max_v}  \\
& \leq  \max_{i \in [4 \log(d)]}  \sqrt{\varphi^* \log(4 \log(d)) \max_{j \in A_i} \log(|A_i| B(z_*, j)|)\frac{c}{2^{i+1}}} \nonumber \\
& \leq  \max_{i \in [4 \log(d)]}  \sqrt{\varphi^* \log(4 \log(d)) \max_{j \in A_i} \log(d| B(z_*, j)|)\frac{c}{2^{i+1}}} \label{eq:prev_combi_A_i_size} \\
& \leq  c \sqrt{\varphi^* \log( \log(d))} \label{eq:prev_combi_A_i_def_2}
\end{align}
where \eqref{eq:prev_combi_crude_width_ub} follows by Lemma \ref{lem:sup_exp} and $\{z \in \Z : z \neq z_*, |z \Delta z_*| \in A_i \} \subset  \cup_{j \in A_i} B(z_*, j)$, line \eqref{eq:prev_combi_max_v} follows by maximizing the constant $v$, \eqref{eq:prev_combi_A_i_size} follows since $|A_i| \leq d$, and line \eqref{eq:prev_combi_A_i_def_2} follows by definition of $A_i$. 

\end{proof}

%
%

\begin{proof}[Proof of Proposition \ref{prop:gamma_matroid}]

Define the allocation
\begin{align*}
\lambda_i \propto \tilde{\Delta}_i^{-2}
\end{align*}
where
\begin{align*}
\tilde{\Delta}_i = \left\{
        \begin{array}{ll}
      \theta^\t z_* - \max_{z \in \Z : i  \in z} \theta^\t z & i \not \in z_* \\
      \theta^\t z_* - \max_{z \in \Z : i  \not \in z} \theta^\t z & i \in z_* 
        \end{array}
    \right. 
\end{align*}

Then, 
\begin{align*}
\gamma^* & \leq \E_{\eta \sim N(0,I)}[ \max_{z \in \Z \setminus z_*} \frac{\sum_{i \in z_* \Delta z} \frac{1}{\sqrt{\lambda_i}} \eta_i}{\sum_{i \in z_* \setminus z} \theta_i-\sum_{j \in z \setminus z_*} \theta_j}]^2 \\
& = \sum_{i=1}^d \tilde{\Delta}_i^{-2} \E_{\eta \sim N(0,I)}[ \max_{z \in \Z \setminus z_*} \frac{\sum_{i \in z_* \Delta z} \tilde{\Delta}_i \eta_i}{\sum_{i \in z_* \setminus z} \theta_i-\sum_{j \in z \setminus z_*} \theta_j}]^2 \\
& =  \sum_{i=1}^d \tilde{\Delta}_i^{-2} \E_{\eta \sim N(0,I)}[ \max_{z \in \Z \setminus z_*} \frac{\sum_{i \in z_* \setminus z} \tilde{\Delta}_i \eta_i + \sum_{i \in z \setminus z_*} \tilde{\Delta}_i \eta_i}{\sum_{i \in z_* \setminus z} \theta_i-\sum_{j \in z \setminus z_*} \theta_j}]^2 \\
& =  \sum_{i=1}^d \tilde{\Delta}_i^{-2} \E_{\eta \sim N(0,I)}[ \max_{z \in \Z \setminus z_*} \frac{v_z^\t \eta + w_z^\t \eta }{\sum_{i \in z_* \setminus z} \theta_i-\sum_{j \in z \setminus z_*} \theta_j}]^2 
\end{align*} 
where we defined the vectors
\begin{align*}
(v_z )_i  &= \frac{ \theta^\t z_* - \max_{z \in \Z : i  \not \in z} \theta^\t z }{\sum_{j \in z_* \setminus z} \theta_j-\sum_{j \in z \setminus z_*} \theta_j} \1\{i \in z_* \setminus z \} \\
(w_z )_i  &= \frac{ \theta^\t z_* - \max_{z \in \Z : i  \in z} \theta^\t z }{\sum_{j \in z_* \setminus z} \theta_j-\sum_{j \in z \setminus z_*} \theta_j} \1\{i \in z \setminus z_* \} \\
\end{align*}

It remains to bound the expected suprema. Suppose wlog $z_* = \{1,\ldots, r\}$. By Lemma \ref{lem:matroid_lemma}, there exists a bijection $\sigma: z_* \longrightarrow z$ such that for every $i \in z_*$, $z^{(i)} := (z_* \setminus \{i\}) \cup \{\sigma(i)\} \in \Z$. Note that
\begin{align*}
\theta^\t z_* - \max_{z \in \Z : i  \not \in z} \theta^\t z \leq \theta^\t (z_* -z^{(i)}) = \theta_i - \theta_{\sigma(i)}.
\end{align*}
Therefore,
\begin{align*}
\norm{v_z}_1 & = \sum_{i \in z_* \setminus z} \frac{ |\theta^\t z_* - \max_{z \in \Z : i  \not \in z} \theta^\t z| }{\sum_{j \in z_* \setminus z} \theta_j-\sum_{j \in z \setminus z_*} \theta_j} \\
& \leq \sum_{i \in z_* \setminus z}  \frac{\theta_i - \theta_{\sigma(i)}  }{\sum_{j \in z_* \setminus z} \theta_j-\sum_{j \in z \setminus z_*} \theta_j} \\
& \leq 1.
\end{align*}
A similar argument show that $\norm{w_z}_1 \leq 1$. Thus, 
 \begin{align*}
 \E_{\eta \sim N(0,I)}[ \max_{z \subset \Z \setminus z_*} v_z^\t \eta + w_z^\t \eta]^2 & \leq \E_{\eta \sim N(0,I)}[ {\max}_{v: \norm{v}_1 \leq 1} v^\t \eta +  {\max}_{w :  \norm{w}_1 \leq 1}w^\t \eta]^2 \\
& \leq c \log(d)
\end{align*} 
where in the final inequality we used Example 7.5.9 of \citep{vershynin2019high}.

\end{proof}

The following Lemma appears as Corollary 3 in \citep{brualdi1969comments}.

\begin{lemma}
\label{lem:matroid_lemma}
Given two bases $B_1$ and $B_2$ of a matroid $\M = (E,I)$, there exists a bijection $\sigma:B_1 \longrightarrow B_2$ such that $(B_2 \setminus \sigma(e)) \cup e \in I$ for all $e \in I$. 
\end{lemma}

\section{Additional Lower Bounds}
\label{sec:additional_lower_bounds}

In this section, we show that in several common situations $\Omega(d)$ samples are required. The following Theorem applies to combinatorial bandits.

\begin{theorem}
\label{thm:combi_lower_bound_d}
Let $\delta \in (0,1/4)$. Consider the combinatorial bandit setting. Fix $\theta \in \Theta$ such that there is a unique best arm. Suppose $\Theta$ satisfies the following property: for all $i \in [d]$
 \begin{align*}
(\theta + e_i \cdot \min_{i \in z_* \setminus z} \Delta_z \in \Theta) \text{ or } (\theta - e_i \cdot \min_{i \in z \setminus z_*} \Delta_z \in \Theta) \text{ is true.}
 \end{align*}
If an algorithm $\A$ is $\delta$-pac wrt $(\X,\Z, \Theta)$, then 
\begin{align*}
\E_\theta[\sum_{i=1}^d T_i ] \geq \frac{d}{2}. 
\end{align*}
where $T_i $ denote the number of times that $\A$ pulls $e_i$.
\end{theorem}

The intuition behind the argument in Theorem \ref{thm:combi_lower_bound_d} is that if $\Omega(d)$ directions are not explored with constant probability, then there is some $\theta_i$ that the algorithm has no information about with constant probability and this could alter $\argmax_{z \in \Z} \theta^\t z$. 
 
\begin{remark}
Note that if $\Theta = \R^d$, then $\Theta$ satisfies the condition in the above Theorem. 
\end{remark}

\begin{proof}[Proof of Theorem \ref{thm:combi_lower_bound_d}]
Without loss of generality, suppose $1 = \argmax_i \theta^\t z_i$. Towards a contradiction, suppose there is some arm $i$ such that $\E_\theta[T_i] \leq \frac{1}{2}$. Let $z_j$ such that $i \in z_j \Delta z_1$ and suppose that $i \in z_j \setminus z_1$ (the other case is similar). Define
\begin{align*}
\tilde{\theta}_k =  \begin{cases}
                                   \theta_k & \text{if } k \neq i\\
                                   \theta_i + 2\theta^\t(z_1-z_j) & \text{if } k = i.
  \end{cases}
\end{align*}
Note that $(z_1-z_j)^\t \tilde{\theta} < 0$. Observe that 
\begin{align*}
\frac{1}{2} \geq \E_\theta[T_i] \geq \P_\theta(T_i > 0). 
\end{align*}
Define the event $A = \{T_i = 0\} \cap \{I = 1\}$, where $I$ denotes the index of the set output by $\A$ as its answer for the best set. Note that
\begin{align*}
\P_\theta(A^c) \leq \P_\theta(T_i > 0) + \P_\theta(I \neq 1) \leq \frac{1}{2} + \delta \leq \frac{3}{4}
\end{align*}
so that $\P_\theta(A) \geq \frac{1}{4}$. 

Define
\begin{align*}
\widehat{\text{kl}}_{i,T_i} = \sum_{s=1}^{T_i} \log(f_{\theta}(Z_s)/f_{\tilde{\theta}}(Z_s))
\end{align*}
where  $Z_s$ is the observation on the $s$th pull of $e_i$, $f_\theta$ denotes the density of the distribution associated with $e_i \in \X$ under $\theta$, and $f_{\tilde{\theta}}$ denotes the density of the distribution associated with $e_i \in \X$ under $\tilde{\theta}$. Then, by the change of measure identity (Lemma 18) from \citep{kaufmann2016complexity},
\begin{align*}
 \P_{\tilde{\theta}}(I=1) &\geq \P_{\tilde{\theta}}(A) \\
 &= \E_\theta[\one\{A\} \exp(-T_i \widehat{\text{kl}}_{i,T_i})] \\
 & = \P_\theta(A) \\
 & \geq \frac{1}{4}.
\end{align*}
where we used the fact that the only difference between problem $\theta$ and problem $\tilde{\theta}$ is the $i$th arm and on the event $A$, $T_i = 0$. Thus, on problem instance $(\Z, \tilde{\theta})$, $\A$ gives the incorrect answer with probability $1/4 > \delta$, which is a contradiction.

\end{proof}

The following Theorem gives a lower bound for best arm identification in linear bandits. 

\begin{theorem}
\label{thm:lin_bandits_lower_bound_d}
Let $\delta \in (0,1)$. Let $\X \subset \R^d$, such that $\norm{x_i}_2 \leq 1$ for all $i \in [|\X|]$, $\Z = \X$, and $\Theta = \R^d$. Fix $\theta \in \Theta$ such that there is a unique best arm and let $x_1 = \argmax_i \theta^\t x_i$. If an algorithm $\A$ is $\delta$-pac wrt $(\X,\Z, \Theta)$ and $d \geq 3$, then 
\begin{align*}
\E_\theta[\sum_{x \in \X} T_x] \geq c \log(\frac{1}{2.4} \delta)  \min_{i \neq 1} \frac{d}{\theta^\t(x_1 -x_i)^2}
\end{align*}
where $T_x$ denotes the number of times that $\A$ pulls $x \in \X$.  
\end{theorem}

\begin{proof}[Proof of Theorem \ref{thm:lin_bandits_lower_bound_d}]
By Theorem 1 of \citep{fiez2019sequential}, we have that 
\begin{align*}
\E_\theta[\sum_{x \in \X} T_x] \geq \log(\frac{1}{2.4} \delta) c \rho^*
\end{align*}
so it suffices to lower bound $\rho^*$. 

Since
\begin{align*}
\rho^* = \min_\lambda \max_{i \neq 1} \frac{\norm{x_1 -x_i}^2_{A(\lambda)^{-1}}}{\theta^\t(x_1 -x_i)^2} \geq [\min_{j \neq 1} \frac{c d}{\theta^\t(x_1 -x_j)^2}] \min_\lambda \max_{i \neq 1} \norm{x_1 -x_i}^2_{A(\lambda)^{-1}} ,
\end{align*}
it suffices to show that 
\begin{align*}
\min_\lambda \max_{i \neq 1} \norm{x_1 -x_i}^2_{A(\lambda)^{-1}} \geq c d.
\end{align*}
 Let
\begin{align*}
\lambda^* = \argmin_\lambda \max_{i \neq 1} \norm{x_1 - x_i}_{A(\lambda)^{-1}}.
\end{align*} 
Then,
\begin{align}
\sqrt{d} & = \min_\lambda \max_{i \in [|\X|]} \norm{x_i}_{A(\lambda)^{-1}}  \nonumber \\
& \leq \min_\lambda \max_{i \neq 1} \norm{x_1 - x_i}_{A(\lambda)^{-1}} + \norm{x_1}_{A(\lambda)^{-1}} \nonumber \\
& \leq \max_{i \neq 1} \norm{x_1 - x_i}_{[\frac{1}{2} A(\lambda^*) + \frac{1}{2} x_1 x_1^\t]^{-1}} + \norm{x_1}_{[\frac{1}{2} A(\lambda^*) + \frac{1}{2}x_1 x_1^\t]^{-1}} \nonumber \\
& \leq \sqrt{2} \max_{i \neq 1} \norm{x_1 - x_i}_{ A(\lambda^*)^{-1}} + \sqrt{2} \norm{x_1}_{ (x_1 x_1^\t)^{+}} \nonumber \\
& = \sqrt{2} \max_{i \neq 1} \norm{x_1 - x_i}_{ A(\lambda^*)^{-1}} + \sqrt{2}. \label{eq:lin_bandits_lb_d} 
\end{align}
The first line follows by Keifer-Wolfowitz (Theorem 21.1 in \citep{lattimore_szepesvari_2020}). The second to last inequality follows because 
\begin{align*}
\frac{1}{2}  A(\lambda^*) + \frac{1}{2}  x_1 x_1^\t \succeq \frac{1}{2}  A(\lambda^*) 
\end{align*}
which implies 
\begin{align*}
(\frac{1}{2}   A(\lambda^*))^{-1} \succeq (\frac{1}{2}   A(\lambda^*) + \frac{1}{2}   x_1 x_1^\t)^{-1}.
\end{align*}
Also, since $x_1 \in \text{span}(x_1)$, the same fact implies that 
\begin{align*}
\norm{x_1}_{ (\frac{1}{2}   x_1 x_1^\t)^{+}} \geq \norm{x_1}_{[\frac{1}{2}   A(\lambda^*) + \frac{1}{2}  x_1 x_1^\t]^{-1}}. 
\end{align*} 
Rearranging the inequality \eqref{eq:lin_bandits_lb_d}, we obtain
\begin{align*}
2  (\sqrt{d} -  \sqrt{2})^2 \leq \min_\lambda \max_{i \neq 1} \norm{x_1 -x_i}^2_{A(\lambda)^{-1}} 
\end{align*}
and thus the result follows.

\end{proof}

\section{Rounding}
\label{sec:rounding}

In this Section, we justify the application of the rounding procedure from \citep{allen2020near}. Define 
\begin{align*}
S_{N} = \{v \in \N^{|X|} : \sum_{i=1}^{|\X|} v_i \leq N \} \\
C_{N} = \{v \in [0,N]^{|X|} : \sum_{i=1}^{|\X|} v_i \leq N \}
\end{align*}

The following Theorem appears in \citep{allen2020near}. 
\begin{theorem}
\label{thm:round_theorem}
Let $F: \S_d^+ \longrightarrow \R$ such that
\begin{itemize}
\item For any $A, B \in \S_d^+$, if $A \preceq B$, then $F(A) \geq F(B)$,
\item for any $A \in \S_d^+$ and $t \in (0,1)$, $F(t A) = t^{-1} F(A)$. 
\end{itemize}
Let $\epsilon \in (0,1/6]$. Then, if $|\X| \geq N \geq 5\frac{d}{\epsilon^2}$, for any $\pi \in C_N$, there exists an algorithm that in $\tilde{O}(|\X| d^2)$ time rounds $\pi$ to $\kappa \in S_N$ such that
\begin{align*}
F(A(\kappa)) \leq (1+6 \epsilon) F(A(\pi)).
\end{align*}
\end{theorem}

The following result shows that the optimization problem
\begin{lemma}
\label{lem:rounding_lemma}
Fix $V \subset \R^d$. Define the functions $F, G: \S_d^+ \longrightarrow \R$
\begin{align*}
F(A) & =  \E_{\eta \sim N(0,I)}[ \max_{v \in V} v^\t A^{-1/2} \eta] \\
G(A) & = \max_{v \in V} v^\t A v.
\end{align*}
$F$ and $G$ satisfy the conditions of Theorem \ref{thm:round_theorem}.
\end{lemma}

\begin{proof}
It is trivial to see that $G$ satisfies the conditions of Theorem \ref{thm:round_theorem}. Thus, we focus on the function $F$. Let $A, B \in \S_d^+$ such that $A \preceq B$. Then, $A^{-1} \succeq B^{-1}$. Fix $v,w \in V$. Then,
\begin{align*}
\E [(v - w)^\t A^{-1/2} \eta]^2 = \norm{v-w}_{A^{-1}} \leq \norm{v-w}_{B^{-1}} = \E [(v - w)^\t B^{-1/2} \eta]^2
\end{align*} Then, by Sudakov-Fernique inequality (Theorem 7.2.11 in \citep{vershynin2019high}), it follows that $F(A) \geq F(B)$.

The second condition is trivial. 
\end{proof}

\section{Technical Lemmas related to $\gamma^*$} 
\label{sec:gammastar_technical}

In this Section, we state and prove several useful technical lemmas.

\begin{lemma}
\label{lem:rho_gamma}
Let $S \subset \Z$. Then,
\begin{align*}
\E_{\eta \sim N(0,I)}[ \max_{z,z^\prime \in S} [A(\lambda)^{-1/2}(z-z^\prime)]^\t  \eta]^2 \geq \frac{2}{\pi} \max_{z,z^\prime \in S} \norm{z-z^\prime}_{A(\lambda)^{-1}}^2.
\end{align*}
\end{lemma}

\begin{proof}
Fix $z_1,z_2 \in S$. Then,
\begin{align*}
\E_{\eta \sim N(0,I)}[ \max_{z,z^\prime \in S} [A(\lambda)^{-1/2}(z-z^\prime)]^\t  \eta] & \geq \E_{\eta \sim N(0,I)}\left[ |A(\lambda)^{-1/2}(z_1-z_2)]^\t  \eta | \right]  \\
& = \norm{z_1-z_2} \sqrt{\frac{2}{\pi}}
\end{align*}
\end{proof}

\begin{lemma}
\label{lem:get_rhostar_gammastar}
Let $\alpha > 0$ be a constant. Then, 
\begin{align*}
& \inf_\lambda \E_{\eta \sim N(0,I)}[ \max_{z \in \Z \setminus \{z_* \}} \frac{(z^*-z)^\t A(\lambda)^{-1/2} \eta}{\theta^\t(z^*-z)}]^2  +\max_{z \in \Z \setminus \{z_* \}} \frac{\norm{z^*-z}^2_{A(\lambda)^{-1}}}{\theta^\t(z^*-z)^2}  \alpha \\
& \qquad  \qquad   \leq c[\gamma^* + \rho^*   \alpha]
\end{align*}
\end{lemma}

\begin{proof}
Let $ \lambda_1 $ denote the solution to $\gamma^*$ and $\lambda_2$ the solution to $\rho^*$. Define $\lambda = \frac{1}{2}(\lambda_1 +\lambda_2)$.  It suffices to show that 
\begin{align*}
 \E_{\eta \sim N(0,I)}[ & \max_{z \in \Z \setminus \{z_* \}} \frac{(z^*-z)^\t A(\lambda)^{-1/2} \eta}{\theta^\t(z^*-z)}]^2 \\
 & \leq c  \E_{\eta \sim N(0,I)}[ \max_{z \in \Z \setminus \{z_* \}} \frac{(z^*-z)^\t A(\lambda_1)^{-1/2} \eta}{\theta^\t(z^*-z)}]^2
\end{align*}
and 
\begin{align*}
\max_{z \in \Z \setminus \{z_* \}} \frac{\norm{z^*-z}^2_{A(\lambda)^{-1}}}{\theta^\t(z^*-z)^2} \leq c \max_{z \in \Z \setminus \{z_* \}} \frac{\norm{z^*-z}^2_{A(\lambda_2)^{-1}}}{\theta^\t(z^*-z)^2}.
\end{align*}

 Note that 
\begin{align*}
1/2\sum_{x \in \X} (\lambda_{1,x} + \lambda_{2,x} ) x x^\t \succeq 1/2\sum_{x \in \X} \lambda_{i,x} xx^\t
\end{align*}
for $i =1,2$. Therefore, 
\begin{align}
2 A(\lambda_i)^{-1} \geq A(\lambda)^{-1} \label{eq:cov_ineq}
\end{align}
for $i=1,2$.

\eqref{eq:cov_ineq} immediately implies
\begin{align*}
\max_{z \in \Z \setminus \{z_* \}} \frac{\norm{z^*-z}^2_{A(\lambda)^{-1}}}{\theta^\t(z^*-z)^2} \leq c \max_{z \in \Z \setminus \{z_* \}} \frac{\norm{z^*-z}^2_{A(\lambda_2)^{-1}}}{\theta^\t(z^*-z)^2}.
\end{align*}
\eqref{eq:cov_ineq} implies via Sudakov-Fernique inequality (Theorem 7.2.11 in \citep{vershynin2019high}) that
\begin{align*}
 \E_{\eta \sim N(0,I)}[ & \max_{z \in \Z \setminus \{z_* \}} \frac{(z^*-z)^\t A(\lambda)^{-1/2} \eta}{\theta^\t(z^*-z)}]^2 \\
 & \leq c  \E_{\eta \sim N(0,I)}[ \max_{z \in \Z \setminus \{z_* \}} \frac{(z^*-z)^\t A(\lambda_1)^{-1/2} \eta}{\theta^\t(z^*-z)}]^2.
\end{align*}

\end{proof}

\begin{lemma}
\label{lem:width_scale_set}
Let $V = \{v_1, \ldots, v_l \} \subset \R^d$
and suppose $0 \in V$. Let $a_i \geq 1$ for all $i$. Then,
\begin{align*}
\E_{\eta \sim N(0,I)} \sup_{v_i \in V} v_i^\t \eta \leq \E_{\eta \sim N(0,I)} \sup_{v_i \in V} a_i v_i^\t \eta
\end{align*}
\end{lemma}

\begin{proof}
Fix $\eta \in \R^d$. Then, clearly, 
\begin{align*}
\sup_{v_i \in V} v_i^\t \eta \leq  \sup_{v_i \in V} a_i v_i^\t \eta. 
\end{align*}
Taking the expectation wrt $\eta \sim N(0,I)$ yields the result. 
\end{proof}

\begin{lemma}
\label{lem:width_nonneg}
Fix $V \subset \R^d$. Then, 
\begin{align*}
\E_{\eta \sim N(0,I)} \sup_{v \in V} v^\t \eta \geq 0.
\end{align*}
\end{lemma}

\begin{proof}
Fix $v_0 \in V$. Then,
\begin{align*}
 \E_{\eta \sim N(0,I)} \sup_{v \in V} v^\t \eta \geq  \E_{\eta \sim N(0,I)}  v_0^\t \eta = 0. 
\end{align*}
\end{proof}

\begin{lemma}
\label{lem:width_0}
Let $V \subset \R^d$ and suppose $0 \in V$. Fix $v_0 \in V$. Then,
\begin{align*}
\E \sup_{v \in V} v^\t \eta \leq 2 (\norm{v_0}_2 + \E \sup_{v \in V \setminus \{ 0 \}} v^\t g )
\end{align*}
\end{lemma}

\begin{proof}
\begin{align*}
\E \sup_{v \in V} v^\t \eta & \leq \E \sup_{v \in V \setminus \{0\}} |v^\t \eta| \leq 2 (\norm{v_0}_2 + \E \sup_{v \in V \setminus \{ 0 \}} v^\t g )
\end{align*}
where the last inequality follows by exercise 7.6.9 of \citep{vershynin2019high}.

\end{proof}

\begin{lemma}\label{lem:sup_exp}
Consider a sub-Gaussian random process $X_t$ indexed by $t \in \mc{T}$ such that for any $\nu$ we have $\E[ \exp( \nu X_t ) ] \leq \exp(\nu^2 \sigma_t^2 /2)$. Then $ \E\left[ \sup_{t \in \mc{T}} X_t \right] \leq \sqrt{2 \sup_{t \in \mc{T}} \sigma_t^2 \log\left( |\mc{T}|\right) }$.
\end{lemma}
\begin{proof}
\begin{align*}
    \E\left[ \sup_{t \in \mc{T}} X_t \right] &= \frac{1}{\nu} \E\left[ \sup_{t \in \mc{T}} \nu X_t \right] \\
    &= \frac{1}{\nu} \E\left[ \log\left( \sup_{t \in \mc{T}} \exp\left( \nu X_t\right) \right) \right] \\
    &\leq \frac{1}{\nu} \log\left(\E\left[  \sup_{t \in \mc{T}} \exp\left( \nu X_t\right) \right] \right) \\
    &\leq \frac{1}{\nu} \log\left( |\mc{T}| \sup_{t \in \mc{T}} \E\left[ \exp\left( \nu X_t\right) \right] \right)  \\
    &\leq \frac{1}{\nu} \log\left( |\mc{T}| \sup_{t \in \mc{T}} \exp\left( \nu^2 \sigma_t^2 /2\right) \right)  \\
    &= \frac{1}{\nu} \log\left( |\mc{T}|\right) + \nu \sup_{t \in \mc{T}} \sigma_t^2 /2  \\
    &\leq \sqrt{2 \sup_{t \in \mc{T}} \sigma_t^2 \log\left( |\mc{T}|\right) }
\end{align*}
\end{proof}

\section{Some Useful Results regarding Computational Efficiency}
\label{sec:comp_effic}

The following result shows that after a suitable monotonic transformation, the objective function in the optimization problems for finding a good allocation in Algorithms \ref{alg:action} and \ref{alg:action_comp} is convex when $\X = \{e_1, \ldots, e_d\}$, which holds in the combinatorial bandit problem. We note that Lemma \ref{lem:width_nonneg} shows that the gaussian width is nonnegative and thus it suffices consider the squareroot of the objective function. 

\begin{proposition}
\label{prop:gamma_convex_combi}
Fix $V \subset \R^d$. 
\begin{align*}
f(\lambda) =  \E_{\eta \sim N(0,I)}[ \max_{v \in V} v^\t \diag(\frac{1}{\lambda_i^{1/2}})\eta]
\end{align*}
is convex.
\end{proposition}

\begin{proof}
Fix $\lambda, \kappa \in \S^{n-1}$ and $\alpha \in [0,1]$. By matrix convexity,
\begin{align*}
\diag(\frac{1}{\alpha\lambda_i +(1-\alpha)\kappa_i})^{1/2} \preceq \alpha \diag(\frac{1}{\lambda_i })^{1/2} +(1-\alpha)\diag(\frac{1}{\kappa_i})^{1/2} .
\end{align*}
Furthermore, since the above matrices are diagonal,
\begin{align*}
\diag(\frac{1}{\alpha\lambda_i +(1-\alpha)\kappa_i}) \preceq (\alpha \diag(\frac{1}{\lambda_i })^{1/2} +(1-\alpha)\diag(\frac{1}{\kappa_i})^{1/2})^2 .
\end{align*}
Then, by Sudakov-Fernique inequality (Theorem 7.2.11 \citep{vershynin2019high}),
\begin{align*}
f(\alpha\lambda +(1-\alpha)\kappa) & = \E_{\eta \sim N(0,\diag(\frac{1}{\alpha\lambda_i +(1-\alpha)\kappa_i}))} \sup_{v \in V} v^\t \eta \\
& \leq \E_{\eta \sim N(0,(\alpha \diag(\frac{1}{\lambda_i })^{1/2} +(1-\alpha)\diag(\frac{1}{\kappa_i})^{1/2})^2)} \sup_{v \in V} z^\t \eta \\
& = \E_{\eta \sim N(0,I)} \sup_{v \in V} v^\t (\alpha \diag(\frac{1}{\lambda_i })^{1/2} +(1-\alpha)\diag(\frac{1}{\kappa_i})^{1/2})\eta \\
& \leq  \alpha \E_{\eta \sim N(0,I)} \sup_{v \in V} v^\t  \diag(\frac{1}{\lambda_i })^{1/2} \eta \\
& +  (1-\alpha)\E_{\eta \sim N(0,I)} \sup_{v \in V} v^\t \diag(\frac{1}{\kappa_i})^{1/2}\eta \\
& = \alpha f(\lambda ) + (1-\alpha)f(\kappa) 
\end{align*}

\end{proof}

\section{Comparison Results}
\label{sec:comp_results}

In this Section, we prove various results related to the sample complexities proposed in other works. Recall the notation for the sphere $B(z,r) = \{z^\prime \in \Z : \norm{z-z^\prime}_2 = r \}$. 


\begin{proof}[Proof of Proposition \ref{prop:cao_jain_counter}]
Define $\theta_1 = \ldots = \theta_k = 1/2$, $\theta_{k+1} = \ldots = \theta_{2k-1} = \frac{1}{2}- \frac{1}{k^{1/2}}$ and $\theta_{2k} = \ldots = \theta_d = 0$ and $d = k^2$. Define
\begin{align*}
\Delta_i & = \begin{cases}
\theta_i - \theta_{k+1} & : i \leq k \\
\theta_k - \theta_i & : i > k
\end{cases} \\
\bar{\lambda}_i & = 
\frac{\Delta_i^{-2}}{\sum_{i=1}^d \Delta_i^{-2}}
\end{align*} 
Note that 
\begin{align*}
\rho^* & \leq \sum_i \Delta_i^{-2} \max_{z \neq z_*} \frac{\sum_{i \in z_* \Delta z} \Delta_i^2}{\Delta_z^2} \\
& \leq c \sum_i \Delta_i^{-2}  \\
& \leq c[k^2 + d] \\
& \leq c^\prime d.
\end{align*}
Consider arm $d$. We will show that $\varphi_d \geq ck \log(d) $. Fix $\tilde{z} = \{k+1, k+2, \ldots,2k-1, d\}$ and $z_* = [k]$. It suffices to show that
\begin{align*}
\frac{\norm{z_* - \tilde{z}}_1  \log( |B(z_*, |z_* \Delta \tilde{z}|)}{\theta^\t(z_* - \tilde{z})^2}  \geq c \log(d) k,
\end{align*}
from which the claim will follow. Note that
\begin{align*}
\frac{\norm{z_* - \tilde{z}}_1 }{\theta^\t(z_* - \tilde{z})^2} & = \frac{2k}{(\frac{k-1}{\sqrt{k}} + \frac{1}{2})^2} \geq c  . 
\end{align*}
Furthermore,
\begin{align*}
 \log( |B(z_*, |z_* \Delta \tilde{z}|)|) & \geq \log({d-2k \choose k}) \\
& \geq \log(\frac{(d-2k)^k}{k!}) \\
& \geq k  \log(\frac{d-2k}{k}) \\
& \geq k \log(k - 2) \\
& \geq \frac{1}{4} k \log(d)
\end{align*}
where in the last inequality we used $d=k^2$. Thus, the claim follows and $\varphi_d \geq ck$. A similar argument applies to arms $\{2k, \ldots, d-1\}$ yielding the result.
\end{proof}

The following proposition shows that $\rho^*$ is lower bounded by the typical measure of hardness for top-k \citep{kaufmann2016complexity}. It implies that the sample complexity of \citep{chen2017nearly, fiez2019sequential}  is off by a factor of $k$.

\begin{proposition}
Consider the top-k problem where $\theta_1 \geq \ldots \theta_k > \theta_{k+1} \geq \ldots \geq \theta_n$.
\begin{align*}
\rho^* \geq \sum_{i \leq k} \frac{1}{(\theta_i - \theta_{k+1})^2} + \sum_{i > k} \frac{1}{(\theta_k - \theta_{i})^2} 
\end{align*}
\end{proposition}

\begin{proof}
\begin{align*}
\rho^* & = \inf_\lambda \max_{z \in \Z \setminus \{z_* \}} \frac{\norm{z^*-z}^2_{A(\lambda)^{-1}}}{\theta^\t(z^*-z)^2} \\
& =  \inf_{\lambda} \max_{A \neq [k]} \frac{\sum_{i \in A \Delta [k]} \frac{1}{\lambda_i}}{(\sum_{i \in [k]} \theta_i - \sum_{i \in A} \theta_i)^2 } \\
& \geq \min_{\lambda} \max( \max_{i \in [k]} \frac{\frac{1}{\lambda_i} + \frac{1}{\lambda_{k+1}}}{(\theta_i-\theta_{k+1})^2},  \max_{i \in [d]\setminus [k]} \frac{\frac{1}{\lambda_i} + \frac{1}{\lambda_{k}}}{(\theta_k-\theta_{i})^2}) \\ 
&  \geq  \min_{\lambda} \max( \max_{i \in [k]} \frac{\frac{1}{\lambda_i} }{(\theta_i-\theta_{k+1})^2},  \max_{i \in [d]\setminus [k]} \frac{\frac{1}{\lambda_i} }{(\theta_k-\theta_{i})^2}) 
\end{align*}
To minimize the RHS, we set it to a constant $c$. Then, 
\begin{align*}
\lambda_i = \begin{cases} 
        \frac{1}{c(\theta_i - \theta_{k+1})^2} & i \leq k \\
       \frac{1}{c(\theta_i - \theta_{k})^2} & \text{otherwise} 
          \end{cases}
\end{align*}
Then, the result follows from the below and solving for $c$. 
\begin{align*}
1 = \sum_{i=1} \lambda_i = \frac{1}{c}[\sum_{i \leq k} \frac{1}{(\theta_i - \theta_{k+1})^2} + \sum_{i > k} \frac{1}{(\theta_k - \theta_{i})^2} ].
\end{align*}

\end{proof}

The following gives an instance where $|\Z|$ is linear in the dimension $d$, but $\varphi^*$ is loose by a $\sqrt{d}$ factor.

\begin{proposition}
Consider the combinatorial bandit setting. There exists a problem where $|\Z|$ is linear in the dimension $d$ and $\varphi^* \geq c \rho^* \sqrt{d}$.
\end{proposition}

\begin{proof}
Fix $k < d$. Define $z_1 = [k], z_2 = \{k +1 \}, z_3 = \{k + 3 \}, \ldots, z_{d-k} = \{d\}$ and let $\Z = \{z_1, \ldots, z_{d-k}\}$. Note $|\Z|  \leq d$ and thus satisfies the hypothesis. Fix $\epsilon > 0$ and let 
\begin{align*}
  \theta_i =
  \begin{cases}
                                   \epsilon & i \leq k\\
                                   0 & \text{otherwise}
  \end{cases}.
\end{align*}
Then, $z_* = z_1$. The upper bound guarantee of \citep{cao2019disagreement, jain2019new} is at least
\begin{align}
\sum_{i=1}^k \max_{z \neq z_*} \frac{|z_* \Delta z| }{\theta^\t(z_* -z)^2 } + \sum_{i=k+1}^d  \frac{|z_* \Delta z_i| }{\theta^\t(z_* -z_i)^2 } & = d \frac{k+1}{(k \epsilon)^2} \nonumber \\
& \geq \frac{d}{k \epsilon^2}. \label{eq:akshay_jain_comparison_1}
\end{align}
On the other hand, we have that 
\begin{align}
\rho^*  & = \max_{z \neq z_*} \frac{\norm{z_*-z}^2_{A(\lambda)^{-1}}}{\theta^\t (z_*-z)^2} \leq 2[ \frac{k^2 + d}{(k\epsilon)^2} \nonumber \\
& \leq 2[\frac{1}{\epsilon^2} + \frac{d}{(k \epsilon)^2}] \label{eq:akshay_jain_comparison_2}
\end{align}
where we took
\begin{align*}
  \lambda_i =
  \begin{cases}
                                   \frac{1}{2k} + \frac{1}{2d} & i \leq k\\
                                   \frac{1}{2d} & \text{otherwise}
  \end{cases}.
\end{align*}
Putting $k = \sqrt{d}$ into \eqref{eq:akshay_jain_comparison_1} and \eqref{eq:akshay_jain_comparison_2} yields the result. 
\end{proof}

In the matching problem, if  $\theta = \one\{i \in z\} \Delta$ for some $z \in Z$ and $\Delta > 0$, we say that it is an instance of \textsc{Homogenous Matching}. The following result appears in \citep{cao2019disagreement}. It shows that the sample complexity of \citep{cao2019disagreement,jain2019new} is correct for the homogeneous matching problem.

\begin{proposition}
Consider the homogenous \textsc{matching} problem. Then, $\rho^* = \Theta(d/\Delta^2)$. Further, letting
\begin{align*}
\varphi_i = \max_{z \in \Z \setminus \{z_* \}  : i \in z^* \Delta z} \frac{|z_* \Delta z| \log( |B(z_*, |z_* \Delta z|)|)}{\Delta_z^2}
\end{align*}
we have that $\sum_{i=1}^n \varphi_i = O(d/\Delta^2)$.
\end{proposition}

\begin{remark}
It follows from Proposition \ref{prop:gamma_previous_combi_results} that for the homogenous \textsc{matching} problem, $\gamma^* \leq O(\log(\log(d)) d/\Delta^2)$
\end{remark}

The following result appears in \citep{chen2017nearly}. It shows that there is a gap of order $d$ between the sample complexities in \citep{chen2014combinatorial} and \citep{gabillon2016improved} and the lower bound.

\begin{proposition}
\label{prop:chen_gabillon_counter}
Let $d$ be even. Consider the combinatorial bandit setting where $\X = \{e_1,\ldots, e_d \}$ and $\Z = \{ [d/2], \{d/2+1,\ldots, d\}$ and $\theta_i = \epsilon \one \{i \leq d/2\}$. Then, the guarantee of the CLUCB in \citep{chen2014combinatorial} and the algorithm in \citep{gabillon2016improved} is $\Omega(d \epsilon^{-2} \ln(1/d))$. On the other hand, $\rho^* = \epsilon^{-2}$. 
\end{proposition}

The following result shows that the sample complexity cannot depend on $\log(\Z)$ because $|\Z|$ can be arbitrarily large while $\gamma^* \leq 1$.

\begin{proposition}
\label{prop:log_z_bad}
For any $N \in \N$, there exists an instance of the transductive linear bandit problem where $|\Z| \geq N$ and $\gamma^* \leq 1$.
\end{proposition}

\begin{proof}[Proof of Proposition \ref{prop:log_z_bad}]
Let $\X = \{e_1, \ldots, e_d\}$. Let $\theta = a e_1$ for a constant $a > 0$ to be chosen later. Fix $\epsilon > 0$.  Let $z_1 = e_1$. There exists $z_2, \ldots, z_N$ such that for every $i$, $\norm{z_1-z_i}_2 = \epsilon$, $e_1^\t z_i = 0$, and $\norm{z_i}_2  = 1 $. Then, $ \Delta_i := \theta^\t (z_1 - z) = a$ for all $i$ and some $\Delta > 0$. Then, by Proposition 7.5.2 of \citep{vershynin2019high}, we have that 
\begin{align*}
\inf_{\lambda \in \simp} \E[ \sup_{i > 1} \frac{(z_* - z_i)^\t A(\lambda)^{-1/2} \eta}{\Delta_i}] & = \frac{1}{a} \inf_{\lambda \in \simp} \E[ \sup_{i > 1} (z_* - z_i)^\t A(\lambda)^{-1/2} \eta] \\
&  \leq \frac{\sqrt{d}}{a} \max_{i > 1} \norm{A(\lambda)^{-1/2} (z_*-z_i)}_2 \\
& \leq  \frac{d}{a} \max_{i > 1} \norm{z_*-z_i}_2 \\
& = \frac{d}{a} \epsilon \\
& \leq \frac{1}{a} \\
& = \frac{1}{a} \\
& \leq 1
\end{align*}
for small enough $\epsilon > 0$ and $a > 0$ large enough. Thus, the claim follows.

\end{proof}

\section{Extension to SubGaussian noise}

We briefly sketch the extension to SubGaussian noise. First, we define some notation: If $Y$ is a random variable, define $\norm{Y}_{\psi_2} := \inf\{s > 0 : E\frac{Y^2}{s^2} \leq 1\}$, i.e., the 2-Orlicz norm. If $Y$ is a random vector, then $\norm{Y}_{\psi_2} = \sup_{v : \norm{v}_2 = 1} \norm{v^\t Y}_{\psi_2}$ (see \citep{vershynin2019high} for a reference). 

Let $n \geq d$ and fix a set of measurements $x_{I_1}, \ldots, x_{I_n}$ and let $y_1, \ldots, y_n$ be the associated observations where we assume $y_i = x_i^\t \theta + \eta_i$ for $\eta_i$ is independent mean-0 subGauss($1$) noise. Define the matrix 
\begin{align*}
X= \begin{pmatrix}
x_{I_1}^\t \\
\vdots \\
x_{I_T}^\t 
\end{pmatrix} \\
\end{align*}
Define $\widehat{\theta} = (X^\t X)^{-1} X^\t Y$. Note that $\widehat{\theta} - \theta = (X^\t X)^{-1} X^\t \eta $. Note that $\norm{\eta}_{\psi_2} \leq 1$. For any $v \in \R^d$,
\begin{align*}
\norm{v^\t (X^\t X)^{-1} X^\t \eta}_{\psi_2} & = \norm{X (X^\t X)^{-1} v  }_2 \norm{\frac{1}{\norm{X (X^\t X)^{-1} v  }_2} v^\t (X^\t X)^{-1} X^\t \eta}_{\psi_2} \\
& \leq \norm{X (X^\t X)^{-1} v  }_2 \\
& = \norm{ v  }_{(X^\t X)^{-1}}
\end{align*}
This shows that $\norm{v^\t (X^\t X)^{-1} X^\t \eta}_{\psi_2} \leq \norm{v^\t (X^\t X)^{-1} X^\t \tilde{\eta}}_{\psi_2}$ where $\tilde{\eta} \sim N(0,I)$. Thus, applying Theorem 8.5.5 and Talagrand's majorizing measure theomem (Theorem 8.6.1) from \citep{vershynin2019high} yields for all $z \in \Z \setminus \{z_\ast \}$
\begin{align*}
    (z_{\ast} -z )^{\top}\widehat\theta\geq (z_{\ast} - z)^\top \theta &- c\Big(\E_{\eta\sim N(0, I_d)}\left[\sup_{z\in \Z \setminus \{z_*\}} (z_{\ast} - z)^{\top} A^{-1/2}\eta\right] \\
    &-\sqrt{2\sup_{z\in \Z \setminus \{z_*\}} \|z_{\ast}-z\|_{A^{-1}}^2 \log(\tfrac{1}{\delta})}\Big),
\end{align*}
where $c > 0$ is a universal constant, which is the essential concentration inequality used for the arguments in this paper.

\section{Experiment Details}
\label{sec:experimental_details}

\begin{figure}
\centering
\includegraphics[scale=0.8]{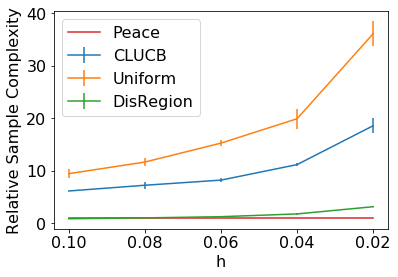}
\caption{The relative performance of CLUCB, UA, and the algorithm from \citep{cao2019disagreement} (denoted DisRegion) to PEACE on an instance of biclique.}
\label{fig:biclique_experiment}
\end{figure}

\textbf{Biclique Experiment:} In the biclique problem, we are given a complete balanced bipartite graph with $\sqrt{d}$ nodes in each group and a total of $d$ edges. $\Z$ is the set of bicliques with $\sqrt{s}$ nodes from each group in the bipartite graph. This problem is NP-hard, so there is no linear maximization oracle. Therefore, we consider a small instance where $\sqrt{d} = 8$ and $\sqrt{s} = 2$. We use a similar setup to the other combinatorial bandit experiments. We pick two random non-overlapping bicliques and let $\B_1$ and $\B_2$ denote the set of their respective edges. If $e \in \B_1$, we set $\theta_e = 1$, and if $e \in \B_2$, we set $\theta_e = 1-h$ for $h \in \{.1,.8,.6,.4,.2\}$. Otherwise, we set $\theta_e = 0$. 

We also compare to the Algorithm 4 from \citep{cao2019disagreement} (denoted DisRegion), which attains the best sample complexity result from that paper. Figure \ref{fig:biclique_experiment} shows that as the gap between the best biclique and the second best biclique decreases, the performance of the competing algorithms degrades relative to Peace. For example, for large $h$, Peace and DisRegion have similar performace but for $h=.2$, DisRegion requires more than 3 times as many samples as Peace.

\textbf{Combinatorial Bandit Experiments:} We used Python 3 and parallized the simulations on an Intel(R) Xeon(R) CPU E5-2690. For each experiment, we generate noise from a standard normal distribution. We used the stochastic mirror descent algorithm described in Section \ref{sec:comp_results}, but let $\lambda \in \simp$ (instead of $\simpm$). We ran the algorithm for 1000 iterations with a batch size of $10$ on all experiments. Once we obtained a $\lambda \in \simp$, we used 2,000 samples to form an empirical mean to estimate the Gaussian width. We considered the setting where it is known that $\max_z \Delta_z \leq 2d$, which holds for example when $\theta \in [-1,1]$, and thus solved
\begin{align*}
\inf_{\lambda \in \simp} \E_{\eta \sim N(0,I)}[ \max_{z \in \Z} \frac{(\tilde{z}_k- z)^\t A(\lambda)^{-1/2} \eta}{2^{-k} \cdot 2d+ \widehat{\theta}^\t_{k}(\tilde{z}_k - z) }]^2 
\end{align*}
instead of \eqref{eq:action_comp_2_supp}. We rounded our designs $\tau_k \lambda $ simply by taking the ceiling (which only incurs a loss of an additive factor of $d$ because $|\X| \leq d$. 

To implement CLUCB, we use a state-of-the-art anytime confidence bound (inequality (2) from \citep{howard2018uniform}), which is much better than the one used in \citep{chen2014combinatorial}. For the uniform allocation algorithm, we use the termination condition that one obtains from applying the TIS inequality (Theorem 5.8 in \citep{boucheron2013concentration}) to the process $\widehat{\theta}^\t(z-z^\prime)$. 

We used 20 trials for the matching experiment, 30 trials for the shortest path experiment, and 60 trials for the biclique experiment. We generated $95\%$ confidence intervals using the bootstrap.

\textbf{Transductive Linear Bandits: }
We made two main changes to the algorithm as written, both focused on computing the objective $\inf_{\lambda\in \Delta} \tau(\lambda;\mathcal{Z}_k)$ more effectively. 
Firstly, we considered two different subproblems: $\min_{\lambda\in \Delta} \mathbb{E}_{\eta\sim N(0,I)}[ \max_{z’,z\in \mathcal{Z}_k} (z-z’)^{\top}A(\lambda)^{-1/2}\eta]^2$ and $\min_{\lambda}\max_{z,z’} \|z’-z\|_{A(\lambda)^{-1}}^2$. 
In the setting where there are extremely large number of arms, it is not practical to take a max over all pairs of them - so in both subproblems we only took the max over $\hat{z}_k-\mathcal{Z}_k$ where $\hat{z}_k = \text{argmax}_{z\in \mathcal{Z}_k} \hat{\theta}_k^{\top} z_k$. 
To justify this, we point out that by Theorem 7.5.2 of \cite{vershynin2019high} $\mathbb{E}_{\eta\sim N(0,I)}[ \max_{z’,z\in \mathcal{Z}_k} (z-z’)^{\top}A(\lambda)^{-1/2}\eta] =2 \mathbb{E}_{\eta\sim N(0,I)}[\max_{ z\in \mathcal{Z}_k} (\hat{z}_k-z)^{\top}A(\lambda)^{-1/2}\eta]$, and $\min_{\lambda}\max_{z,z’} \|z’-z\|_{A(\lambda)^{-1}}^2 \leq 4\min_{\lambda}\max_{z\in \mathcal{Z}_k} \|\hat{z}_k-z\|_{A(\lambda)^{-1}}^2$. 
Motivated by this, we computed the distribution $\lambda’ = \text{argmin}_{\lambda} \mathbb{E}_{\eta\sim N(0,I)}[ \max_{z\in \mathcal{Z}_k} (\hat{z}_k-z)^{\top}A(\lambda)^{-1/2}\eta]$ and $\lambda^{\prime \prime } =  \min_{\lambda}\max_{z} \|\hat{z}_k-z\|^2_{A(\lambda)^{-1}}$ and set $\lambda_k = (\lambda’+\lambda^{\prime \prime })/2$. 
Note that using this distribution only makes the algorithm perform worst than if the optimal - it does not affect correctness in anyway. 

\textbf{Fixed Budget:}
As in the previous, we computed an allocation not using $\gamma(Z_k)$ but rather a minimum over the differences $\hat{z}_k - \mathcal{Z}_k$.

\clearpage
\bibliography{refs}

\end{document}